\documentclass[twoside]{article}

\usepackage[preprint]{aistats2027}
\usepackage[numbers,sort&compress]{natbib}

\usepackage[utf8]{inputenc} 
\usepackage[T1]{fontenc}    
\usepackage[backref = page, colorlinks = True, linkcolor = violet!70!white, citecolor = teal]{hyperref}       
\usepackage{url}            
\usepackage{booktabs}       
\usepackage{amsfonts}       
\usepackage{nicefrac}       
\usepackage{microtype}      
\usepackage[dvipsnames]{xcolor}
\usepackage{float}

\usepackage{amsmath}
\usepackage{amssymb}
\usepackage{mathtools}
\usepackage{amsthm}

\usepackage[capitalize,noabbrev]{cleveref}

\theoremstyle{plain}
\newtheorem{theorem}{Theorem}[section]
\newtheorem{proposition}[theorem]{Proposition}
\newtheorem{lemma}[theorem]{Lemma}

\theoremstyle{definition}

\theoremstyle{remark}
\newtheorem{remark}[theorem]{Remark}

\usepackage{graphicx}
\usepackage{subcaption}
\usepackage{placeins}
\usepackage{enumitem}

\usepackage{subcaption}

\usepackage{wrapfig}
\usepackage{graphicx}

\newcommand{\Data}{\mathbf{D}}
\newcommand{\Lu}{\mathrm{IG}_{\phi}(\Data)}           
\newcommand{\Lc}{\mathcal{I}_{\phi;\Data\mid\theta}}  
\newcommand{\Cpl}{\mathcal{I}_{\theta;\phi\mid\Data}}  
\newcommand{\Itheta}{\mathcal{I}_{\theta;y}}  
\newcommand{\VOI}{\mathrm{VOI}}
\providecommand{\Cov}{\operatorname{Cov}}

\newcommand{\Cpldw}{\mathcal{I}_{\theta;\phi\mid\Data,\omega}}

\newcommand{\Luw}{\mathrm{IG}_{\phi|\omega}(\Data)}
\newcommand{\Lcwpi}{\mathcal{I}_{\phi;\Data\mid\theta,\omega,\pi}} 

\newcommand{\Ithetapi}{\mathcal{I}_{\theta;\Data|\pi}}
\newcommand{\Cplpi}{\mathcal{I}_{\theta;\phi\mid\Omega,\pi}}
\newcommand{\Lupi}{\mathcal{I}_{\phi;\Data\mid\Omega,\pi}}
\newcommand{\Lcpi}{\mathcal{I}_{\phi;\Data\mid\Theta,\Omega,\pi}}

\newcommand{\Fc}{F_{\theta,\phi}}                 
\newcommand{\Fphi}{F_{\phi}}            
\newcommand{\Fcond}{F_{\phi|\theta}}    
\newcommand{\Rc}{R_{\theta,\phi}}                 
\newcommand{\Rphi}{R_{\phi}}            
\newcommand{\orient}{\eta}    
\newcommand{\Rcond}{R_{\phi|\theta}}    

\usepackage{tikz}

\usetikzlibrary{arrows.meta,positioning,fit,calc,shapes.geometric}
\tikzset{
  rv/.style={circle, draw, thick, minimum size=7.6mm, inner sep=0pt, font=\small},
  obs/.style={rv, fill=black!14},
  ctx/.style={rv, double, double distance=0.9pt},
  prod/.style={circle, draw, thick, minimum size=4.6mm, inner sep=0pt, font=\scriptsize},
  des/.style={rectangle, draw, thick, minimum size=6.2mm, inner sep=1.5pt, font=\small},
  edge/.style={-{Latex[length=2mm]}, thick},
  cut/.style={-{Latex[length=2mm]}, thick, dashed, red!75!black},
  weak/.style={-{Latex[length=2mm]}, thick, densely dotted, black!60},
  induced/.style={{Latex[length=1.8mm]}-{Latex[length=1.8mm]}, thick, densely dotted, blue!70!black},
  lab/.style={font=\scriptsize, inner sep=1.5pt},
}

\usepackage[textsize=tiny]{todonotes}

\begin{document}

%

%
\runningauthor{Hedman, Alger, Lehmann, Holmes, Rainforth}

\twocolumn[

\aistatstitle{FairBED: Incorporating Fairness into Bayesian Experimental Design}

\aistatsauthor{ Marcel Hedman$^{1}$ \qquad Emily Alger$^{2}$ \qquad Brieuc Lehmann$^{3}$ \qquad Chris Holmes$^{1}$ \qquad Tom Rainforth$^{1}$ }

\aistatsaddress{ $^{1}$Department of Statistics, University of Oxford \\
                 $^{2}$CeBAM, Nuffield Department of Medicine, University of Oxford \\
                 $^{3}$Department of Statistical Science, University College London } ]

\begin{abstract}
\looseness=-1
Bayesian experimental design (BED) has traditionally focused exclusively on collecting data that is maximally informative, with little consideration of other desirable dataset characteristics, such as fairness. 
We therefore extend BED to target data that is both informative and \emph{fair}.
To this end, we introduce FairBED, which provides a formalism for the fairness of datasets themselves by measuring how compatible they are with a \emph{fair world}, in which the target quantity we wish to learn about is unrelated to the sensitive attributes we want to be fair to. 
We then use this to construct practical fairness-aware BED objectives that maximise expected information gain about the target variables, while penalising information about sensitive attributes and/or their coupling with the target.
We further extend FairBED to active learning, and show empirically that models
trained on data gathered using FairBED provide improved fairness--accuracy trade-offs compared to randomly acquired data and conventional BED.
\end{abstract}

\newcommand{\MH}[1]{\textcolor{purple}{#1}}
\newcommand{\EA}[1]{\textcolor{blue}{#1}}


\section{Introduction}
\label{sec:intro}
\begin{figure*}[t]
\centering
\captionsetup[subfigure]{justification=centering}
\newcommand{\gmbox}{\useasboundingbox (-0.55,-2.0) rectangle (2.75,1.05);}
\begin{subfigure}[t]{0.24\textwidth}
\centering
\begin{tikzpicture}
  \gmbox
  \node[rv]  (th) at (0,0)     {$\theta$};
  \node[rv]  (ph) at (2.2,0)   {$\phi$};
  \node[obs] (D)  at (1.1,-1.5) {$\Data$};
  \draw[edge] (th) -- (D);
\end{tikzpicture}
\caption{Fair and private world:\\ $\Lu=0$, $\Cpl=0$.}
\label{fig:gm_fair}
\end{subfigure}\hfill
\begin{subfigure}[t]{0.24\textwidth}
\centering
\begin{tikzpicture}
  \gmbox
  \node[rv]  (th) at (0,0)     {$\theta$};
  \node[rv]  (ph) at (2.2,0)   {$\phi$};
  \node[obs] (D)  at (1.1,-1.5) {$\Data$};
  \draw[edge] (th) -- (D);
  \draw[edge] (ph) -- (D);
  \draw[induced] (th.north east) to[bend left=38] node[lab, above] {$\Cpl$} (ph.north west);
\end{tikzpicture}
\caption{Generic:\\ $\Lu>0$, $\Cpl>0$.}
\label{fig:gm_generic}
\end{subfigure}\hfill
\begin{subfigure}[t]{0.24\textwidth}
\centering
\begin{tikzpicture}
  \gmbox
  \node[rv]   (th) at (0,0)      {$\theta$};
  \node[rv]   (ph) at (2.2,0)    {$\phi$};
  \node[prod] (x)  at (1.1,-0.62) {$\times$};
  \node[obs]  (D)  at (1.1,-1.5)  {$\Data$};
  \draw[edge] (th) -- (x);
  \draw[edge] (ph) -- (x);
  \draw[edge] (x) -- (D);
  \draw[induced] (th.north east) to[bend left=38] node[lab, above] {$\Cpl$} (ph.north west);
\end{tikzpicture}
\caption{Interaction only:\\ $\Lu=0$, $\Cpl>0$.}
\label{fig:gm_blind}
\end{subfigure}\hfill
\begin{subfigure}[t]{0.24\textwidth}
\centering
\begin{tikzpicture}
  \gmbox
  \node[rv]  (th) at (0,0)     {$\theta$};
  \node[rv]  (ph) at (2.2,0)   {$\phi$};
  \node[obs] (Dt) at (0,-1.5)  {$\Data_\theta$};
  \node[obs] (Dp) at (2.2,-1.5) {$\Data_\phi$};
  \draw[edge] (th) -- (Dt);
  \draw[edge] (ph) -- (Dp);
\end{tikzpicture}
\caption{Factorised likelihood:\\ $\Lu>0$, $\Cpl=0$.}
\label{fig:gm_factorised}
\end{subfigure}
\vspace{-2pt}
\caption{\textbf{What data can reveal.}
The fair-world prior has no edge between $\theta$ and $\phi$.
(a)~In a fair and private world, $\phi$ has no bearing at all on $\Data$, so $\Cpl=\Lu=0$.
(b)~In general $\Data$ both reveals $\phi$ and couples $\theta$ and $\phi$ once observed (dotted line).
(c)~If $\phi$ enters only through an interaction with $\theta$, e.g.~$\Data=\phi\theta+\varepsilon$, the data are sensitive-blind yet teach the relationship.
(d)~A factorised likelihood reveals $\phi$ without coupling it to $\theta$.
\vspace{-7pt}
}
\label{fig:graphical_models}
\end{figure*}

\looseness=-1
Bayesian experimental design (BED) provides a powerful framework for gathering data that is maximally informative about variables or tasks of interest~\citep{lindley1956,lindley1972,chaloner1995,ryan2014towards,rainforth2024modern,huan2024optimal}.
This is often framed as choosing designs that maximally reduce uncertainty under a given model, such that our priority is to learn as fast as possible.

However, informativeness is not the only desirable characteristic we have for datasets.  
In particular, in impactful domains such as healthcare \citep{health-bias}, online advertising \citep{advertising-bias}, and criminal justice \citep{chouldechova2017fairpredictiondisparateimpact}, datasets have the potential to perpetuate societal biases, such that models trained on them 
can produce unfair predictions and decisions with real negative consequences. 
While various downstream mitigation strategies have been introduced to try and address this---such as data reweighting \citep{fairness-reweighting, krasanakis2018adaptive}, constrained optimisation during training \citep{fair-representations}, and pre-processing or “data-repair” methods \citep{removing-disparate-impact, fair-ml-survey}---such methods will inevitably be hindered by working with data that is itself inherently unfair~\cite{he2025fairness,shen2022metric,sharaf2022promoting}.

We argue, therefore, that fairness should also be prioritised at the data gathering stage.
In particular, we believe that there is a pressing need to extend optimal experimental design methods, like BED, to explicitly target fairer data, rather than solely seeking out the most informative data.
This, though, raises a surprisingly difficult, and underexplored, question: what does it mean for a \emph{dataset} to be fair?
While most of us would have some intuitive understanding of how a dataset might be unfair, there is currently a lack of formalisation for this in the literature, with most work focussed on the fairness of predictive models trained from data, rather than the fairness of the data itself.

\looseness=-1
To this end, we introduce FairBED, a Bayesian experimental design (BED,~\cite{lindley1956,lindley1972,chaloner1995,rainforth2024modern}) approach that formalises the utility of data in a way that balances the needs of effective learning and fairness.
Our starting point is the notion of a fair world: a hypothetical world in which the target quantity we wish to learn about, $\theta$, is unrelated to a sensitive attribute, $\phi$, we wish to be fair with respect to (e.g., a demographic identifier or a site-specific factor).  
We then take fair data to be data that do not teach a relationship between the two that such a world would not exhibit.
Specifically, we quantify the unfairness of data through the coupling it induces between $\theta$ and $\phi$, measured as their conditional mutual information given the data $\Cpl$.
This, in turn, encourages any models trained on such data to ignore $\phi$ when predicting $\theta$.
In particular, we show that if $\theta$ and $\phi$ are a priori independent and $\Cpl=0$, then downstream Bayes-optimal predictors of $\theta | \Data, \phi$ are invariant in $\phi$.

Sometimes it is also beneficial for data to convey little information about the sensitive variables, as measured by the information gain in $\phi$ from $\Data$, $\Lu$. This helps guard against adversarial downstream predictive models that might try and deliberately use $\phi$ to create unfair predictions.  If data leaves $\phi$ both hidden and uncoupled to $\theta$, then it is unhelpful to such an adversary.
\Cref{fig:graphical_models} shows the possible interactions and why data not providing information about $\phi$ does not mean it does not provide information about the coupling between $\theta$ and $\phi$ and vice versa.

We thus use $\Cpl$ and $\Lu$ to construct concrete BED objectives for practical fair data gathering.
Specifically, we learn data gathering policies that balance the expected information gained in the target variables $\Itheta$ against one or both of our expected fairness measures.
This allows us to target data that is both informative and fair.
Critically, this is not a strict trade-off, as data must be informative about what we want to predict to avoid large generalisation errors that could otherwise undermine both predictive capability and fairness.

We further extend FairBED to active learning, where $\theta$ and $\phi$ are the labels of a future individual. 
This setting creates additional complications as the inputs to our predictive model form a context, $\omega$, that can itself provide information about coupling between $\theta$ and $\phi$.
We find that this creates a trade-off between individual and group fairness considerations: minimising conditional coupling can be thought of as encouraging an individual fairness with any dependence encoded through $\omega$ deemed admissible.
Group fairness then follows if different groups induce the same distribution on $\omega$, but may require reducing the admissible context, or other adjustments, when group distributions vary.

\looseness=-1
Empirically we find that FairBED can improve trade-offs between informativeness and fairness of datasets compared to randomly acquired data and standard BED.
In particular, we find that regularising against either fairness criterion can reduce the respective leakage of these terms at test time.
We find that one can also sometimes successfully regulate both factors by minimising for their expected sum over data.

\section{Background}
\label{sec:background}


\looseness=-1
\textbf{Bayesian Experimental Design (BED)}~\citep{lindley1956,rainforth2024modern,huan2024optimal} is a principled approach to data gathering wherein designs $\xi\in\mathcal{X}$ are selected to maximise some expected utility of the generated data.
It postulates some Bayesian model, $p(\theta)p(y\mid \theta, \xi)$, over hypothetical possible observations, $y\in \mathcal{Y}$, and some target variables of interest $\theta$, which might themselves represent a real-world quantity, model parameters, or downstream predictions and decisions we wish to make.
This model is then used to both simulate possible data and evaluate the utility of possible data.
The latter is done by defining a loss function, $\ell(\theta,\xi,y)$.
The optimal design is then the one that minimises the \emph{expected loss} over possible observations:
\begin{align}
\label{eq:minimal_expected_loss}
\xi^* = \arg\min\nolimits_{\xi \in \mathcal{X}} \mathbb{E}_{p(\theta)p(y \mid \theta,\xi)}\!\left[ \ell(\theta,\xi,y) \right].
\end{align}
A canonical choice of loss is the log loss of the posterior~\citep{rainforth2024modern}, $\ell(\theta,\xi,y)=-\log p(\theta | \xi, y)$.
Minimising the resulting expected loss is then equivalent to maximising the \emph{expected information gain} (EIG) in $\theta$~\citep{lindley1956},
\begin{equation} 
\mathcal{I}_{\theta;y}(\xi) = \mathbb{E}_{p(y \mid \xi)}\left[ \mathcal{H}(p(\theta))-\mathcal{H}(p(\theta\mid\xi,y))\right], 
\label{eq:standard_eig_objective}  
\end{equation}
where the information gain itself, $\mathrm{IG}_\theta(\xi,y)=\mathcal{H}(p(\theta))-\mathcal{H}(p(\theta\mid\xi,y))$, represents the reduction in Shannon entropy from prior to posterior.
Recently there has been increasing interest in approaches that can target more general loss functions based on the data's informativeness for downstream tasks~\citep{Rossa2026ActionBEDTB,go2026goal,huang2024amortized,huang2026loss,filstroff2021targeted}.
Our work will continue this trend of going beyond the EIG as an objective for BED.

BED can also be extended to sequential settings~\citep{mackay1992information,cavagnaro2010adaptive,foster2021dad}, where designs are adaptively selected conditioned on past observations.
Traditionally, this is done by updating the model with the history $h_{t-1} = \{(\xi_1, y_1), \dots, (\xi_{t-1}, y_{t-1})\}$, replacing the prior $p(\theta)$ in~\Cref{eq:minimal_expected_loss} by current beliefs $p(\theta|h_{t-1})$ when choosing $\xi_t$. 
Some modern methods have instead looked to learn explicit neural network policies,  $\pi : h_{t-1} \mapsto \xi_t$ upfront that directly map from histories to next design decisions, and are trained to minimise the final expected loss~\citep{foster2021dad,ivanova2021implicit,blau2022optimizing,hedman2025stepdad,Huang2025ALINEJA,Rossa2026ActionBEDTB,phillips2026bayesian,huang2026efficient}. 

\looseness=-1
\textbf{Fairness in Machine Learning.}~
What it means for a model to be fair has been formalised in a variety of, often incompatible, ways~\citep{fair-ml-survey, Mitchell2021AlgorithmicFC}. \textit{Group} notions compare statistics of
the predictions across groups defined by a sensitive
attribute $\phi$ (e.g., gender or age), as in demographic
parity, equalised odds and predictive parity \citep{fair-representations, chouldechova2017fairpredictiondisparateimpact, NIPS2016_9d268236} (\cref{app:fairness definitions}). \textit{Individual} notions instead require that similar individuals be treated similarly \cite{dwork2011fairnessawareness}, with similarity defined by a task-specific metric \citep{shen2022metric, Lahoti2018iFairLI, Yurochkin2020Training} or, in counterfactual and causal formulations, by comparing an individual with their counterpart differing only in the sensitive attribute \citep{chiappa2019path, discrim_causal, kusner2017counterfactual}; see \citep{group_individual_conflict, fair_about_IF} for how the two families relate. These notions assess the fairness of predictions made
by a particular model. FairBED instead provides a notion rooted in the fairness of datasets themselves.
It targets individual notions of fairness, but can also induce group-level fairness as well in some circumstances (see~\cref{subsec:guarantees}).

\vspace{-5pt}
\section{Fairer Data Acquisition}
\label{sec:method}
\vspace{-3pt}
Existing BED methods target the informativeness or downstream utility of the data they collect, without explicitly considering the fairness of the resulting dataset.
Our aim is to introduce a framework that quantifies the fairness of the data, and then explicitly targets fairer data by appropriately controlling the acquisition process. Our formulation, which is focused around different invariances to
sensitive attributes, is not the only way one might measure fairness, 
but provides an important first step towards gathering fairer data.

\subsection{Fair Worlds}
\label{subsec:fairness}
\vspace{-3pt}

\looseness=-1
As a motivating example, consider acquiring data to estimate insurance risk, which forms our target variable $\theta$, while being fair to people's skin pigmentation, which thus forms our sensitive attribute, $\phi$.
There is no {direct} causal link from pigmentation to risk, so in an idealised setting a fair downstream predictor would not use $\phi$ when estimating $\theta$~\citep{Mitchell2021AlgorithmicFC}.
Unfortunately, as pigmentation may correlate with observed outcomes in real-world settings due to societal inequities, predictions and inferences based on a naively gathered dataset are unlikely to be naturally fair. Instead, they will reflect the unfairness of the data gathered from an unfair world.

\looseness=-1
Imagine, though, a hypothetical alternative world where these societal biases do not exist: pigmentation is independent of risk or observations thereof.
Models trained on data from this world should naturally preserve
this independence, as the data itself provides no evidence against it.
Our key insight is that this gives rise to a natural notion of dataset fairness: fair data are data that a fair world could have generated, such that it implies no coupling between $\theta$ and $\phi$ (see~\Cref{fig:graphical_models}).

\looseness=-1
Data gathered from the real world may fall short of this by teaching a relationship between the target and the sensitive attribute
that no fair world would exhibit. 
In other words, our data may leak unwanted information about the unfairness of the world, capturing dependencies that are, for example, a result of societal inequities and we wish not to perpetuate.
Our ultimate aim is thus to design procedures that minimise this information leakage and produce data that is as compatible as possible with a fair world.

\looseness=-1
This raises a complication: we must gather data in the real, unfair world, yet want it to be consistent with a fair world. 
Moreover, we need to take actions in the real world, so our data still needs to guide us towards accurate \emph{real-world predictions and inferences}.
While there will inevitably be tension between these desires, akin to the standard trade-off between fairness and performance~\citep{delbarrio2020reviewmathematicalframeworksfairness}, they critically do \emph{not} form a zero-sum game and it is possible to achieve beneficial trade-offs.
First, since we only control the designs $\xi$, never the outcome mechanism $y \mid \xi$, every observation is a genuine observation of the real world; what we choose is which observations to make, and we can focus on observations that are relatively more helpful for prediction than revealing unfairness.
Second, there is not a \emph{single} hypothetical fair world: we can seek data representative of the fair world that is closest to the real world we need to predict in.
As discussed in~\Cref{sub:fairbed-objective}, we achieve this in practice by taking data that is maximally informative about what we want to predict, while being minimally informative about what we want to avoid leaking.

Note that avoiding using $\phi$ at test time is not sufficient to achieve fair predictions, as the data itself, or other available context, may provide information about $\phi$, allowing a dependency to be induced indirectly.
In some settings we may therefore also wish for our data to be minimally informative about $\phi$, so that it cannot be indirectly recovered and then used for unfair predictions.
This is particularly important when the data is released or reused outside of our control, as adversarial users may try to deliberately create unfair models from the data.
We can undermine such usage by ensuring that our data is both fair and \emph{sensitive-blind}, in the sense that it maintains the privacy of the sensitive attributes.

\vspace{-3pt}
\subsection{Formalising Dataset Fairness}
\label{subsec:formalising_data_fairness}
\vspace{-3pt}

To formalise our notion of the fairness of a dataset
$\Data$, we consider what information it reveals about
sensitive attributes, $\phi$, and their relationship with
some target variables we do want to learn about $\theta$.

\looseness=-1
We work with a Bayesian model over a target quantity $\theta$, a sensitive quantity $\phi$, the data $\Data$ generated by our acquisition policy $\pi$ (e.g.~fixed designs or a learned policy), and a \emph{context} $\omega$ comprising whatever is taken as known when the data are assessed (see~\Cref{subsec:context}),
\begin{equation}
    p(\omega)\,p(\theta\mid\omega)\,p(\phi\mid\omega)\,p(\Data\mid\theta,\phi,\omega,\pi).
    \label{eq:model}
\end{equation}
The factorised prior over $\theta$ and $\phi$ encodes the fair world: conditioned on $\omega$, we do not a priori believe the target and the sensitive attribute to be related, and our aim is to gather data that leave this belief intact.
This independence is always conditional on the context, which in active learning holds the admissible inputs of an individual, but is often empty in standard BED settings.
The likelihood describes how the true, possibly unfair, outcome mechanism may depend on both quantities; the policy, $\pi$, decides which observations are made, never how outcomes are generated.
Note that $\theta$ and $\phi$ are assumed to not be directly observed in $\Data$: they typically correspond to population-level variables, underlying parameters, or future predictions.

\looseness=-1
We measure how much a dataset is incompatible with a fair world by the \emph{coupling} it induces between $\theta$ and $\phi$, namely their \emph{mutual information} under the posterior:
\begin{equation}
    \Cpldw:=\mathbb{E}_{p(\theta\mid\Data,\omega)}\left[\mathrm{KL}(p(\phi\mid\theta,\Data,\omega)\|p(\phi\mid\Data,\omega))\right].
    \label{eq:components}
\end{equation}
We refer to data for which $\Cpldw=0$ as \emph{coupling-free}.
This occurs if and only if  knowing $\phi$ does not change our beliefs on $\theta$ given the context, $p(\theta,\phi\mid\Data,\omega)=p(\theta\mid\Data,\omega)\,p(\phi\mid\Data,\omega)$, such that the data has taught us nothing about how $\theta$ and $\phi$ relate beyond what $\omega$ explains.
This is a Bayesian form of \emph{individual fairness}~\citep{dwork2011fairnessawareness,kusner2017counterfactual,corbett2017algorithmic} (\Cref{subsec:guarantees}).

\looseness=-1
The \emph{marginal leakage} $\Luw:=\mathcal{H}(p(\phi\mid\omega))-\mathcal{H}(p(\phi\mid\Data,\omega))$ is the information the data reveal about $\phi$ itself.
We refer to data that leave our beliefs about $\phi$ unchanged, $p(\phi\mid\Data,\omega)=p(\phi\mid\omega)$, and hence have $\Luw=0$, as \emph{sensitive-blind}; such data maintain the privacy of the sensitive attribute itself.

\looseness=-1
At acquisition time $\Data$ and $\omega$ do not yet exist, so we evaluate the policy by the expected values,
\begin{align}
\label{eq:coupling_policy}
    \Cplpi&:= \mathbb{E}_{p(\omega,\Data;\pi)}\left[\Cpldw\right], \\
    \Lupi&:=\mathbb{E}_{p(\omega,\Data;\pi)}\left[\Luw\right]
    \label{eq:leakage_policy}
\end{align}
where $p(\omega,\Data;\pi)$ is the corresponding marginal of~\eqref{eq:model}
and our $\Omega$ notation indicates we take an expectation over $\omega$.
Note that $\pi$ enters only through our expectations: $\Cpldw$ and $\Luw$ are independent of our policy given $\Data$.
These policy-level fairness metrics can also ensure the fairness of individual datasets as
\begin{align}
\label{eq:coupling_certification}
    \Cplpi =0 &\implies \Cpldw=0 \quad \mathrm{a.s.}\\
    \Lupi =0 &\implies \mathbb{E}_{p(\Data|\omega;\pi)}\left[\Luw\right] =0  ~~\mathrm{a.s.}
    \label{eq:sensitive_certification}
\end{align}
These follow from the fact that mutual informations (and  EIGs) are non-negative, so if the expected conditional mutual information is zero each individual conditional mutual information is also almost surely zero.

More precisely, each fairness metric gives its own guarantee as follows (proofs in \Cref{app:proof_guarantee}):
\begin{proposition}
\label{prop:guarantee}
Under the fair-world prior, for almost every dataset $\Data$ and context $\omega$:
(i)~if $\Lupi=0$, the data reveal nothing about $\phi$, $p(\phi\mid\Data,\omega)=p(\phi\mid\omega)$;
and if $\Cplpi=0$, then both
(ii)~for every loss $\ell(a,\theta,\omega)$, the Bayes actions of a decision maker who knows $\phi$ coincide with those of one who does not, and
(iii)~every posterior predictive $p(z\mid\Data,\omega,\phi)=\int p(z\mid\theta,\omega)\,p(\theta\mid\Data,\omega,\phi)\,d\theta$ of a model $p(z\mid\theta,\omega)$ that does not take $\phi$ as input equals $p(z\mid\Data,\omega)$.
\end{proposition}
\vspace{-4pt}
\looseness=-1
In~\Cref{prop:approx} we further show that the guarantees degrade gracefully when a penalty is small rather than zero, with the expected total-variation distance between the predictions of users with and without knowledge of $\phi$ being at most $\sqrt{\Cpldw/2}$.

\looseness=-1
It is also interesting to note that, under the fair-world prior for any given $\omega$, the expected conditional leakage in $\phi$ over possible $\theta$  
is equal to the sum of the coupling and marginal leakage, as formalised below.
\begin{proposition}
\label{lem:decomposition}
Under the fair-world prior, $\forall \Data, \omega$
\begin{align}
\begin{split}
    \mathbb{E}_{p(\theta|\Data,\omega)}&\left[\mathrm{IG}_{\phi|\theta,\omega}(\Data)\right] = \Cpldw + \Luw.
\end{split}
\end{align}
Consequently, $\Cplpi+\Lupi\!=\!\Lcpi ~\forall \pi$, where
\begin{align}
\label{eq:conditional_fairness}
    \Lcpi := \mathbb{E}_{p(\theta,\omega)}\left[\Lcwpi\right].
\end{align}
\end{proposition}
\looseness=-1
The proof is provided in~\Cref{app:proof_decomposition}, where we also expand the result to unfair priors.
Its implication is that if we can achieve $\Lcpi=0$, that is data that is never informative about $\phi$ \emph{given} $\theta$, this produces policies that are both coupling and leakage-free for all $\omega$.

We use the shorthands $\Fc$, $\Fphi$, and $\Fcond$ for $\Cplpi$, $\Lupi$, and $\Lcpi$ respectively from here on.

\subsection{The FairBED Objectives}
\label{sub:fairbed-objective}

\looseness=-1
Penalties alone are not appropriate design objectives, as they are trivially optimised by gathering no data at all.
FairBED thus trades a utility for the target, for which we use the canonical $\mathrm{EIG}_\theta(\pi)=\Ithetapi$, against one or both of the components by using the penalisation $\mathcal{I}\in\{\Fc,\Fphi,\Fcond\}$,
\begin{equation}
    R_{\mathcal{I}}(\pi)=\mathrm{EIG}_\theta(\pi)-\beta\,\mathcal{I}(\pi),
    \label{eq:fairbed_family}
\end{equation}
where $\beta>0$ controls the level of fairness enforcement.
We refer to these objectives as $\Rc$ (coupling), $\Rphi$ (marginal), and $\Rcond$ (conditional), choosing whichever is the appropriate penalisation for our particular problem.
Note here that $\mathrm{EIG}_\theta(\pi)\neq \mathcal{I}_{\theta;\Data | \Omega, \pi}$ in general, as it looks at the EIG in the marginals, rather than the expected conditional EIG given $\omega$,
allowing it to implicitly target data that is also helpful in identifying the correct model (which in turn improves prediction of $\theta$).

\looseness=-1
As discussed in~\Cref{subsec:fairness}, jointly targeting $\mathrm{EIG}_\theta(\pi)$ and the fairness penalisations is crucial, because without $\mathrm{EIG}_\theta(\pi)$ our data may no longer be informative about, and representative of, the real world.
All three penalties are (expected) EIGs so the objectives can be estimated and optimised with standard BED machinery, including nested and contrastive Monte Carlo bounds~\citep{foster2020unified,foster2021dad} and policy-based approaches for adaptive design (see~\Cref{app:app_fairbed_bed_detail}).
We can also exploit that $\Fc=\Fcond-\Fphi$ to indirectly estimate terms if desired.

\looseness=-1
\textbf{Early stopping.} An important subtlety with the FairBED objectives is that data might be actively harmful if the fairness penalty it incurs is higher than the gain in predictive performance.
This can be controlled by noting that gathering no data always yields $R_{\mathcal{I}}(\emptyset)=0$.
One can utilise this to perform early stopping when the optimal design yields a negative objective.

\subsection{The Context}
\label{subsec:context}

\looseness=-1
All of the above is conditional on the context $\omega$.
It plays two roles: our fairness penalisations try to impose invariance for every value of $\omega$, and any dependence between $\theta$ and $\phi$ that $\omega$ explains is not counted as coupling but declared \emph{admissible}.
Three kinds of information belong in it.
First, it may contain \emph{side information} that users of the data will also hold, such as additional data, which may aid them in recovering sensitive information.
This is generally dictated by the problem setting rather than our modelling choices.

\looseness=-1
Second, it can be used to express uncertainty about the plausible \emph{assessment models} that might later be applied to the data, e.g.~indexing over a class of possible models.
This is important because it allows us to improve robustness to the model specification and helps ensure that any fairness certifications are not simply the result of deficiencies in our model.
Note here that by~\Cref{eq:coupling_certification,eq:sensitive_certification}, if we manage to drive our expected fairness penalties to zero, this ensures fairness across all $\omega$, and thus for all models.

Finally, $\omega$ should contain the admissible \emph{inputs}, or covariates, that will be used downstream for prediction.
For many BED problems this will be empty, but in active learning settings this is not the case (see~\Cref{sec:active_learning}).
Here the decision of what to include in context becomes an important modelling decision.
In particular, because couplings induced directly through $\omega$ are admissible, we should avoid anything that is a proxy for $\phi$ in $\omega$
and carefully choose what is admissible.
As we now discuss, further restricting $\omega$ can also help balance between individual- and group-level fairness.

\subsection{Individual vs Group Fairness}
\label{subsec:guarantees}

\looseness=-1
As (ii) and (iii) in~\Cref{prop:guarantee} hold pointwise in the context, when $\omega$ contains the inputs of an individual, two individuals with the same inputs receive the same predictions and Bayes-optimal decisions whatever their group, and no user, even one who knows the group, has reason to treat them differently.
This is the requirement of individual fairness that those who differ only in their group be treated alike~\citep{dwork2011fairnessawareness,Lahoti2018iFairLI,Yurochkin2020Training}.
Group-level fairness follows when the context is empty---as (iii) is then exactly the demographic parity of the predictive---or when predictions are pooled over contexts that carry no information about the group.
Thus any disparity between groups is confined to, and bounded by, differences between the groups' contexts (\Cref{prop:deployment} in \Cref{app:pooling}).
The choice of context is thus a trade-off between individual and group fairness: a richer admissible context strengthens the former but admits any disparity its variables carry, so group fairness may require a smaller context or further adjustments.

\vspace{-3pt}
\subsection{Active Learning}
\label{sec:active_learning}
\label{sub:fairbed_in_al}
\vspace{-3pt}

\looseness=-1
Active learning (AL) is an important special case of experimental design in which we sequentially select inputs $x$ to label, in order to learn a predictive model $p(y\mid x)$~\citep{mackay1992information,settlestr09}.
Each datum now carries a target label $y_\theta$ and a sensitive label $y_\phi$, so acquiring at $x$ returns the pair $y=(y_\theta,y_\phi)$. 
The fair-world prior is then that an individual's labels are conditionally independent given their inputs; the marginal leakage is on predicting a new individual's group from their inputs; and the coupling is the value of knowing the group when predicting the outcome from the inputs, i.e.~a \emph{direct effect} of the group that the data have taught.

\looseness=-1
This setup can now be directly encoded within the FairBED framework by using the \emph{expected predictive information gain} (EPIG~\citep{pmlr-v206-bickfordsmith23a,BickfordSmith2024MakingBU}) as the basis for our acquisitions, which equates to specifying a test-time distribution on admissible inputs $x_*\sim p_*$, then taking $\theta=y_{*\theta}$, $\phi=y_{*\phi}$ and $\omega=(x_*,m)$.
Further explanations and discussions on this setting are provided in~\Cref{app:al_theory}.
Two points from this are of particular note.
First, the assessment model must be a \emph{joint} model of the label pair, as separate models for the two labels are incapable of capturing the coupling, making $\Fc$ trivially zero due to model deficiencies.
Second, since earlier acquisitions may already have induced coupling, the coupling penalty is the \emph{increase} in coupling caused by the new pair, with the per-step penalties summing to those of \Cref{sub:fairbed-objective} for the whole dataset (\Cref{lem:al_chain}).

\begin{figure*}[!t]
    \centering
    \includegraphics[width=\linewidth]{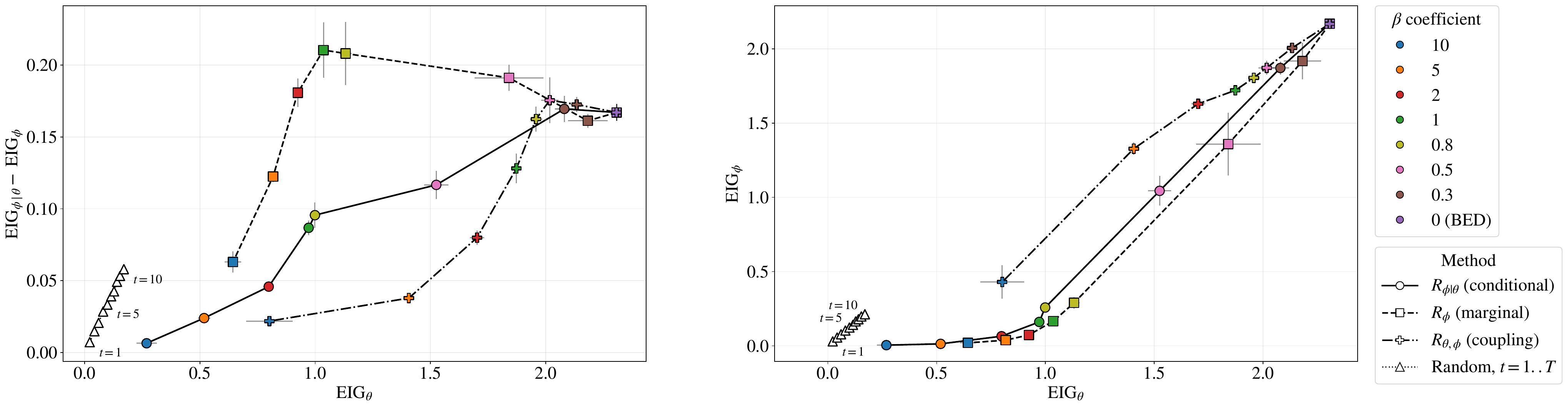}
    \caption{\textbf{Location finding (DAD, $T{=}10$).}
    $\mathrm{EIG}_\theta$ against the coupling $\Fc=\mathrm{EIG}_{\phi\mid\theta}-\mathrm{EIG}_\phi$ (left) and the marginal leakage $\Fphi=\mathrm{EIG}_\phi$ (right) across $\beta$ for DAD policies trained with $\Rcond$, $\Rphi$ and $\Rc$; lower right is better.
    Standard BED equates to $\beta{=}0$.
    Lower-bound sPCE estimates; mean $\pm$ 1 s.e.\ over 8 seeds.}
    \label{fig:loc_find_frontier_coupling_dad}
\end{figure*}

\section{Related Work}
\label{sec:rel_works}

\textbf{Bayesian experimental design.}~~
To our knowledge, no previous work on BED has considered the fairness of the data it gathers, with previous works focusing only on maximising the informativeness of the data, whether that is measured through uncertainty or performance on downstream tasks~\citep{lindley1956,chaloner1995,ryan2016review,rainforth2024modern,huan2024optimal,Rossa2026ActionBEDTB,huang2024amortized,huang2026efficient}.

\looseness=-1
\textbf{Fairness mitigations in ML.}~~
Most mitigations act after data collection, by \emph{pre-processing} the inputs \citep{dwork2018decoupled, removing-disparate-impact, krasanakis2018adaptive}, learning \emph{fair representations} \citep{fair-representations, Locatello2019OnTF}, constraining training (\emph{in-processing}) \citep{NIPS2016_9d268236, agarwal2018reductions, zafar2017fairness, abraham2019fairness, Just-train-twice, martinez2020minimax}, or \emph{post-processing} outputs \citep{cruz2023unprocessing, kamiran2012decision}.
FairBED acts upstream of all of these and can be combined with them; \citet{Yang2024TheLO} show that gains from post-hoc interventions often fail to generalise in deployment, strengthening the case for upstream intervention.

\textbf{Active learning for fairness.}~~
\looseness=-1
Prior AL approaches target downstream models that meet a \emph{group} criterion, such as demographic parity or equalised odds, through fairness-aware query scores \citep{Fair-active-learning, FAL-CUR, pang2024fairness, sharaf2022promoting}, bandit formulations \citep{tae2024falconfairactivelearning}, or sampling towards the worst-off group \citep{abernethy2020active, shekhar2021adaptive, cai2022adaptive}.
These methods assume more information than our setting provides: the group of every candidate before it is labelled \citep{Fair-active-learning, FAL-CUR, abernethy2020active, shekhar2021adaptive}, or a group-labelled validation set beyond the label budget \citep{pang2024fairness, sharaf2022promoting}.
In FairBED, as in surveys that record demographics with the responses, a candidate's group is revealed only when it is labelled, and candidates are scored using only the predictive beliefs of its models.
The one individual-fairness method, metric-fair AL \citep{shen2022metric}, requires a task-specific similarity metric and a halfspace classifier, whereas FairBED treats individuals as alike when they share the admissible context.
More broadly, FairBED targets an individual-level property of the data rather than a group criterion of a model, which group parity can contradict (\Cref{subsec:guarantees}), and applies beyond AL to the general BED setting.

\section{Experiments}
\label{sec:experiments}

\looseness=-1
Throughout, we label each FairBED variant by its objective, $\Rc$, $\Rphi$ or $\Rcond$~\eqref{eq:fairbed_family}, with standard BED (EPIG in active learning) as the $\beta{=}0$ case; full details and additional results are given in \Cref{app:experiment_details}.

\subsection{Source Location Finding}
\label{subsec:source_loc_find}

\looseness=-1
We build on the source location finding benchmark of \citet{foster2021dad}, in which a hidden source emits a signal that decays with distance and each experiment places a sensor and records a noisy intensity.
The source lies at $(\theta,\phi)\in\mathbb{R}^2$, where $\theta$ is the target coordinate and $\phi$ the sensitive one, with independent priors.
In the original benchmark every sensor is isotropic, so a reading depends only on the distance to the source and cannot inform $\theta$ without also informing $\phi$.
We instead use collimated detectors: each design chooses a sensor position, a degree of collimation and an orientation, and a collimated sensor responds mainly to the offset of the source along its orientation, at the cost of noisier readings.
Sensors aligned with the target axis then measure $\theta$ almost independently of $\phi$, while rotated sensors confound the two (\cref{app:location_finding}).

\begin{figure*}[!t]
    \centering
    \includegraphics[width=\linewidth]{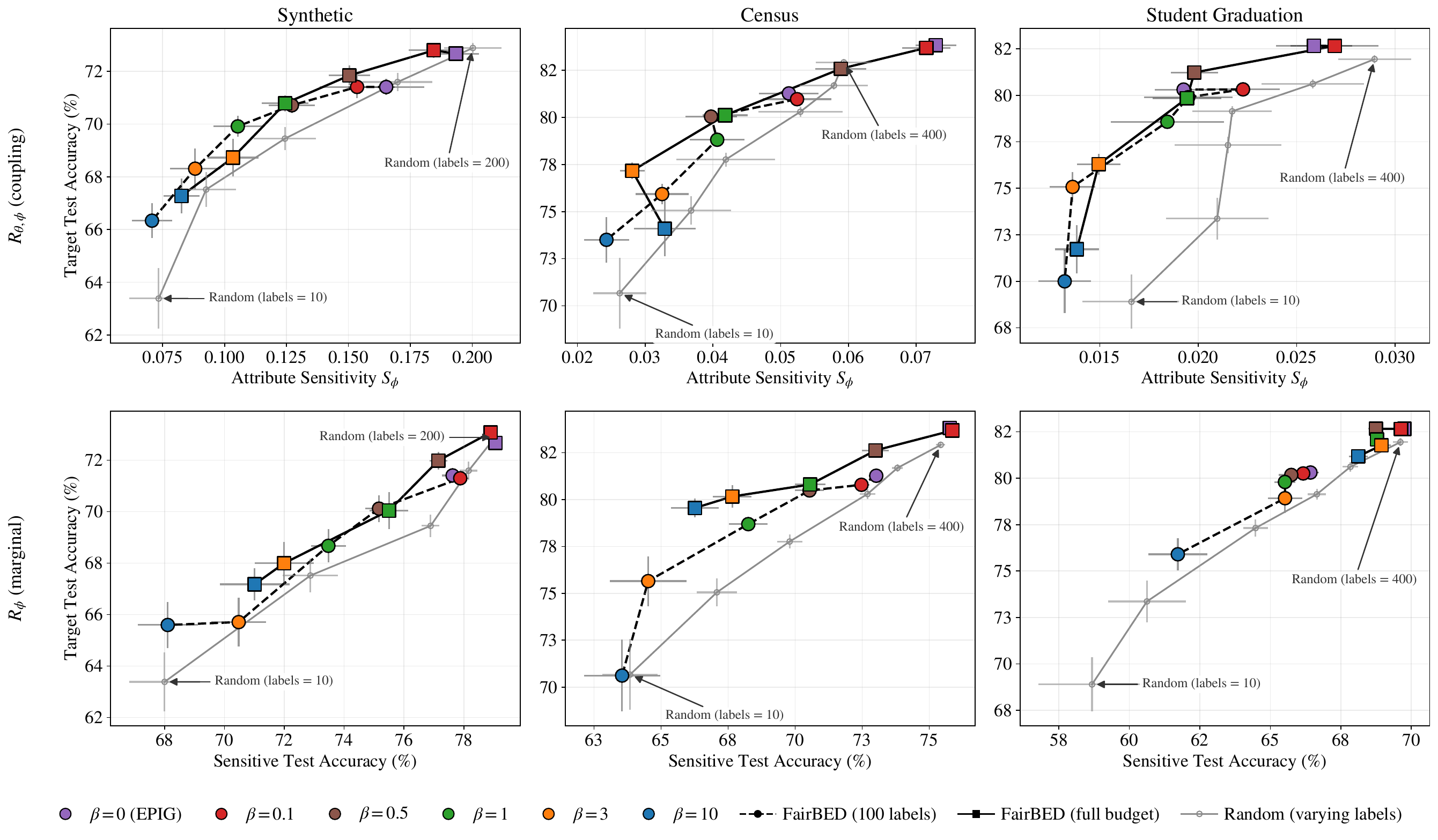}
    \vspace{-14pt}
    \caption{\textbf{Targeting a component of the leakage reduces it.} Target accuracy against attribute sensitivity $S_\phi$ for data acquired with $\Rc$ (top), and against sensitive test accuracy for data acquired with $\Rphi$ (bottom), across $\beta$; grey: Random at increasing label budgets. Mean $\pm$ s.e.\ over 24 seeds.}
    \vspace{-10pt}
    \label{fig:al_targeted_objectives}
\end{figure*}

The context is empty here, so $\Fphi=\mathrm{EIG}_\phi$, $\Fcond=\mathrm{EIG}_{\phi\mid\theta}$ and $\Fc=\mathrm{EIG}_{\phi\mid\theta}-\mathrm{EIG}_\phi$.
We train DAD and static policies~\citep{foster2021dad} for $T{=}10$ experiments with each objective over $\beta\in\{0.3,0.5,0.8,1,2,5,10\}$, and compare against standard BED ($\beta{=}0$), which maximises an sPCE lower bound on $\mathrm{EIG}_\theta$, and Random designs. Example selected designs are in~\cref{app:loc_find_results}.

\looseness=-1
\Cref{fig:loc_find_frontier_coupling_dad} plots $\mathrm{EIG}_\theta$ against each component of the leakage for the different FairBED objectives and baselines.
\textbf{Coupling (left).}
$\Rc$ attains the lowest coupling over the range of $\mathrm{EIG}_\theta$ it reaches, below $\Rcond$ in every seed.
It needs larger values of $\beta$, though, to achieve the same level of coupling as $\Rcond$.
$\Rcond$ in turn attains lower coupling than $\Rphi$ in every seed, and its coupling falls steadily with $\beta$, to $96\%$ below standard BED at $\beta{=}10$.
Both thus provide a beneficial trade-off between the objectives compared to random acquisition, with much higher $\mathrm{EIG}_\theta$ values.
Standard BED gives very high $\mathrm{EIG}_\theta$, but results in significantly increased coupling.
$\Rphi$ can produce even higher levels of coupling for intermediate values of $\beta$, demonstrating that simply trying to minimise information in $\phi$ is not sufficient, and can even be harmful, for reducing coupling.
Interestingly, it does start to reduce coupling again at very large values of $\beta$ though while still maintaining a $\mathrm{EIG}_\theta$ that is much higher than random.
\textbf{Marginal leakage (right).}
$\Rphi$ and $\Rcond$ both remove almost all of the marginal leakage for $\beta>2$ while still providing significant improvements in $\mathrm{EIG}_\theta$, with the expected ordering of $\Rphi$ providing the best trade-off, followed by $\Rcond$, and then $\Rc$.
Interestingly, all methods seem to provide a monotonic trade-off between $\mathrm{EIG}_\theta$ and leakage with varying values of $\beta$, with standard BED yielding very high leakage.
Each objective thus reduces the component it penalises, and only $\Rcond$ reduces both.

\subsection{Active Learning Evaluations}
\label{subsec:al_experiments}

\looseness=-1
\textbf{Setup.}~~
We use three binary tasks, each with a binary sensitive attribute: a \emph{Synthetic} population with a known direct effect of the group on the outcome (\Cref{app:al_details}); \emph{Census}, the class-balanced Census-Income (KDD) data~\citep{census-income_(kdd)_117} with income above \$50k as target and sex as sensitive attribute (pool $6{,}026$, test $4{,}932$); and \emph{Student Graduation}~\citep{predict_students_dropout_and_academic_success_697}, with graduation as target and gender as sensitive attribute (pool $2{,}433$, test $1{,}991$).
Each run starts from four random labels, one per $(y_\theta,y_\phi)$ combination, and acquires one label at a time up to $400$ ($200$ on Synthetic); we report means over $24$ seeds.
The predictive models are random forests of $100$ trees, each tree treated as one ensemble member, and all terms are estimated on $200$ inputs $x_*$ drawn from the pool.
The coupling is only visible to a joint model of $(y_\theta,y_\phi)$ (\Cref{prop:al_no_coupling}), so $\Rcond$ and $\Rc$ estimate their penalties with a single four-class forest over the label pairs, while $\Rphi$ uses a separate forest for the group; all three use the same target forest for $\mathrm{EPIG}_\theta$.
Baselines are Random, EPIG~\citep{pmlr-v206-bickfordsmith23a} and predictive entropy. A candidate's group is revealed only when it is labelled and no group-labelled validation set is available in our setting.

\begin{figure}[!t]
    \centering
    \includegraphics[width=\linewidth]{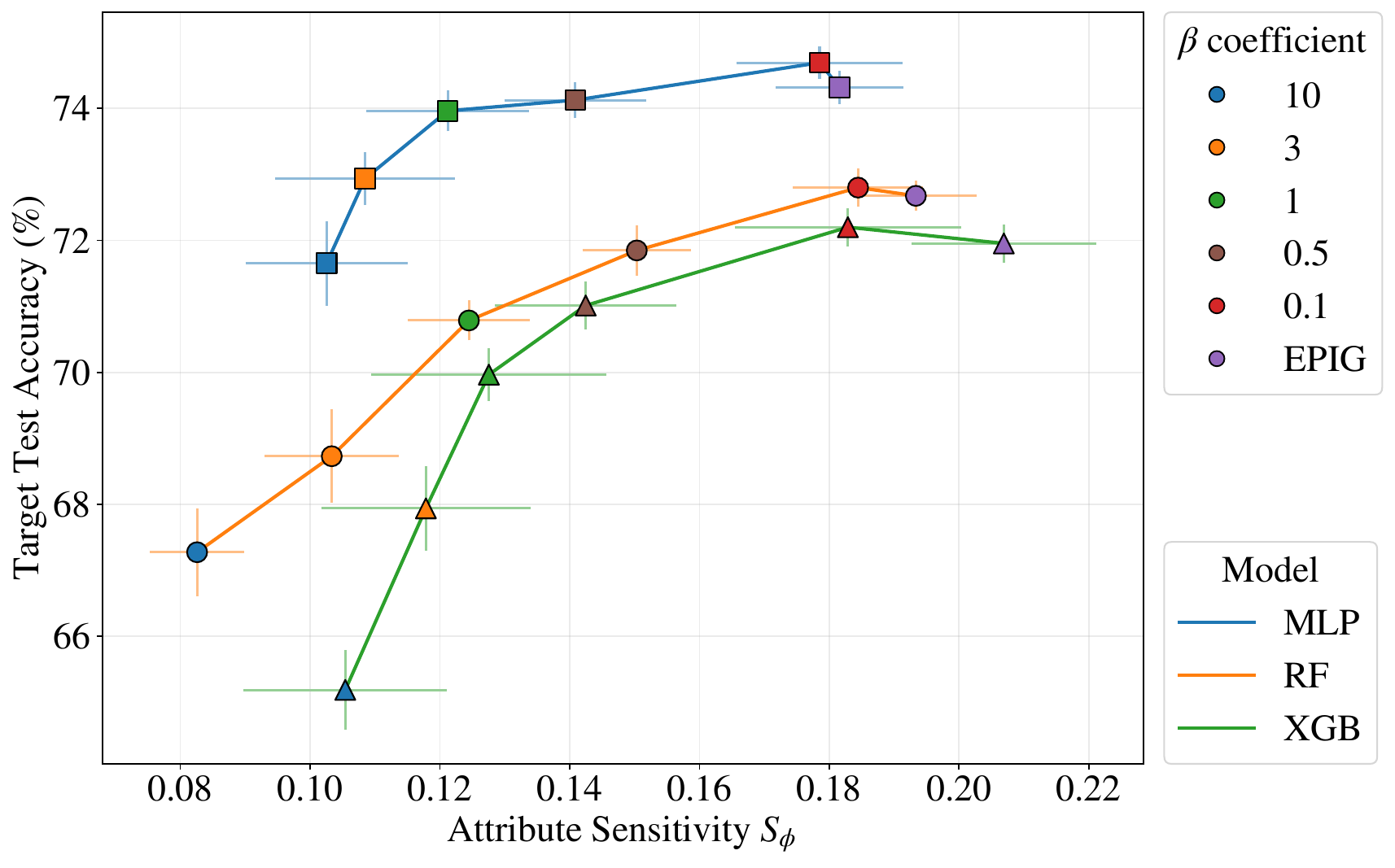}
    \caption{\textbf{Fairness transfers across model classes} (Synthetic, $\Rc$, $200$ labels). Target accuracy against attribute sensitivity for a random forest (the acquisition model class), an MLP and XGBoost, trained on the same data, across $\beta$. Mean $\pm$ s.e.\ over 24 seeds.}
    \label{fig:al_downstream}
\end{figure}

\begin{table*}[t]
\caption{\textbf{Active learning results} after the full budget ($200$ labels on Synthetic, $400$ on Census and Student Graduation), mean $\pm$ s.e.\ over 24 seeds, starting from one random label per $(y_\theta,y_\phi)$ combination.
Acc.: target accuracy; Attr.\ sens.: attribute sensitivity $S_\phi$; Adv.: sensitive test accuracy, of an adversary predicting $y_\phi$ from the inputs; DP gap: demographic-parity gap; all from random forests trained on the acquired data.}
\label{tab:al_joint}
\centering
\setlength{\tabcolsep}{3pt}
\resizebox{\textwidth}{!}{%
\begin{tabular}{l cccc cccc cccc}
\toprule
& \multicolumn{4}{c}{\textbf{Synthetic}}
& \multicolumn{4}{c}{\textbf{Census}}
& \multicolumn{4}{c}{\textbf{Student Graduation}} \\
\cmidrule(lr){2-5}
\cmidrule(lr){6-9}
\cmidrule(lr){10-13}
\textbf{Method}
& Acc.\,(\%)$\uparrow$ & Attr.\ sens.$\downarrow$ & Adv.\,(\%)$\downarrow$ & DP gap$\downarrow$
& Acc.\,(\%)$\uparrow$ & Attr.\ sens.$\downarrow$ & Adv.\,(\%)$\downarrow$ & DP gap$\downarrow$
& Acc.\,(\%)$\uparrow$ & Attr.\ sens.$\downarrow$ & Adv.\,(\%)$\downarrow$ & DP gap$\downarrow$ \\
\midrule
Random & $72.9 \pm 0.2$ & $0.200 \pm 0.012$ & $79.0 \pm 0.2$ & $0.232 \pm 0.008$ & $82.9 \pm 0.1$ & $0.059 \pm 0.004$ & $75.4 \pm 0.2$ & $0.244 \pm 0.003$ & $81.9 \pm 0.2$ & $0.029 \pm 0.002$ & $69.6 \pm 0.3$ & $0.166 \pm 0.002$ \\
Predictive entropy & $73.2 \pm 0.2$ & $0.189 \pm 0.013$ & $78.8 \pm 0.3$ & $0.193 \pm 0.006$ & $84.1 \pm 0.1$ & $0.086 \pm 0.004$ & $75.5 \pm 0.2$ & $0.181 \pm 0.003$ & $82.7 \pm 0.1$ & $0.032 \pm 0.003$ & $69.4 \pm 0.2$ & $0.108 \pm 0.002$ \\
EPIG & $72.7 \pm 0.2$ & $0.193 \pm 0.009$ & $79.1 \pm 0.2$ & $0.212 \pm 0.006$ & $83.8 \pm 0.1$ & $0.073 \pm 0.003$ & $75.8 \pm 0.2$ & $0.238 \pm 0.003$ & $82.7 \pm 0.1$ & $0.026 \pm 0.002$ & $69.8 \pm 0.2$ & $0.154 \pm 0.002$ \\
\midrule
$\Rcond$ ($\beta=1$) & $68.6 \pm 0.8$ & $0.094 \pm 0.009$ & $71.9 \pm 0.8$ & $0.130 \pm 0.009$ & $75.6 \pm 1.1$ & $0.030 \pm 0.003$ & $65.5 \pm 0.9$ & $0.178 \pm 0.008$ & $77.8 \pm 0.4$ & $0.017 \pm 0.001$ & $67.9 \pm 0.3$ & $0.155 \pm 0.004$ \\
$\Rcond$ ($\beta=10$) & $65.5 \pm 0.9$ & $0.082 \pm 0.010$ & $70.8 \pm 1.1$ & $0.113 \pm 0.011$ & $66.6 \pm 2.2$ & $0.025 \pm 0.006$ & $62.0 \pm 1.8$ & $0.116 \pm 0.010$ & $65.0 \pm 2.2$ & $0.014 \pm 0.002$ & $66.7 \pm 0.5$ & $0.099 \pm 0.005$ \\
$\Rphi$ ($\beta=1$) & $70.0 \pm 0.7$ & $0.145 \pm 0.012$ & $75.5 \pm 0.6$ & $0.164 \pm 0.008$ & $80.8 \pm 0.3$ & $0.061 \pm 0.005$ & $70.6 \pm 0.6$ & $0.218 \pm 0.004$ & $82.1 \pm 0.2$ & $0.027 \pm 0.002$ & $68.8 \pm 0.3$ & $0.167 \pm 0.002$ \\
$\Rphi$ ($\beta=10$) & $67.2 \pm 0.6$ & $0.114 \pm 0.011$ & $71.0 \pm 1.2$ & $0.128 \pm 0.010$ & $79.6 \pm 0.5$ & $0.055 \pm 0.005$ & $66.3 \pm 0.9$ & $0.179 \pm 0.008$ & $81.2 \pm 0.3$ & $0.027 \pm 0.003$ & $68.1 \pm 0.3$ & $0.155 \pm 0.003$ \\
$\Rc$ ($\beta=1$) & $70.8 \pm 0.3$ & $0.124 \pm 0.009$ & $76.0 \pm 0.4$ & $0.183 \pm 0.007$ & $80.1 \pm 0.3$ & $0.042 \pm 0.003$ & $71.9 \pm 0.3$ & $0.230 \pm 0.005$ & $79.8 \pm 0.4$ & $0.019 \pm 0.002$ & $68.7 \pm 0.3$ & $0.166 \pm 0.003$ \\
$\Rc$ ($\beta=10$) & $67.3 \pm 0.7$ & $0.083 \pm 0.007$ & $75.0 \pm 0.7$ & $0.138 \pm 0.008$ & $74.1 \pm 1.5$ & $0.033 \pm 0.005$ & $68.0 \pm 0.8$ & $0.171 \pm 0.011$ & $71.7 \pm 1.3$ & $0.014 \pm 0.001$ & $66.6 \pm 0.5$ & $0.120 \pm 0.005$ \\
\bottomrule
\end{tabular}%
}
\end{table*}

\looseness=-1
\textbf{Metrics.}~~
All metrics come from random forests trained on the acquired data and evaluated on the held-out test set.
The \emph{sensitive test accuracy}, of a forest predicting $y_\phi$ from the inputs, measures how much the data reveal about group membership, the marginal leakage.
The \emph{attribute sensitivity} $S_\phi=\mathbb{E}_{x_*}\big|p(y_\theta{=}1\mid x_*,y_\phi{=}1, \Data)- p(y_\theta{=}1\mid x_*,y_\phi{=}0, \Data)\big|$, computed with a target forest that also takes $y_\phi$ as input, measures the direct effect of the group that the data teach, i.e.~the coupling ($0$ means none).
We also report the demographic-parity (DP) gap of the target forest, $\big|\mathbb{E}[p(y_\theta{=}1\mid x_*,\Data)\mid y_\phi{=}1]-\mathbb{E}[p(y_\theta{=}1\mid x_*,\Data)\mid y_\phi{=}0]\big|$ to see how the data fare on a downstream \emph{group} metric that FairBED does not explicitly target.
Since all metrics change with the number of labels, we compare against Random at every budget from $10$ to $400$ labels.

\looseness=-1
\textbf{Results.}~~
Standard acquisition does not reduce the leakage: EPIG and predictive entropy leave the sensitive test accuracy within half a point of Random on every dataset, and raise the attribute sensitivity on Census (\Cref{tab:al_joint}).
\Cref{fig:al_targeted_objectives} shows that the objective used to acquire the data determines which leakage FairBED removes.
Acquiring data with $\Rc$ lowers the attribute sensitivity relative to Random at the same target accuracy on all three datasets, by up to $17\%$, $26\%$ and $33\%$ on Synthetic, Census and Student Graduation.
Acquiring data with $\Rphi$ lowers the sensitive test accuracy instead, by up to $5.5$ percentage points on Census and $1.8$ on Synthetic.
$\Rcond$ penalises both components and lowers both (\Cref{tab:al_joint}): at $\beta{=}10$ its attribute sensitivity and sensitive test accuracy are the lowest or within one standard error of the lowest, at a cost in target accuracy.
At smaller $\beta$, which component $\Rcond$ reduces most depends on the dataset (\Cref{app:al_details}).
Results with early stopping when the objectives go negative are provided in~\Cref{app:al_early_stopping}.

\looseness=-1

\subsection{Fairness as a property of the data}
\label{subsec:downstream}
FairBED makes fairness a property of the data rather than of a particular model, so it should hold for models other than the one used during acquisition.
\Cref{fig:al_downstream} trains an MLP and XGBoost on the data acquired by $\Rc$ with random forests on Synthetic.
The attribute sensitivity falls steadily with $\beta$ for all three model classes, by $43$ to $57\%$ from EPIG ($\beta{=}0$) to $\beta{=}10$, so the reduction transfers to distinct model classes that played no part in acquisition.
\section{Conclusion}
\looseness=-1
We introduced FairBED, which formalises the fairness of \emph{data} through its compatibility with a fair world in which the target is unrelated to the sensitive attribute.
It further operationalises this with BED objectives that trade the information gained about the target against the coupling the data induce between the two, the information they reveal about the attribute, or both.
Coupling-free data make the attribute irrelevant to every Bayes-optimal downstream decision in every context, an individual-level guarantee from which group-level fairness follows under conditions on the context.
Empirically, we find the FairBED objectives provide improvements on the fairness metrics they target on both BED and active learning tasks, with improved trade-offs compared to random acquisition and generalisation of the dataset fairness to different downstream models.

\clearpage

\subsection*{AI use statement}

We used generative AI tools for assistance with Propositions A.2 and A.4 in the appendix, and for checking and reviewing of notation, writing, and mathematical claims throughout. We did not use generative AI tools for any other tasks requiring disclosure.
We also used generative AI tools for creating figures and processing results. All AI-assisted content was reviewed and verified by the authors, who take full responsibility for the final work.

\subsubsection*{Acknowledgements}
MH is supported by funding provided by Novo Nordisk
and by the EPSRC Centre for Doctoral Training in Modern
Statistics and Statistical Machine Learning (EP/S023151/1).


\bibliography{refs}
\bibliographystyle{plainnat}



\clearpage
\newpage
\appendix
\onecolumn

\section{Proofs and Additional Theory}
\label{app:proofs}

Throughout, quantities are under the model~\eqref{eq:model} with the policy $\pi$ fixed and suppressed, and $I(\cdot;\cdot\mid\cdot)$ denotes a (conditional) mutual information under this model.
The context $\omega$ may contain an assessment-model index $m$, so every statement holds for each model in $\mathcal{M}$ separately.
When the context is suppressed we write $\Lu$, $\Cpl$ and $\Lc$ for the marginal leakage, coupling and conditional leakage of a realised dataset, as in \Cref{fig:graphical_models}.

\subsection{The Decomposition}
\label{app:proof_decomposition}

\begin{proof}[Proof of \Cref{lem:decomposition}]
The quantity on the left-hand side of the proposition is the posterior-averaged information gain about $\phi$ when $\theta$ is known,
\begin{equation}
    \mathrm{IG}_{\phi\mid\theta,\omega}(\Data):=\mathcal{H}\big(p(\phi\mid\theta,\omega)\big)-\mathcal{H}\big(p(\phi\mid\theta,\Data,\omega)\big),
    \label{eq:dataset_loss}
\end{equation}
averaged over $p(\theta\mid\Data,\omega)$.
Temporarily suppressing the dependence on $\omega$ to simplify notation (as everything is conditioned on it), for a realised dataset $\Data$ write $\mathcal{H}(\phi\mid\Data)$ for the entropy of $p(\phi\mid\Data)$ and $\mathcal{H}(\phi\mid\theta,\Data)$ for that of $p(\phi\mid\theta,\Data)$ at a given $\theta$, and similarly without $\Data$ for prior quantities.
Adding and subtracting $\mathcal{H}(\phi)$ and $\mathcal{H}(\phi\mid\Data)$ in~\eqref{eq:dataset_loss},
\begin{align*}
\mathbb{E}_{p(\theta\mid\Data)}\big[\mathrm{IG}_{\phi\mid\theta}(\Data)\big]
=\underbrace{\mathcal{H}(\phi)-\mathcal{H}(\phi\mid\Data)}_{\Lu}
+\underbrace{\mathcal{H}(\phi\mid\Data)-\mathbb{E}_{p(\theta\mid\Data)}\big[\mathcal{H}(\phi\mid\theta,\Data)\big]}_{\Cpl}
-\underbrace{\Big(\mathcal{H}(\phi)-\mathbb{E}_{p(\theta\mid\Data)}\big[\mathcal{H}(\phi\mid\theta)\big]\Big)}_{(\ast)}.
\end{align*}
The second brace equals the coupling~\eqref{eq:components}, because $p(\phi\mid\Data)=\mathbb{E}_{p(\theta\mid\Data)}[p(\phi\mid\theta,\Data)]$ gives $\mathbb{E}_{p(\theta\mid\Data)}[\mathrm{KL}(p(\phi\mid\theta,\Data)\,\|\,p(\phi\mid\Data))]=\mathcal{H}(\phi\mid\Data)-\mathbb{E}_{p(\theta\mid\Data)}[\mathcal{H}(\phi\mid\theta,\Data)]$.
Under the fair-world prior $\mathcal{H}(\phi\mid\theta)=\mathcal{H}(\phi)$ for every $\theta$, so $(\ast)=0$, which proves the first statement.
Taking expectations over $p(\Data,\omega;\pi)$, and using that $\mathbb{E}_{p(\Data)}\mathbb{E}_{p(\theta\mid\Data)}[\cdot]=\mathbb{E}_{p(\theta,\Data)}[\cdot]$ so that the expectation of the left-hand side is $\Lcpi$, gives the second statement.
In general, without the fair-world prior, $(\ast)$ is a posterior-weighted prior coupling whose expectation over $p(\Data\mid\omega)$ is $I(\theta;\phi\mid\omega)$, so
\begin{equation}
    \Lcpi=\Lupi+\Cplpi-\mathcal{I}_{\theta;\phi\mid\Omega},
    \qquad \mathcal{I}_{\theta;\phi\mid\Omega}:=\mathbb{E}_{p(\omega)}\big[I(\theta;\phi\mid\omega)\big],
    \label{eq:general_decomposition}
\end{equation}
equivalently, expanding $I(\phi;\Data,\theta\mid\omega)$ by the chain rule in both orders, $I(\phi;\theta\mid\omega)+I(\phi;\Data\mid\theta,\omega)=I(\phi;\Data\mid\omega)+I(\phi;\theta\mid\Data,\omega)$.
The conditional leakage is thus the marginal leakage plus the \emph{increase} in coupling induced by the data; this general form is the one that applies after earlier acquisitions in active learning (\Cref{sec:active_learning}).
\end{proof}

\subsection{The Sign-Flip Example}
\label{app:signflip}

Neither component of the conditional leakage implies the other (\Cref{fig:gm_blind,fig:gm_factorised}).
Let $\theta\sim\mathcal{N}(0,1)$ and, independently, $\phi\in\{-1,+1\}$ with equal probability; let $\Data=\phi\theta+\varepsilon$ with $\varepsilon\sim\mathcal{N}(0,1)$ independent of both, and let the context be empty.
Then $p(\Data\mid\phi)=\int\mathcal{N}(\theta;0,1)\,\mathcal{N}(\Data;\phi\theta,1)\,d\theta=\mathcal{N}(\Data;0,2)$ for both values of $\phi$, so $p(\Data\mid\phi)=p(\Data)$ and $\Lu=0$: the data are sensitive-blind.
The likelihood $p(\Data\mid\theta)=\tfrac12\mathcal{N}(\Data;\theta,1)+\tfrac12\mathcal{N}(\Data;-\theta,1)$ depends on $|\theta|$, so $\mathrm{EIG}_\theta>0$.
Given $\phi$, conjugacy gives $p(\theta\mid\Data,\phi)=\mathcal{N}(\theta;\phi\Data/2,\tfrac12)$, whereas $p(\theta\mid\Data)=\tfrac12\mathcal{N}(\theta;\Data/2,\tfrac12)+\tfrac12\mathcal{N}(\theta;-\Data/2,\tfrac12)$, the two components having equal weight because both values of $\phi$ have the same evidence.
These differ whenever $\Data\neq0$, so $\Cpl>0$ almost surely and, by \Cref{lem:decomposition}, $\Lc=\Cpl>0$: the data teach that the sign of $\theta$ is the sign of $\phi$.
A decision maker who knows $\phi$ acts differently towards the two groups: under squared error, for instance, the $\phi$-aware estimate of $\theta$ is $\phi\Data/2$ while the $\phi$-blind estimate is $0$.
Conversely, a likelihood that factorises as $p(\Data_\theta\mid\theta)\,p(\Data_\phi\mid\phi)$ has $\Lu>0$ whenever $\Data_\phi$ is informative, yet the posterior factorises for every dataset, so $\Cpl=0$: the data reveal $\phi$ without coupling it to $\theta$.
Penalising only the marginal leakage would therefore accept data of the first kind, and penalising only the coupling data of the second kind; the conditional leakage rejects both.

\subsection{Where Shared Variables Go}
\label{app:shared_variables}

The conditional leakage conditions on both $\theta$ and $\omega$, so it cannot depend on how a variable is divided between them; the two components it decomposes into can.

\begin{lemma}[Relocating a variable]
\label{lem:relocation}
For any variable $v$,
\begin{align*}
I(\phi;\Data\mid(\theta,v),\omega)&=I(\phi;\Data\mid\theta,(v,\omega)),\\
I(\phi;\Data\mid v,\omega)&=I(\phi;\Data\mid\omega)+I(\phi;v\mid\Data,\omega)-I(\phi;v\mid\omega),\\
I((\theta,v);\phi\mid\Data,\omega)&=I(\theta;\phi\mid\Data,v,\omega)+I(v;\phi\mid\Data,\omega),\\
I((\theta,v);\phi\mid\omega)&=I(\theta;\phi\mid v,\omega)+I(v;\phi\mid\omega).
\end{align*}
Hence moving $v$ from the context into the target leaves the conditional leakage unchanged and transfers $I(\phi;v\mid\Data,\omega)-I(\phi;v\mid\omega)$ from the marginal leakage to the increase in coupling: the coupling gains $I(\phi;v\mid\Data,\omega)$ and the prior coupling gains $I(\phi;v\mid\omega)$.
\end{lemma}

\begin{proof}
The first line writes one conditional mutual information in two ways, the third and fourth are the chain rule, and the second follows from expanding $I(\phi;\Data,v\mid\omega)$ in both orders.
\end{proof}

\paragraph{Placements in active learning.}
\Cref{tab:placements} applies \Cref{lem:relocation} to the test input $x_*$, which both labels depend on.
Placing it in the target moves the predictability of an individual's group into the coupling and breaks the fair-world prior whenever inputs predict group; giving each label an independent copy of the input removes the coupling altogether; marginalising the inputs yields a criterion whose coupling is the dependence relevant to demographic parity.
Only the context placement of \Cref{sec:active_learning} gives both components their intended meaning.

\begin{table}[h]
\centering\small
\caption{\textbf{Where the test input goes} in active learning, with $\theta\ni y_{*\theta}$ and $\phi\ni y_{*\phi}$. The conditional leakage is the same in the first two rows (\Cref{lem:relocation}).}
\label{tab:placements}
\setlength{\tabcolsep}{4pt}
\resizebox{\textwidth}{!}{%
\begin{tabular}{llll}
\toprule
Placement & Marginal leakage $\Fphi$ & Coupling $\Fc$ & Fair-world prior \\
\midrule
$\omega=x_*$ (ours) & $\mathbb{E}\,I(y_{*\phi};\Data\mid x_*)$: individual's group & $\mathbb{E}\,I(y_{*\theta};y_{*\phi}\mid x_*,\Data)$: direct effect & $y_{*\theta}\perp y_{*\phi}\mid x_*$ \\
$x_*\in\theta$ & $I(y_{*\phi};\Data)$: group proportion only & $I(x_*;y_{*\phi}\mid\Data)+\mathbb{E}\,I(y_{*\theta};y_{*\phi}\mid x_*,\Data)$ & fails if $x_*\not\perp y_{*\phi}$ \\
independent copies $x_*'\in\phi$ & $\mathbb{E}\,I(y_{*\phi};\Data\mid x'_*)$: individual's group & $0$ for every dataset & holds trivially \\
$x_*$ marginalised & $I(y_{*\phi};\Data)$: group proportion only & $I(y_{*\theta};y_{*\phi}\mid\Data)$: demographic-parity coupling & $y_{*\theta}\perp y_{*\phi}$ \\
\bottomrule
\end{tabular}}
\end{table}

\subsection{Decision-Theoretic Guarantees}
\label{app:proof_guarantee}

Consider a decision maker who must choose an action $a\in\mathcal{A}$ with loss $\ell(a,\theta,\omega)$ depending on the unknown target, with posterior expected losses $\rho(a)=\mathbb{E}_{p(\theta\mid\Data,\omega)}[\ell(a,\theta,\omega)]$ and $\rho_\phi(a)=\mathbb{E}_{p(\theta\mid\Data,\phi,\omega)}[\ell(a,\theta,\omega)]$ without and with knowledge of $\phi$.
The \emph{value of information} of $\phi$ to them is
\begin{equation}
    \VOI_\phi(\Data\mid\omega):=\mathbb{E}_{p(\phi\mid\Data,\omega)}\big[\min\nolimits_a\rho(a)-\min\nolimits_a\rho_\phi(a)\big]\ge0,
    \label{eq:voi}
\end{equation}
where non-negativity follows from $\mathbb{E}_{\phi}[\min_a\rho_\phi(a)]\le\min_a\mathbb{E}_\phi[\rho_\phi(a)]=\min_a\rho(a)$.

\begin{proof}[Proof of \Cref{prop:guarantee}]
(i)~$\Lupi=\mathbb{E}_{p(\omega)}[I(\phi;\Data\mid\omega)]=0$ with each term non-negative gives $I(\phi;\Data\mid\omega)=0$ for almost every $\omega$, which is equivalent to $\phi\perp\Data\mid\omega$, i.e.\ $p(\phi\mid\Data,\omega)=p(\phi\mid\omega)$ for almost every $(\Data,\omega)$.
(ii)~$\Cplpi=\mathbb{E}[\Cpldw]=0$ with $\Cpldw\ge0$ gives $\Cpldw=0$ for almost every $(\Data,\omega)$, so $\mathrm{KL}(p(\phi\mid\theta,\Data,\omega)\,\|\,p(\phi\mid\Data,\omega))=0$ for $p(\theta\mid\Data,\omega)$-almost every $\theta$; that is, $\theta\perp\phi\mid\Data,\omega$, or
\begin{equation}
p(\theta\mid\Data,\omega,\phi)=p(\theta\mid\Data,\omega)\qquad\text{for $p(\phi\mid\Data,\omega)$-almost every $\phi$}.
\label{eq:posterior_eq}
\end{equation}
For any loss $\ell(a,\theta,\omega)$, $\rho_\phi$ and $\rho$ are therefore the same function of $a$, so their minimisers coincide and every term in~\eqref{eq:voi} vanishes.
(iii)~For $p(z\mid\theta,\omega)$ not depending on $\phi$, \eqref{eq:posterior_eq} gives $p(z\mid\Data,\omega,\phi)=\int p(z\mid\theta,\omega)\,p(\theta\mid\Data,\omega,\phi)\,d\theta=\int p(z\mid\theta,\omega)\,p(\theta\mid\Data,\omega)\,d\theta=p(z\mid\Data,\omega)$; with an empty context this is demographic parity of the predictive, $p(z\mid\Data,\phi)=p(z\mid\Data)$.
\end{proof}

\paragraph{The coupling as a value of information.}
If the action is a distribution $q$ over $\theta$ scored by $-\log q(\theta)$, then $\min_q\mathbb{E}_{p}[-\log q(\theta)]=\mathcal{H}(p)$ by Gibbs' inequality, so~\eqref{eq:voi} becomes $\mathcal{H}(p(\theta\mid\Data,\omega))-\mathbb{E}_{p(\phi\mid\Data,\omega)}\mathcal{H}(p(\theta\mid\Data,\omega,\phi))$, the mutual information between $\theta$ and $\phi$ under the posterior, which is $\Cpldw$.
The coupling is thus the value of $\phi$ to a log-score decision maker acting on $\theta$, while the marginal leakage $\Luw$ is the value of the data to an adversary whose aim is to infer $\phi$~\citep{Rossa2026ActionBEDTB}.

\begin{proposition}[Approximate guarantees]
\label{prop:approx}
For every dataset $\Data$ and context $\omega$, and every predictive model $p(z\mid\theta,\omega)$ that does not take $\phi$ as input,
\[
\mathbb{E}_{p(\phi\mid\Data,\omega)}\mathrm{TV}\big(p(z\mid\Data,\omega,\phi),\,p(z\mid\Data,\omega)\big)\le\sqrt{\Cpldw/2},
\]
and for every loss with values in $[0,L]$, $\VOI_\phi(\Data\mid\omega)\le L\sqrt{\Cpldw/2}$.
Moreover, the adversary's posterior satisfies $\mathbb{E}_{p(\Data\mid\omega)}\mathrm{TV}\big(p(\phi\mid\Data,\omega),\,p(\phi\mid\omega)\big)\le\sqrt{I(\phi;\Data\mid\omega)/2}$ at each context.
\end{proposition}
\begin{proof}
Since $\Cpldw=\mathbb{E}_{p(\phi\mid\Data,\omega)}\mathrm{KL}\big(p(\theta\mid\Data,\omega,\phi)\,\|\,p(\theta\mid\Data,\omega)\big)$, Pinsker's inequality $\mathrm{TV}(P,Q)\le\sqrt{\mathrm{KL}(P\|Q)/2}$~\citep{cover1999elements} and Jensen's inequality give $\mathbb{E}_{p(\phi\mid\Data,\omega)}\mathrm{TV}\big(p(\theta\mid\Data,\omega,\phi),p(\theta\mid\Data,\omega)\big)\le\sqrt{\Cpldw/2}$.
The predictives are obtained from these posteriors through the same kernel $p(z\mid\theta,\omega)$, so their total-variation distance is no larger (data processing).
For a loss with values in $[0,L]$, $|\rho(a)-\rho_\phi(a)|\le L\,\mathrm{TV}\big(p(\theta\mid\Data,\omega,\phi),p(\theta\mid\Data,\omega)\big)$ for every $a$; with $a^*_\phi$ a minimiser of $\rho_\phi$, $\min_a\rho(a)\le\rho(a^*_\phi)\le\min_a\rho_\phi(a)+L\,\mathrm{TV}$, and averaging over $\phi$ gives the bound on $\VOI_\phi$.
The last statement is the same argument applied to $p(\phi\mid\Data,\omega)$ and $p(\phi\mid\omega)$, whose expected KL divergence is $I(\phi;\Data\mid\omega)$.
\end{proof}

\subsection{Richer Contexts and Deployment}
\label{app:pooling}

Adding a variable $v$ to the context has two effects that pull in opposite directions, formalised by the following two propositions, which underpin the discussion in \Cref{subsec:context,subsec:guarantees}.
Below, $\Lc(\pi\mid\omega):=I(\phi;\Data\mid\theta,\omega)$ denotes the conditional leakage at a fixed context, so that $\Lcpi=\mathbb{E}_{p(\omega)}[\Lc(\pi\mid\omega)]$.

\begin{proposition}[Richer contexts]
\label{prop:context}
For any variable $v$, $\Lc(\pi\mid\omega,v)=\Lc(\pi\mid\omega)+\Delta_v$, where $\Delta_v:=I(\phi;v\mid\Data,\theta,\omega)-I(\phi;v\mid\theta,\omega)$ is the increase, due to the data, in what $v$ reveals about $\phi$.
If $v$ carries no information about $\phi$ given $(\theta,\omega)$ a priori, then $\Delta_v\ge0$, so $\Lc(\pi\mid\omega,v)=0$ implies $\Lc(\pi\mid\omega)=0$.
In general, $\Lc(\pi\mid\omega,v)=0$ implies $\Lc(\pi\mid\omega)\le I(\phi;v\mid\theta,\omega)$.
For the assessment-model index $m$, $\Lc(\pi\mid\omega,m)=0$ implies that the conditional leakage vanishes under every model in $\mathcal{M}$ and under any mixture of them, including the Bayesian model average.
\end{proposition}

\begin{proof}
Expanding $I(\phi;\Data,v\mid\theta,\omega)$ by the chain rule in both orders gives $I(\phi;v\mid\theta,\omega)+I(\phi;\Data\mid\theta,\omega,v)=I(\phi;\Data\mid\theta,\omega)+I(\phi;v\mid\Data,\theta,\omega)$, which rearranges to the identity.
If $I(\phi;v\mid\theta,\omega)=0$ then $\Delta_v=I(\phi;v\mid\Data,\theta,\omega)\ge0$.
If $\Lc(\pi\mid\omega,v)=0$ then $\Lc(\pi\mid\omega)=-\Delta_v\le I(\phi;v\mid\theta,\omega)$.
For the model index, $p(\phi\mid\theta,\omega,m)=p(\phi\mid\omega)$ because all models share the fair-world prior, so $I(\phi;m\mid\theta,\omega)=0$; moreover $\Lc(\pi\mid\omega,m)=\mathbb{E}_{p(m)}[\Lc^{(m)}]$ with $p(m)>0$, so its vanishing forces the leakage $\Lc^{(m)}$ under every model to vanish, and the leakage under any mixture of the models, bounded by the corresponding average of the $\Lc^{(m)}$, vanishes too.
\end{proof}

A richer context thus makes the requirement stricter, since the fair world must hold in finer cells and against users who hold more information, but also declares admissible the dependence between $\theta$ and $\phi$ that flows through the added variable.
If that variable differs substantially between groups, data that are fair given the richer context can still support group disparities, which the next result bounds.

\begin{proposition}[Deployment]
\label{prop:deployment}
Let predictions be made at contexts $v$ drawn from a deployment population $q(v,\phi)$ in which contexts may depend on $\phi$, with $q(v)$ absolutely continuous with respect to the context distribution of the assessment, and suppose $\Fc=0$ with context $v$.
Then \Cref{prop:guarantee}(ii)--(iii) hold for $q$-almost every $v$, and for a binary prediction $f_\Data(v)=p(z{=}1\mid\Data,v)$, which is also that of any user who knows $\phi$, the demographic-parity gap satisfies
\begin{align}
\Delta_q&:=\mathbb{E}_q[f_\Data\mid\phi{=}1]-\mathbb{E}_q[f_\Data\mid\phi{=}0]=\Cov_q(f_\Data,g_q),\nonumber\\
|\Delta_q|&\le\mathrm{TV}\big(q(v\mid\phi{=}1),\,q(v\mid\phi{=}0)\big),
\label{eq:dp_gap_cov}
\end{align}
where $g_q(v)=\big(q(v\mid\phi{=}1)-q(v\mid\phi{=}0)\big)/q(v)$.
\end{proposition}

\begin{proof}
The conclusions of \Cref{prop:guarantee} hold for almost every $v$ under the context distribution of the assessment, hence $q$-almost everywhere by absolute continuity; in particular $p(z\mid\Data,v,\phi)=p(z\mid\Data,v)$, so $f_\Data$ is also the prediction of a user who knows $\phi$.
Writing $q_a(v)=q(v\mid\phi{=}a)$, $\Delta_q=\int f_\Data(v)\big(q_1(v)-q_0(v)\big)\,dv=\mathbb{E}_{q(v)}[f_\Data\,g_q]$, and $\mathbb{E}_{q(v)}[g_q]=\int(q_1-q_0)\,dv=0$, so $\Delta_q=\Cov_q(f_\Data,g_q)$.
Since $0\le f_\Data\le1$, $\int f_\Data\,(q_1-q_0)\,dv\le\int(q_1-q_0)_+\,dv=\mathrm{TV}(q_1,q_0)$, and symmetrically from below.
\end{proof}

\paragraph{Remarks on deployment.}
\Cref{prop:deployment} explains why coupling-free data can be expected to remain fair when a model is deployed in the real world, where covariates typically do depend on the group.
With inputs in the context, the fair world allows inputs to differ across groups and only forbids outcomes given inputs to do so, so the guarantee does not presuppose covariate independence: it holds at every context covered by the assessment, however contexts are distributed at deployment.
What covariate dependence does affect is the group-level aggregate: the gap is the covariance between the learned function and a fixed group discriminant of the deployment population, and with coupling-free data the learned function is a belief about a group-independent outcome function rather than a proxy for group membership.
In terms of the group probability $w(v) = q(\phi = 1 \mid v)$ and
$\tau_a = q(\phi = a)$, $g_q = w/\tau_1 - (1 - w)/\tau_0$.
The Cauchy--Schwarz inequality gives the complementary bound $|\Delta_q|\le\mathrm{sd}_q(f_\Data)\,\mathrm{sd}_q(g_q)$, where $\mathrm{sd}_q(g_q)^2=\int(q_1-q_0)^2/q\,dv$ measures how strongly the deployment contexts discriminate between groups: the gap is small when the learned function varies little along the directions in which the groups differ, which is effect~(c) of \Cref{app:effects}.
The TV bound holds for any group-blind predictor; what coupling-freeness adds is that it also covers users who know $\phi$, and that the dependence of $f_\Data$ on the context reflects a group-independent outcome function rather than an inferred group membership.
For a user who knows $\phi$ when the coupling is small but non-zero, writing $f_a(v)=p(z{=}1\mid\Data,v,\phi{=}a)$, the gap is $\int f_1q_1-\int f_0q_0=\Delta_q+\int(f_1-f_\Data)q_1-\int(f_0-f_\Data)q_0$, so it exceeds $|\Delta_q|$ by at most $\sum_a\mathbb{E}_{q_a}|f_a-f_\Data|$, a term that vanishes with the coupling (\Cref{prop:approx}).
The guarantee does not extend to deployment contexts outside the support of the assessment's context distribution, where the data say nothing and fairness, like accuracy, is an extrapolation; the context distribution used during acquisition, $p_*$ in active learning, should therefore cover the deployment population.
An empty context guarantees demographic parity directly, but then also penalises learning legitimate covariate effects whenever covariates differ between groups, which is the familiar tension between demographic parity and accuracy:

\begin{remark}[Demographic parity and accuracy]
If the inputs are independent of the group given the outcome, $v\perp\phi\mid y$, as when group differences in inputs are fully explained by differences in outcomes, then $p(v\mid\phi{=}a)=\sum_y p(v\mid y)\,p(y\mid\phi{=}a)$ and the gap~\eqref{eq:dp_gap_cov} equals $\big(p(y{=}1\mid\phi{=}1)-p(y{=}1\mid\phi{=}0)\big)\big(\mathbb{E}[f_\Data\mid y{=}1]-\mathbb{E}[f_\Data\mid y{=}0]\big)$: the difference in base rates times the separation of the predictor. Demographic parity then fails for any predictor that separates the classes unless base rates are equal.
\end{remark}

\subsection{Effects of the Penalties on the Acquired Data}
\label{app:effects}

Because FairBED changes which observations are made, it can also change the data in ways it does not directly target.
(a)~Designs concentrate where the outcome mechanism is least sensitive to $\phi$ given $\theta$, or where $\phi$ is already determined, reducing design diversity (\Cref{fig:loc_find_designs_static}).
(b)~Nothing in the penalties asks for balance across groups, so the composition of the data can shift in either direction, up to acquiring from a single group; composition and per-group performance should therefore be reported and may need to be constrained.
(c)~Data that do not reveal $\phi$ tend to vary little along the directions of input space that discriminate between groups, so downstream models rely less on proxies of $\phi$, which can improve demographic parity beyond what \Cref{prop:deployment} guarantees but also removes legitimate signal (\Cref{app:contrast}).
(d)~Coupling-free data withhold evidence about the relationship between $\theta$ and $\phi$, so they cannot be used downstream to estimate or correct a direct effect of $\phi$ (\Cref{sec:rel_works}).
(e)~As with any adaptive design, the data are not a random sample; this does not bias Bayesian inference, since the design is a function of past data and drops out of the likelihood, but naive frequentist estimates of population quantities from the acquired data are biased.

\subsection{Active Learning}
\label{app:al_theory}

This section expands on \Cref{sec:active_learning}: it details the instantiation of FairBED in active learning (\Cref{fig:al_graphical_model}), shows why the assessment model must be a joint model of the label pair, defines the per-step EPIG-based objective~\eqref{eq:fairbed_AL} and its decomposition, and discusses the guarantees it inherits and what the penalties do to the acquired data.

\begin{figure}[t]
\centering
\begin{tikzpicture}[x=1cm,y=1cm]
  \node[des] (pi)  at (-1.2,-1)  {$\pi$};
  \node[obs] (x)   at (0,-1)     {$x^i$};
  \node[obs] (yt)  at (1.25,0)   {$y_\theta^i$};
  \node[obs] (yp)  at (1.25,-2)  {$y_\phi^i$};
  \node[rv]  (psi) at (2.65,0)   {$\psi$};
  \node[rv]  (kap) at (2.65,-2)  {$\kappa$};
  \node[rv]  (yst) at (4.05,0)   {$y_{*\theta}$};
  \node[rv]  (ysp) at (4.05,-2)  {$y_{*\phi}$};
  \node[ctx] (xs)  at (5.3,-1)   {$x_*$};
  \node[draw, thick, rounded corners=3pt, inner sep=5pt, fit=(x)(yt)(yp)] (plate) {};
  \node[lab, anchor=south west] at (plate.south west) {$i\le t$};
  \draw[edge] (pi) -- (x);
  \draw[edge] (x) -- (yt);
  \draw[edge] (x) -- (yp);
  \draw[cut]  (yp) -- node[lab, right] {$\delta$} (yt);
  \draw[edge] (psi) -- (yt);
  \draw[cut]  (kap) -- (yp);
  \draw[edge] (psi) -- (yst);
  \draw[edge] (kap) -- (ysp);
  \draw[edge] (xs) -- (yst);
  \draw[edge] (xs) -- (ysp);
  \draw[weak] (ysp) -- node[lab, left] {$\delta$} (yst);
\end{tikzpicture}
\vspace{-2pt}
\caption{\textbf{FairBED for active learning.} Acquired units $i\le t$ (shaded) and a future individual with admissible inputs $x_*$ (context, double circle) and labels $y_{*\theta}=\theta$ and $y_{*\phi}=\phi$; $\psi$ and $\kappa$ govern the outcome and group mechanisms. A joint assessment model allows a direct effect $\delta$ of the group on the outcome at fixed inputs, that is, a dependence of $p(y_\theta\mid x,y_\phi)$ on $y_\phi$ such as a group coefficient in a logistic model; separate models delete it by construction (\Cref{prop:al_no_coupling}). Red dashed: the channels through which acquired pairs carry what FairBED penalises, information about the group mechanism (marginal leakage) and about $\delta$ (coupling).}
\label{fig:al_graphical_model}
\end{figure}

\paragraph{Setup.}
Acquiring at $x$ returns the label pair $y=(y_\theta,y_\phi)$, so $\Data=\{(x^i,y_\theta^i,y_\phi^i)\}_{i=1}^t$, and our aim is to carefully choose the $x$ to learn a predictive model $p(y\mid x)=\mathbb{E}_{p(\upsilon)}[p(y\mid x,\upsilon)]$, where $\upsilon$ are latent model variables.
The sensitive labels of the sampled units are observed and are not what fairness is about; following \Cref{subsec:context}, the quantities of \Cref{subsec:formalising_data_fairness} concern a future individual with admissible inputs $x_*\sim p_*$ drawn from an assumed test-time distribution, and we set $\theta=y_{*\theta}$ and $\phi=y_{*\phi}$, the target and sensitive labels of this individual, with context $\omega=(x_*,m)$, their inputs and the assessment model.
The fair-world prior is that, a priori, the two labels of an individual are conditionally independent given their inputs, which may themselves differ across groups.
The marginal leakage $\mathbb{E}_{x_*}[I(y_{*\phi};\Data\mid x_*)]$ is then how much the data improve the prediction of a new individual's group from their inputs, and the coupling $\mathbb{E}_{x_*}[I(y_{*\theta};y_{*\phi}\mid x_*,\Data)]$ is the value of knowing an individual's group when predicting their outcome from their inputs: a \emph{direct effect} of the sensitive attribute that the data have taught.
Concretely, the direct effect is the dependence of $p(y_\theta\mid x,y_\phi)$ on $y_\phi$ at fixed $x$, such as the group coefficient $\delta$ in $\operatorname{logit}p(y_\theta{=}1\mid x,y_\phi)=f(x)+\delta\,(2y_\phi-1)$; coupled data are data from which a learner would come to believe that $\delta\neq0$.
Other placements of the inputs would measure different things (\Cref{app:shared_variables}).
Under \emph{separate} models the labels are governed by latents $\psi$ and $\kappa$ with $p(\psi,\kappa)=p(\psi)p(\kappa)$, and the data have likelihood
\begin{equation}
p(\Data\mid\psi,\kappa,\pi)=\prod_{i=1}^t p(x^i\mid\Data_{i-1},\pi)\,p(y_\theta^i\mid x^i,\psi)\,p(y_\phi^i\mid x^i,\kappa),
\label{eq:al_likelihood_factorised}
\end{equation}
where the policy terms do not depend on $\psi$ or $\kappa$ and may be deterministic.
A \emph{joint} model replaces the two label factors by $p(y_\theta^i,y_\phi^i\mid x^i,\upsilon)$, which may carry dependence between the labels given the inputs.

\paragraph{Admissible inputs and proxies.}
Placing inputs in the context declares the dependence they carry admissible (\Cref{subsec:context}), so it matters which inputs are placed there.
We therefore split the inputs into \emph{admissible} ones, whose influence on decisions we accept (e.g.\ qualifications), and \emph{proxies}, which are informative about the group but should not influence decisions (e.g.\ a postcode standing in for ethnicity).
Only admissible inputs belong in the context and in the working model that values the data; proxies are then marginalised in the assessment, so that data teaching an outcome dependence that flows through them are penalised rather than certified.
Omitting proxies from the downstream model alone does not achieve this: an acquisition objective that conditions on them gathers data whose value rests on them, and declares the dependence they carry admissible.
Unless stated otherwise, our experiments treat all inputs as admissible.

\paragraph{The assessment model must be joint.}
A natural implementation models the two labels separately, as in~\eqref{eq:al_likelihood_factorised}, but this cannot represent what we wish to measure.

\begin{lemma}[Posterior factorisation under separate models]
\label{lem:al_factorisation}
Under~\eqref{eq:al_likelihood_factorised}, $p(\psi,\kappa\mid\Data)=p(\psi\mid\Data)\,p(\kappa\mid\Data)$ for every dataset and every policy.
\end{lemma}

\begin{proof}
$p(\psi,\kappa\mid\Data)\propto p(\psi)p(\kappa)\,p(\Data\mid\psi,\kappa,\pi)$, and the policy terms are constant in $(\psi,\kappa)$, leaving a product of a function of $\psi$ and a function of $\kappa$; adaptivity is immaterial because the chosen inputs enter only as conditioning variables.
\end{proof}

\begin{proposition}[Separate models cannot represent coupling]
\label{prop:al_no_coupling}
Under separate models~\eqref{eq:al_likelihood_factorised}, $\Cpldw=0$ for every dataset $\Data$ and every, possibly adaptive, policy, so the conditional leakage reduces to the marginal leakage.
\end{proposition}

\begin{proof}
By \Cref{lem:al_factorisation} the posterior over the latents factorises for every $\Data$.
As $y_{*\theta}$ depends only on $(x_*,\psi)$ and $y_{*\phi}$ only on $(x_*,\kappa)$,
$p(y_{*\theta},y_{*\phi}\mid x_*,\Data)=\iint p(y_{*\theta}\mid x_*,\psi)\,p(y_{*\phi}\mid x_*,\kappa)\,p(\psi\mid\Data)\,p(\kappa\mid\Data)\,d\psi\,d\kappa=p(y_{*\theta}\mid x_*,\Data)\,p(y_{*\phi}\mid x_*,\Data)$
for every $x_*$ and $\Data$, so $\Cpldw=0$ and $\Fcond=\Fphi$ by \Cref{lem:decomposition}.
For a single acquisition, the same argument shows that $(y_\theta,y_{*\theta})$ and $(y_\phi,y_{*\phi})$ are independent given $(x,x_*,\Data_{t-1})$, so that, with this conditioning suppressed, $I(y;y_{*\phi}\mid y_{*\theta})=I(y_\theta;y_{*\phi}\mid y_{*\theta})+I(y_\phi;y_{*\phi}\mid y_{*\theta},y_\theta)=0+I(y_\phi;y_{*\phi})$: the conditional-leakage penalty in~\eqref{eq:fairbed_AL} below reduces to the marginal leakage of the new sensitive label alone.
\end{proof}

The conclusion depends only on the model class: under separate models even random acquisition is certified coupling-free.
Following \Cref{subsec:context}, the assessment models must therefore include \emph{joint} models of the label pair, under which the coupling is a property of the data, such as a single four-class model over $(y_\theta,y_\phi)$ or an autoregressive model $p(y_\theta\mid x,\psi)\,p(y_\phi\mid x,y_\theta,\kappa)$ whose second factor may depend on the outcome; the former is used in \Cref{subsec:al_experiments}.
The conditional dependence between the labels given $x$ is identifiable from a single pair per input, just as $p(y\mid x)$ is identifiable without repeated inputs: it is a property of the four cell probabilities $p(y_\theta,y_\phi\mid x)$, which are estimable from independent pairs $(x^i,y_\theta^i,y_\phi^i)$ by any consistent multi-class method, with smoothness in $x$ substituting for replication.
Multi-output decision trees do not estimate it, because each leaf stores the class distribution of each output separately and so assumes independence within the leaf, whereas a four-class forest over the pair, or an autoregressive pair of forests whose second forest takes the outcome as an input, represents it.
A parametric check that is sharper at small sample sizes is to fit a logistic regression of $y_\theta$ on $(x,y_\phi)$ to $\Data$ and read off the standard error of the $y_\phi$ coefficient.

\paragraph{The objective and its decomposition.}
The expected predictive information gain~\citep{pmlr-v206-bickfordsmith23a} scores a candidate input $x$ by the expected information its label carries about the labels of future test inputs, $\mathrm{EPIG}(x)=\mathbb{E}_{p_*(x_*)}[I(y;y_*\mid x,x_*)]$, which targets predictions at test inputs rather than the latent variables $\upsilon$ themselves~\citep{houlsby2011bayesian}.
With the instantiation above, the FairBED objective~\eqref{eq:fairbed_family} for the query made after acquiring $\Data_{t-1}$ is
\begin{equation}
    R_{\mathcal{I}}(x)=\underbrace{\mathbb{E}_{x_*}\big[I(y_\theta;y_{*\theta}\mid x,x_*,\Data_{t-1})\big]}_{\mathrm{EPIG}_\theta(x)}-\beta\,\mathcal{I}(x),
    \label{eq:fairbed_AL}
\end{equation}
where $\mathrm{EPIG}_\theta(x)$ is valued under the working target model, which uses only admissible inputs and not the group, and the penalty $\mathcal{I}(x)\in\{\Fc(x),\Fphi(x),\Fcond(x)\}$ is the per-step counterpart of the corresponding penalty of \Cref{sub:fairbed-objective}: the expected increase in coupling, marginal leakage, or conditional leakage of the single acquisition $y$ under the assessment models, each conditioned on $\Data_{t-1}$.
We refer to the resulting per-step objectives as $\Rc$, $\Rphi$ and $\Rcond$ as in \Cref{sub:fairbed-objective}.
With $x$, $m$ and $\Data_{t-1}$ suppressed, $\Fcond(x)=\mathbb{E}_{p(\omega)}[I(y;y_{*\phi}\mid y_{*\theta},x_*)]$ and $\Fphi(x)=\mathbb{E}_{p(\omega)}[I(y;y_{*\phi}\mid x_*)]$.
The former decomposes as in \Cref{lem:decomposition} with one difference: the beliefs before the acquisition are the posterior given $\Data_{t-1}$, which need not be a fair world, since earlier data may already have induced some coupling.
The general form~\eqref{eq:general_decomposition} therefore applies, with this posterior in place of the prior: the penalty is the marginal leakage of the new pair plus the \emph{increase} in coupling it causes,
\begin{align}
    I(y;y_{*\phi}\mid y_{*\theta},x_*)=I(y;y_{*\phi}\mid x_*)+\big[I(y_{*\theta};y_{*\phi}\mid y,x_*)-I(y_{*\theta};y_{*\phi}\mid x_*)\big],
    \label{eq:al_penalty_decomposition}
\end{align}
where the last term is the coupling already present, which is zero only before any data are acquired.
The coupling penalty $\Fc(x)$ is the bracketed increase, computed as the difference $\Fcond(x)-\Fphi(x)$.
Summed over a sequential policy the increases telescope, so the per-step penalties add up exactly to the corresponding penalties of \Cref{sub:fairbed-objective} for the whole dataset, and greedy acquisition pursues the criterion of \Cref{sec:method} one step at a time:

\begin{lemma}[Sequential penalties]
\label{lem:al_chain}
For a policy acquiring $(x_t,y_t)$ at steps $t=1,\dots,T$, where $x_t$ depends only on $\Data_{t-1}$ and on randomness independent of $(\theta,\phi,\omega)$,
\[
I(\phi;\Data_T\mid\theta,\omega)=\sum_{t=1}^T I(\phi;y_t\mid\theta,\omega,\Data_{t-1},x_t),
\qquad
I(\phi;\Data_T\mid\omega)=\sum_{t=1}^T I(\phi;y_t\mid\omega,\Data_{t-1},x_t),
\]
and each term of the first sum equals $I(y_t;\phi\mid\omega,\Data_{t-1},x_t)+I(\theta;\phi\mid\omega,\Data_t)-I(\theta;\phi\mid\omega,\Data_{t-1})$, the marginal leakage of the step plus the increase in coupling it induces, without any assumption on the prior.
In particular the per-step increases in coupling telescope to $I(\theta;\phi\mid\omega,\Data_T)-I(\theta;\phi\mid\omega)$, which is the coupling of the whole dataset under the fair-world prior, so the per-step penalties of each objective in~\eqref{eq:fairbed_AL} sum to the corresponding penalty of \Cref{sub:fairbed-objective} for the whole dataset.
\end{lemma}

\begin{proof}
By the chain rule $I(\phi;\Data_T\mid\theta,\omega)=\sum_t I(\phi;(x_t,y_t)\mid\theta,\omega,\Data_{t-1})$, and $I(\phi;x_t\mid\theta,\omega,\Data_{t-1})=0$ because $x_t$ is a function of $\Data_{t-1}$ and independent randomness; the second identity is the same argument without $\theta$.
The decomposition of each term is the chain-rule identity~\eqref{eq:general_decomposition} applied conditionally on $(\omega,\Data_{t-1},x_t)$, using that $x_t$ carries no information about $(\theta,\phi)$ given $\Data_{t-1}$.
\end{proof}

\paragraph{Guarantees in active learning.}
Applying \Cref{prop:guarantee} with $\omega=(x_*,m)$: if $\Fcond=0$, then for every $m$ with $p(m)>0$ and $p_*$-almost every $x_*$, (i)~$p(y_{*\phi}\mid x_*,\Data,m)=p(y_{*\phi}\mid x_*,m)$; (ii)~the Bayes actions of users with and without access to $y_{*\phi}$ coincide for every loss; and (iii)~$p(y_{*\theta}\mid x_*,y_{*\phi},\Data,m)=p(y_{*\theta}\mid x_*,\Data,m)$.
By (iii), all the group-specific predictives in $p(y_{*\theta}\mid x_*,\Data,m)=\sum_a p(y_{*\phi}{=}a\mid x_*,\Data,m)\,p(y_{*\theta}\mid x_*,y_{*\phi}{=}a,\Data,m)$ are equal, so the group-blind predictive does not weight group-specific outcome functions by the inferred group of the individual: inputs that predict group membership enter it only through their direct value for the outcome.
At deployment, \Cref{prop:deployment} applies with $v=x_*$: demographic parity holds when $p_*(x_*\mid y_{*\phi})=p_*(x_*)$, and otherwise the gap is bounded by the total-variation distance between the groups' input distributions.

\paragraph{What the penalties do to the data.}
In terms of which inputs are queried, we want inputs from which predictions can be interpolated and extrapolated to future individuals, while avoiding the regions of input space whose labels would reveal the unfairness of the world.
A pair reveals group membership where the assessment model is uncertain about how the group depends on the inputs, and teaches a direct effect where, in addition, both labels are uncertain at that input, so that the pair acts as a comparison of outcomes across groups at fixed inputs.
The two come apart where the probability of each group at $x$ is already known but intermediate: a new pair then reveals nothing further about group membership yet still informs the direct effect, and such pairs are exactly what a separate-model assessment would wrongly score as fair.
Target-informative pairs that are coupling-free are instead found where the group is well predicted by the inputs or the outcome is nearly determined, which is how the composition effects of \Cref{app:effects} arise.
Coupling-free data are of little use for estimating a direct effect, which is intended but marks a genuine difference from methods that correct disparities by using the sensitive attribute explicitly (\Cref{sec:rel_works}).

\subsection{Group Contrast in the Acquired Inputs}
\label{app:contrast}
This section explains effect~(c) of \Cref{app:effects}.
Write $w(x)=p(\phi{=}1\mid x)$ for the probability of group membership at input $x$ in the population, which varies along the directions of input space that discriminate between groups.
A pair acquired at $x$ reveals group membership insofar as it changes the assessment model's beliefs about $w$, and the inputs at which those beliefs change least are those where $w$ is already well determined, typically far from the group boundary.
A penalty on the marginal leakage therefore steers acquisition away from the regions where $w$ is uncertain and changes fastest and, if composition is not controlled, towards the part of input space occupied by one group; either way the acquired inputs tend to vary less along the directions in which $w$ changes.

\paragraph{Consequences for downstream models.}
A model trained on such data has little opportunity to learn how the outcome varies along those directions, so its learned function $f_\Data$ tends to vary less along them too.
The demographic-parity gap at deployment, $\Delta_q=\Cov_q(f_\Data,g_q)$ (\Cref{prop:deployment}), depends only on the component of $f_\Data$ that is correlated with the group discriminant $g_q$: writing $f_\Data=c\,g_q+r$ with $r$ uncorrelated with $g_q$, $\Delta_q=c\,\mathrm{Var}_q(g_q)$.
Reduced variation of the learned function along the group discriminant therefore reduces the gap.
This is a reduction of input-mediated dependence, which \Cref{prop:deployment} declares admissible and does not require: it is a side effect of acquisition rather than part of the guarantee.
It also has a cost, since directions that discriminate between groups often carry legitimate information about the outcome, as proxies do, and a learned function that underuses them is less accurate; this is one source of the accuracy cost of FairBED.
Effect~(c) is therefore one of two mechanisms behind the demographic-parity gains in \Cref{sec:experiments}, the other being that the learned function does not absorb the direct effect (the coupling channel); the spread of $w$ over the acquired inputs distinguishes them.

\section{Estimating and Optimising the FairBED Objectives}
\label{app:app_fairbed_bed_detail}

Each FairBED objective~\eqref{eq:fairbed_family} is a difference of expected information gains, so it can be estimated and optimised with standard BED machinery.
Writing all quantities under the model~\eqref{eq:model}, with the context $\omega$ suppressed (every term below is conditioned on $\omega$ and averaged over $p(\omega)$), and using the fair-world prior $p(\theta,\phi)=p(\theta)p(\phi)$, the three objectives are
\begin{align}
    \Rphi(\pi) &= \mathbb{E}_{p(\theta,\phi)\,p(\Data\mid\theta,\phi,\pi)}\!\left[\log\frac{p(\Data\mid\theta,\pi)}{p(\Data\mid\pi)}-\beta\log\frac{p(\Data\mid\phi,\pi)}{p(\Data\mid\pi)}\right],
    \label{eq:app_R_Lu}\\
    \Rcond(\pi) &= \mathbb{E}_{p(\theta,\phi)\,p(\Data\mid\theta,\phi,\pi)}\!\left[\log\frac{p(\Data\mid\theta,\pi)}{p(\Data\mid\pi)}-\beta\log\frac{p(\Data\mid\theta,\phi,\pi)}{p(\Data\mid\theta,\pi)}\right],
    \label{eq:app_R_Lc}\\
    \Rc(\pi) &= \mathbb{E}_{p(\theta,\phi)\,p(\Data\mid\theta,\phi,\pi)}\!\left[\log\frac{p(\Data\mid\theta,\pi)}{p(\Data\mid\pi)}-\beta\log\frac{p(\Data\mid\theta,\phi,\pi)\,p(\Data\mid\pi)}{p(\Data\mid\theta,\pi)\,p(\Data\mid\phi,\pi)}\right],
    \label{eq:app_R_Cpl}
\end{align}
where the last penalty is the posterior mutual information $\Cpldw$ written through Bayes' rule, and~\eqref{eq:app_R_Cpl} is the difference of~\eqref{eq:app_R_Lc} and~\eqref{eq:app_R_Lu} plus $\mathrm{EIG}_\theta$, in accordance with $\Fc=\Fcond-\Fphi$ (\Cref{lem:decomposition}).
In general BED the utility $\mathrm{EIG}_\theta(\pi)$ is valued under the working model and the penalty under the assessment models indexed by $m\in\omega$; in our location finding experiments the two coincide.
In our experiments we do not average over models, so $\omega$ carries no model index for the data to identify, and the remaining context (empty in location finding, the test input $x_*$ in active learning) is drawn independently of the data, $I(\omega;\Data)=0$.

\paragraph{Estimation.}
The marginal likelihoods $p(\Data\mid\theta,\pi)$, $p(\Data\mid\phi,\pi)$ and $p(\Data\mid\pi)$ are intractable, so each log-ratio is replaced by a nested Monte Carlo (NMC) or sequential prior contrastive (sPCE) estimator~\citep{rainforth2018nesting,foster2020unified,foster2021dad}, which give upper and lower bounds on the corresponding information gain respectively (\Cref{app:pce_nmc_estimators}).
Because the penalties enter with a negative sign, a lower bound on $R$ would require a lower bound on $\mathrm{EIG}_\theta$ together with an upper bound on the penalty; in practice the two estimators are tight at the number of contrastive samples we use (the NMC and sPCE estimates differ by less than $0.008$ nats for every quantity we report), so we replace each term of the objective by its sPCE (or NMC) estimator when optimising, report sPCE estimates, and estimate the coupling $\Fc$ by the difference of the estimators of $\Fcond$ and $\Fphi$ (\Cref{app:location_finding}).
For adaptive design, $\pi$ is a design policy, $\Data$ is the full history $h_T$ generated by rolling it out, and the same estimators apply with the sequential bounds of \citet{foster2021dad}, so that FairBED can be optimised end-to-end as in DAD.

\paragraph{Gradients.}
Let $A(\theta,\phi,\Data,\pi)$ denote the (estimated) integrand of the chosen objective; it depends on $\pi$ through the designs that enter each likelihood term.
When the likelihood is reparameterisable, $\Data=g(\varepsilon,\theta,\phi,\pi)$ with base noise $\varepsilon\sim p(\varepsilon)$ independent of $\pi$ (as in location finding, where $\varepsilon$ collects the observation noise), the gradient is purely pathwise,
\begin{equation*}
    \nabla_\pi R(\pi)=\mathbb{E}_{p(\theta,\phi)\,p(\varepsilon)}\!\left[\nabla_\pi A\big(\theta,\phi,g(\varepsilon,\theta,\phi,\pi),\pi\big)\right],
\end{equation*}
which is the form realised by automatic differentiation in our implementation, with gradients backpropagated to the parameters of the design network when designs are produced by one.
Otherwise the standard score-function form applies,
\begin{equation*}
    \nabla_\pi R(\pi)=\mathbb{E}_{p(\theta,\phi)\,p(\Data\mid\theta,\phi,\pi)}\!\left[A\,\nabla_\pi\log p(\Data\mid\theta,\phi,\pi)+\nabla_\pi A\right].
\end{equation*}

\paragraph{Active learning.}
In active learning the terms of~\eqref{eq:fairbed_AL} are mutual informations between the candidate's label pair $y=(y_\theta,y_\phi)$ and the test labels $(y_{*\theta},y_{*\phi})$, all of which are discrete.
They are computed from the joint predictive $p(y,y_{*\theta},y_{*\phi}\mid x,x_*,\Data_{t-1})=\mathbb{E}_{p(\upsilon\mid\Data_{t-1})}\big[p(y\mid x,\upsilon)\,p(y_{*\theta},y_{*\phi}\mid x_*,\upsilon)\big]$, where the expectation over the latent model variables $\upsilon$ is an average over ensemble members (trees) and the expectation over $x_*\sim p_*$ is a Monte Carlo average over inputs drawn from the pool, exactly as for EPIG~\citep{pmlr-v206-bickfordsmith23a}; the penalties $\Fphi(x)$ and $\Fc(x)$ are marginals and conditionals of the same joint.
Under separate models the joint predictive factorises and only the marginal-leakage penalty is non-trivial (\Cref{prop:al_no_coupling}).

\paragraph{Complexity.}
The acquisition is a standard EPIG-based active learning loop~\citep{pmlr-v206-bickfordsmith23a}, so we only state its cost.
With a pool of $N$ candidates, $M$ target inputs $x_*$ and $K$ ensemble members, each joint predictive above is an average of $K$ products over at most $4\times4$ label pairs, so scoring the pool costs $O(NMK)$ time per acquisition, the same order as EPIG, with the penalties adding only a constant factor.
Memory is $O((N+M)K)$ for the per-member predictive probabilities, with the candidate--target terms computed in chunks of candidates.
Refitting the forests on at most a few hundred labelled points is negligible in comparison, so $T$ acquisitions cost $O(TNMK)$ in total.
In general BED the objectives are estimated with the sPCE and NMC estimators of \citet{foster2021dad} at the cost of standard BED with the same number of contrastive samples, times the number of information gains in the objective (two for $\Rphi$ and $\Rcond$, three for $\Rc$).

\subsection{EIG Estimators}
\label{app:pce_nmc_estimators}

The Nested Monte Carlo (NMC) estimator \citep{rainforth2018nesting} provides an upper bound on the EIG,
\begin{align}
    \text{EIG}(\xi) \leq \mathbb{E}_{p(\theta_0,y|\xi)p(\theta_{1:L})} \left[ \log \frac{p(y|\theta_0,\xi)}{\frac{1}{L} \sum_{l=1}^L p(y | \theta_l,\xi)} \right]
\end{align}
with $\theta_0, y \sim p(\theta_0, y|\xi)$ and $\theta_{1:L} \sim p(\theta)$, while the (sequential) prior contrastive estimator (sPCE) \citep{foster2020unified,foster2021dad} provides a lower bound,
\begin{align}
    \text{EIG}(\xi) \geq \mathbb{E}_{p(\theta_0,y|\xi)p(\theta_{1:L})} \left[ \log \frac{p(y|\theta_0,\xi)}{\frac{1}{L+1} \sum_{l=0}^L p(y | \theta_l,\xi)} \right],
\end{align}
with the same sampling scheme.
Both become exact as $L\to\infty$.
In our implementation, the conditional information gain $\Fcond=\mathrm{EIG}_{\phi\mid\theta}$ is estimated in the same way with $\theta$ held fixed in both the numerator and the contrastive samples, and $\Fphi=\mathrm{EIG}_\phi$ with the likelihood $p(y\mid\phi,\xi)=\mathbb{E}_{p(\theta)}[p(y\mid\theta,\phi,\xi)]$ marginalised over $\theta$ by Monte Carlo (\Cref{tab:parameter_comparison_pretrain}).


\section{Experiment Details}\label{app:experiment_details}

\subsection{Location Finding}\label{app:location_finding}

\paragraph{Setting.}
A single hidden source lies at $(\theta,\phi)\in\mathbb{R}^2$, where $\theta$ is the target coordinate and $\phi$ the sensitive one, with independent priors $\theta,\phi\sim\mathcal{N}(3,1)$.
The benchmark of \citet{foster2021dad} and others \cite{hedman2025stepdad, ivanova2021implicit} measures the source with isotropic sensors, whose readings depend only on the distance to the source.
We adapt it to source localisation with collimated detectors.
A collimator restricts the directions from which a detector receives signal: a collimated sensor responds mainly to the offset of the source along one direction and is nearly blind to the offset across it, at the cost of collecting fewer counts.
The environment also contains a fixed, known background that adds noise to sensors aligned with the target axis, the sensors that would otherwise measure $\theta$ most fairly.

\paragraph{Designs.}
Each experiment $t$ chooses a design $\xi_t=(\mathrm{x}_t,\lambda_t,\orient_t)$: a sensor position $\mathrm{x}_t\in\mathbb{R}^2$, a degree of collimation $\lambda_t\in[\lambda_{\min},1]$, with $\lambda_t{=}1$ an uncollimated sensor and small $\lambda_t$ a narrow slit, and an orientation $\orient_t\in(-\pi/2,\pi/2)$, the direction along which the sensor responds.
The design network outputs unconstrained values $(u_\lambda,u_\orient)$, mapped to $\lambda=\lambda_{\min}+(1-\lambda_{\min})\,\sigma(u_\lambda)$ and $\orient=\tfrac{\pi}{2}\tanh(u_\orient)$.

\paragraph{Likelihood.}
With offset $\Delta=(\theta,\phi)-\mathrm{x}$, the components of the offset along and across the sensor's orientation are
\begin{equation}
    a=\cos\orient\,\Delta_\theta+\sin\orient\,\Delta_\phi,
    \qquad
    c=-\sin\orient\,\Delta_\theta+\cos\orient\,\Delta_\phi,
\end{equation}
and the sensor measures the anisotropic squared distance $d^2=a^2+\lambda\,c^2$.
The mean log-intensity is
\begin{equation}
    \mu(\theta,\phi,\xi)=\log\left(b+\frac{\alpha}{m+d^2}\right),
\end{equation}
where $b>0$ is a constant background level, $m$ a fixed offset controlling the maximum signal and $\alpha$ the source strength.
Observations are log-normal, $y\mid\theta,\phi,\xi\sim\mathcal{N}\bigl(\mu(\theta,\phi,\xi),\,\sigma(\xi)^2\bigr)$, with design-dependent noise
\begin{equation}
    \sigma(\xi)^2=\bigl(\sigma_0\,\lambda^{-\gamma}\bigr)^2+\bigl(b_{\mathrm{bg}}\,s(\lambda,\orient)\bigr)^2,
    \qquad
    s(\lambda,\orient)=\exp\!\left(-(1-\lambda)\,(\orient/w)^2\right).
\end{equation}
The first term is instrument noise, which grows as collimation reduces the counts collected; the second is the background, of strength $b_{\mathrm{bg}}$, seen fully by uncollimated sensors and by slits aligned with the target axis, and falling off once a slit is rotated by more than about the angular width $w$.

\paragraph{Design families.}
A slit aligned with the target axis ($\orient{=}0$, small $\lambda$) measures $\Delta_\theta$ and, given $\theta$, is almost independent of $\phi$, as data from a fair world would be; it pays for this with collimation and background noise.
A rotated slit measures a combination of $\Delta_\theta$ and $\Delta_\phi$: on its own it reveals little about $\phi$, since the variation in $\theta$ masks it, but once $\theta$ is known it reveals $\phi$, so it creates coupling.
An uncollimated sensor reveals both coordinates through the distance.
Standard BED, which values designs by $\mathrm{EIG}_\theta$ alone, almost always chooses the less noisy uncollimated sensors (Fig.~\ref{fig:loc_find_designs_static}).

\paragraph{Relation to the original benchmark.}
Setting $\lambda_{\min}{=}1$ makes every sensor uncollimated: then $d^2=\lVert\Delta\rVert^2$ regardless of $\orient$ and $s{=}1$, and with no background, $b_{\mathrm{bg}}{=}0$, the noise is the constant $\sigma_0$.
The model then reduces exactly to the isotropic benchmark of \citet{foster2021dad}, with designs $\mathrm{x}_t$ only.

\paragraph{Baselines and evaluation.}
Standard BED maximises an sPCE lower bound on $\mathrm{EIG}_\theta$ with the same design network as FairBED, and is trained alongside it in every run.
As the context is empty in this experiment, $\Fphi=\mathrm{EIG}_\phi$ and $\Fcond=\mathrm{EIG}_{\phi\mid\theta}$, the expected information gains about $\phi$ without and with $\theta$ known, and $\Fc=\mathrm{EIG}_{\phi\mid\theta}-\mathrm{EIG}_\phi$; the figures below are labelled in this EIG notation.
Random designs draw every unconstrained design coordinate independently from $\mathcal{N}(0,1)$, so positions are centred on the origin rather than on the prior mean, and random sensors are partially collimated with widely spread orientations.
We report sPCE lower bounds on $\mathrm{EIG}_\theta$, $\Fphi$ and $\Fcond$ (\Cref{app:pce_nmc_estimators}); the corresponding NMC upper bounds differ from them by less than $0.008$ nats for every reported quantity.
The coupling $\Fc=\Fcond-\Fphi$ is estimated by the difference of the two lower bounds, which estimates but does not bound it.
Standard BED and Random are taken from one run per seed ($\Rcond$, $\beta{=}0.8$), so that their standard errors are over the same 8 seeds as the other methods.
Random after $t{=}1,\dots,10$ experiments (Figs.~\ref{fig:loc_find_frontier}--\ref{fig:loc_find_frontier_marginal}) is evaluated with the same estimators in separate runs without training, one per $(t,\text{seed})$; at $t{=}10$ they agree with the Random values above.

\subsubsection{Training details}
\label{app:training_details_loc-find}

Tables \ref{tab:encoder_loc_find_architecture} and \ref{tab:decoder_loc_find_architecture} describe the architecture of the DAD policy network, which takes and outputs the unconstrained design $(\mathrm{x},u_\lambda,u_\orient)$.
Tables \ref{tabapp:parameters} and \ref{tab:parameter_comparison_pretrain} describe the model and training hyper-parameters.
Static designs are initialised at $\mathrm{x}{=}0$, $u_\lambda{=}0$ and $u_\orient\sim\mathcal{N}(0,1)$; the random initial orientation avoids starting every slit at $\orient{=}0$, where the gradient with respect to $\orient$ can vanish by symmetry.
We train $\Rcond$, $\Rphi$ and $\Rc$ for $\beta\in\{0.3,0.5,0.8,1,2,5,10\}$ with 8 seeds each, and each run also trains standard BED.
A static run takes about 30 minutes on a single NVIDIA GH200 GPU.

\begin{table}[t]
    \centering
    \captionof{table}{\textbf{Source location finding.} Encoder network of the DAD policy, architecture following \citet{foster2021dad}.}
    \begin{tabular}{ c  c  c  c }
        \toprule
        Layer & Overview & Dimension & Activation \\
        \midrule
        Design-outcome & $\text{x}, u_\lambda, u_\orient, y$ & 5 & - \\
        H1 & Fully connected & 64 & RELU \\
        H2 & Fully connected & 256 & RELU \\
        Output & Fully Connected & 16 & - \\
        \bottomrule
    \end{tabular}
    \label{tab:encoder_loc_find_architecture}
\end{table}
\begin{table}[t]
    \centering
    \captionof{table}{\textbf{Source location finding.} Decoder network of the DAD policy, architecture following \citet{foster2021dad}.}
    \begin{tabular}{ c  c  c  c }
        \toprule
        Layer & Overview & Dimension & Activation \\
        \midrule
        Input & $E(h_t)$ & 16 & - \\
        H1 & Fully connected & 128 & RELU \\
        H2 & Fully connected & 16 & RELU \\
        Output & $\text{x}, u_\lambda, u_\orient$ & 4 & - \\
        \bottomrule
    \end{tabular}
    \label{tab:decoder_loc_find_architecture}
\end{table}

\begin{table}[t]
    \centering
    \caption{\textbf{Source location finding.} Model parameters. Setting $\lambda_{\min}{=}1$ and $b_{\mathrm{bg}}{=}0$ recovers the isotropic benchmark.}
    \begin{tabular}{l l}
        \toprule
        Parameter & Value \\
        \midrule
        Source strength, $\alpha$ & 1 \\
        Max signal, $m$ & $10^{-4}$ \\
        Base signal, $b$ & $10^{-1}$ \\
        Instrument noise scale, $\sigma_0$ & 0.5 \\
        Tightest collimation, $\lambda_{\min}$ & 0.03 \\
        Collimation noise exponent, $\gamma$ & 0.5 \\
        Background strength, $b_{\mathrm{bg}}$ & 2 \\
        Background angular width, $w$ & $15^\circ$ \\
        Prior on $(\theta,\phi)$ & $\mathcal{N}\bigl((3,3),I_2\bigr)$ \\
        Number of experiments, $T$ & 10 \\
        \bottomrule
    \end{tabular}
    
    \label{tabapp:parameters}
\end{table}

\begin{table}[t]
    \centering
    \caption{\textbf{Source location finding.} Training and evaluation parameters for static designs and DAD.}
    \begin{tabular}{lcc}
        \toprule
        Parameter & Static & DAD \\
        \midrule
        Batch size & 100 & 100 \\
        Number of negative samples & 1000 & 1000 \\
        Number of gradient steps & 10K & 10K \\
        Learning rate (LR) & 0.01 & 0.00005 \\
        Number of samples for marginalising a parameter & 1000 & 1000 \\
        Number of inner samples (evaluation) & 100000 & 100000 \\
        Evaluation batches $\times$ batch size & $100\times100$ & $100\times100$ \\
        Seeds & 8 & 8 \\
        \bottomrule
    \end{tabular}
    \label{tab:parameter_comparison_pretrain}
\end{table}

\subsubsection{Additional results}
\label{app:loc_find_results}

\paragraph{Selected designs.}
Fig.~\ref{fig:loc_find_designs_static} shows the static designs selected at $\beta{=}0.8$.
Standard BED spreads uncollimated sensors around the prior mean in both coordinates.
With static designs, $\Rcond$ places collimated slits along the line $\phi{=}3$, oriented close to the target direction ($6^\circ$ on average at $\beta{=}0.8$), which measure $\theta$ while keeping the reading nearly independent of $\phi$.
$\Rphi$ also collimates, but with slits rotated further from the target direction ($13^\circ$), which conceal $\phi$ from the reading on its own while tying it to $\theta$: averaged over 8 seeds at $\beta{=}1$, $82\%$ of its static designs are rotated slits against $21\%$ for $\Rcond$.
With DAD the difference is more masked, because the policies adapt: each design responds to the readings so far and to where the source appears to be, so orientations vary across rollouts and steps under both objectives (about $18^\circ$ on average at $\beta{=}0.8$ for each), and DAD policies also vary more between seeds; the difference between the objectives therefore shows in the information measures below rather than in the designs themselves.

\begin{figure}[t]
    \centering
    \includegraphics[width=\linewidth]{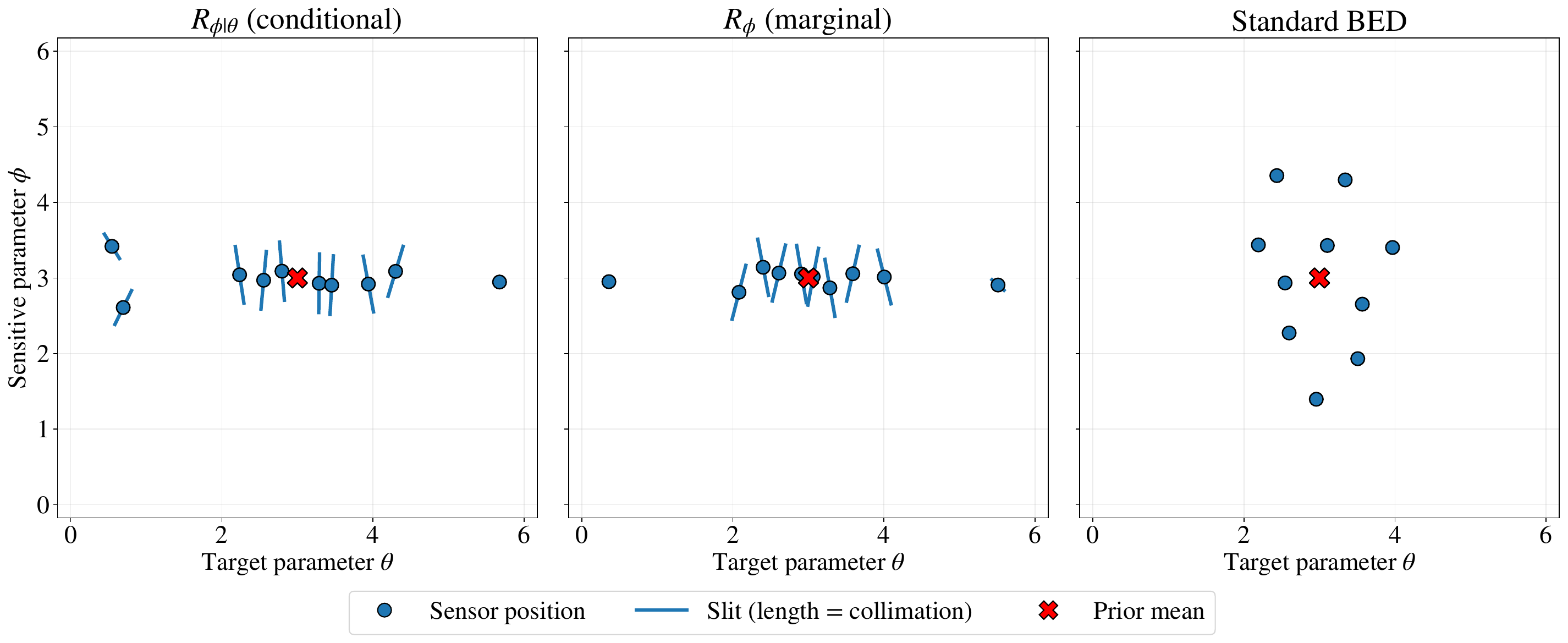}
    \caption{\textbf{Location finding designs (static, $T{=}10$, $\beta{=}0.8$, seed 3).}
    Each point is a sensor position; the segment through it shows the strip a collimated sensor is sensitive to, with length proportional to the degree of collimation (no segment: uncollimated). The red cross marks the prior mean $(3,3)$.
    A vertical segment is a slit aligned with the target axis, whose reading depends on $\theta$ and, given $\theta$, hardly on $\phi$; a tilted segment is a rotated slit, which hides $\phi$ while $\theta$ is unknown but reveals it once $\theta$ is known, creating coupling.
    Standard BED uses uncollimated sensors spread around the prior mean, which reveal both coordinates.
    Both FairBED objectives instead place collimated slits along $\phi{=}3$, the prior mean of $\phi$, where the reading depends least on $\phi$.
    An aligned slit pays for its fairness with background noise, which pollutes its reading of $\theta$; rotating the slit escapes the background and measures $\theta$ more precisely, but introduces coupling.
    $\Rcond$ penalises this coupling and keeps its slits close to vertical, while $\Rphi$ does not, and tilts them further.}
    \label{fig:loc_find_designs_static}
\end{figure}

\paragraph{Change in information with $\beta$.}
Fig.~\ref{fig:loc_find_pct_change_cond} shows the change in $\mathrm{EIG}_\theta$ and the conditional leakage $\Fcond$ relative to standard BED as $\beta$ increases, and Fig.~\ref{fig:loc_find_pct_change} the change in $\mathrm{EIG}_\theta$ and the marginal leakage $\Fphi$.
Both $\Rcond$ and $\Rphi$ remove almost all of $\Fphi$ from $\beta{=}2$, so the marginal leakage alone barely separates them.

\begin{figure}[t]
    \centering
    \begin{subfigure}[t]{0.48\linewidth}
        \centering
        \includegraphics[width=\linewidth]{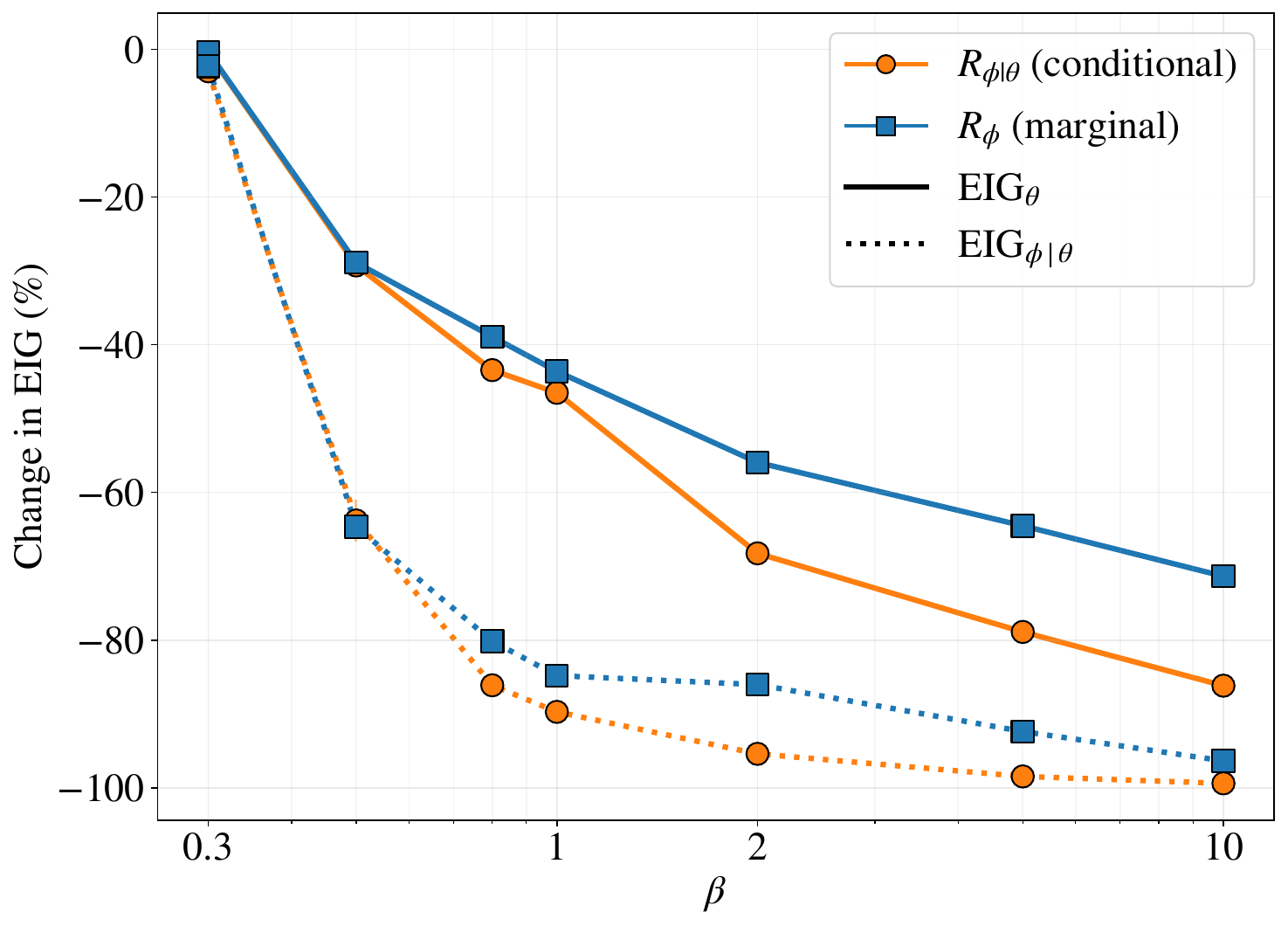}
        \caption{Static.}
    \end{subfigure}
    \hfill
    \begin{subfigure}[t]{0.48\linewidth}
        \centering
        \includegraphics[width=\linewidth]{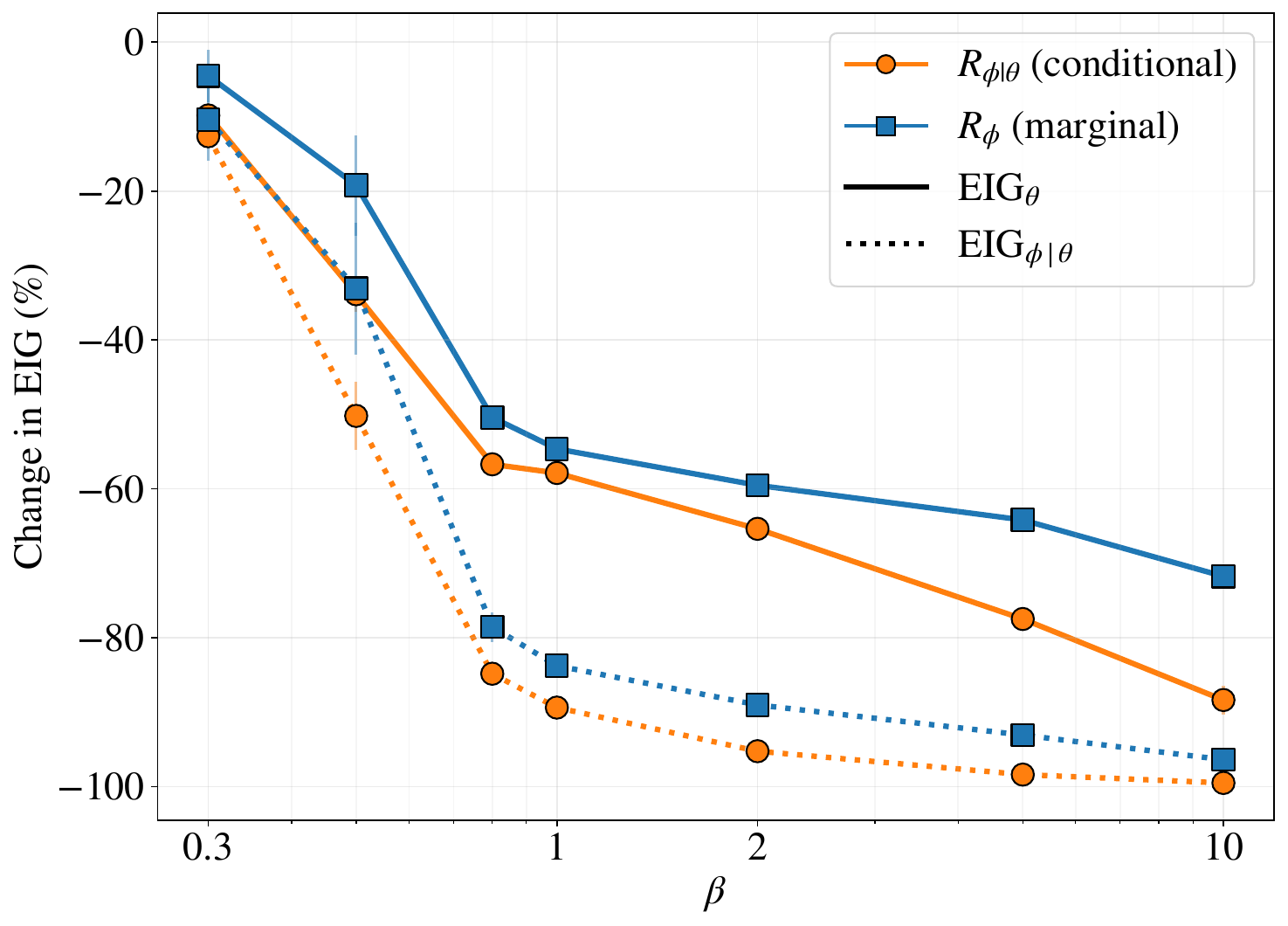}
        \caption{DAD.}
    \end{subfigure}
    \caption{\textbf{Change in $\mathrm{EIG}_\theta$ (solid) and $\Fcond=\mathrm{EIG}_{\phi\mid\theta}$ (dotted) relative to standard BED ($T{=}10$).}
    Lower-bound estimates; mean $\pm$ 1 s.e.\ over 8 seeds.}
    \label{fig:loc_find_pct_change_cond}
\end{figure}

\begin{figure}[t]
    \centering
    \begin{subfigure}[t]{0.48\linewidth}
        \centering
        \includegraphics[width=\linewidth]{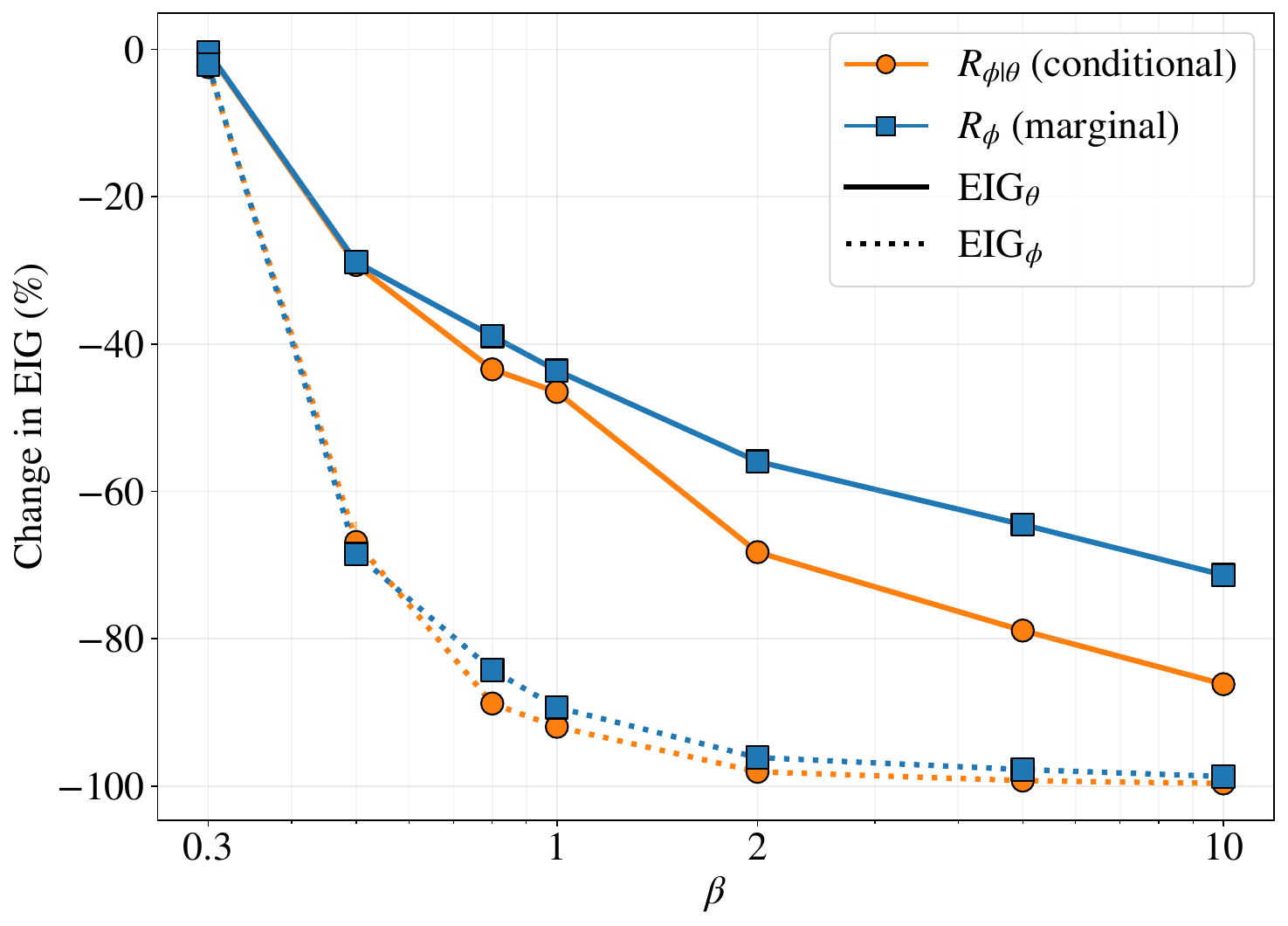}
        \caption{Static.}
    \end{subfigure}
    \hfill
    \begin{subfigure}[t]{0.48\linewidth}
        \centering
        \includegraphics[width=\linewidth]{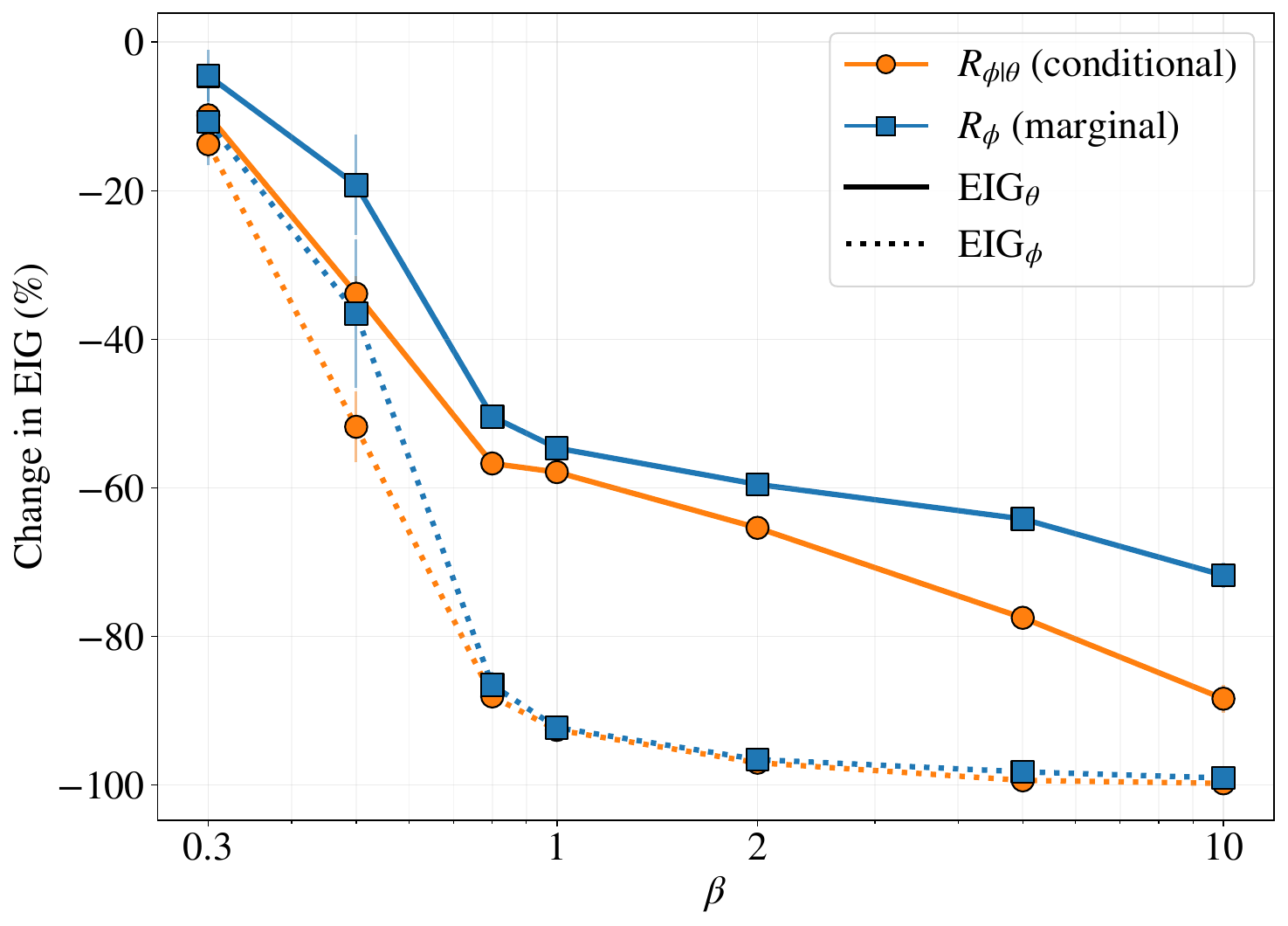}
        \caption{DAD.}
    \end{subfigure}
    \caption{\textbf{Change in $\mathrm{EIG}_\theta$ (solid) and $\Fphi=\mathrm{EIG}_\phi$ (dotted) relative to standard BED ($T{=}10$).}
    Lower-bound estimates; mean $\pm$ 1 s.e.\ over 8 seeds.}
    \label{fig:loc_find_pct_change}
\end{figure}

\paragraph{Coupling.}
Fig.~\ref{fig:loc_find_coupling} separates the coupling $\Fc=\Fcond-\Fphi$ from the marginal leakage.
Under $\Rcond$ the coupling falls steadily as $\beta$ grows, to $97\%$ (static) and $96\%$ (DAD) below standard BED at $\beta{=}10$.
Under $\Rphi$ it does not: much of the reduction in $\Fcond$ that $\Rphi$ achieves comes from the marginal leakage alone, sometimes at the price of more coupling than standard BED.
For DAD at $\beta{=}1$, $\Rphi$ removes $92\%$ of $\Fphi$ while its coupling rises $23\%$ above standard BED, so its conditional leakage falls by only $84\%$, against $89\%$ for $\Rcond$.
Static designs under $\Rphi$ show the same effect at $\beta{=}2$, where the coupling is $31\%$ above standard BED as the designs revert to mostly uncollimated sensors on the line $\phi{=}3$; their readings reveal little about $\phi$ while $\theta$ is unknown but tie the two coordinates together through the distance.
Only at $\beta{=}10$ does the coupling under $\Rphi$ fall well below standard BED, by $69\%$ (static) and $63\%$ (DAD).
This is the failure that penalising the marginal leakage alone permits: the data hide $\phi$ on their own while teaching a relationship between $\phi$ and $\theta$.

\begin{figure}[t]
    \centering
    \begin{subfigure}[t]{0.48\linewidth}
        \centering
        \includegraphics[width=\linewidth]{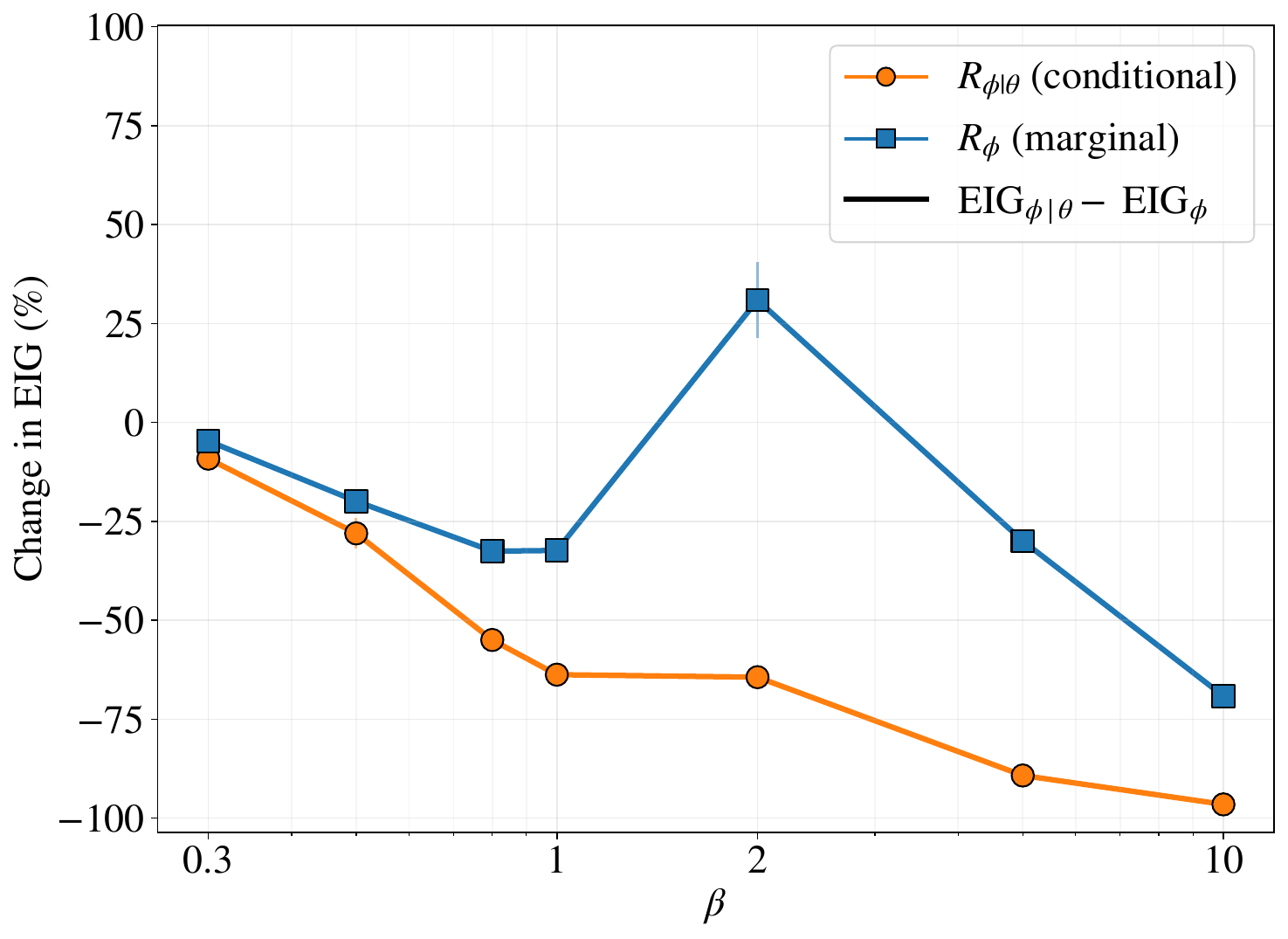}
        \caption{Static.}
    \end{subfigure}
    \hfill
    \begin{subfigure}[t]{0.48\linewidth}
        \centering
        \includegraphics[width=\linewidth]{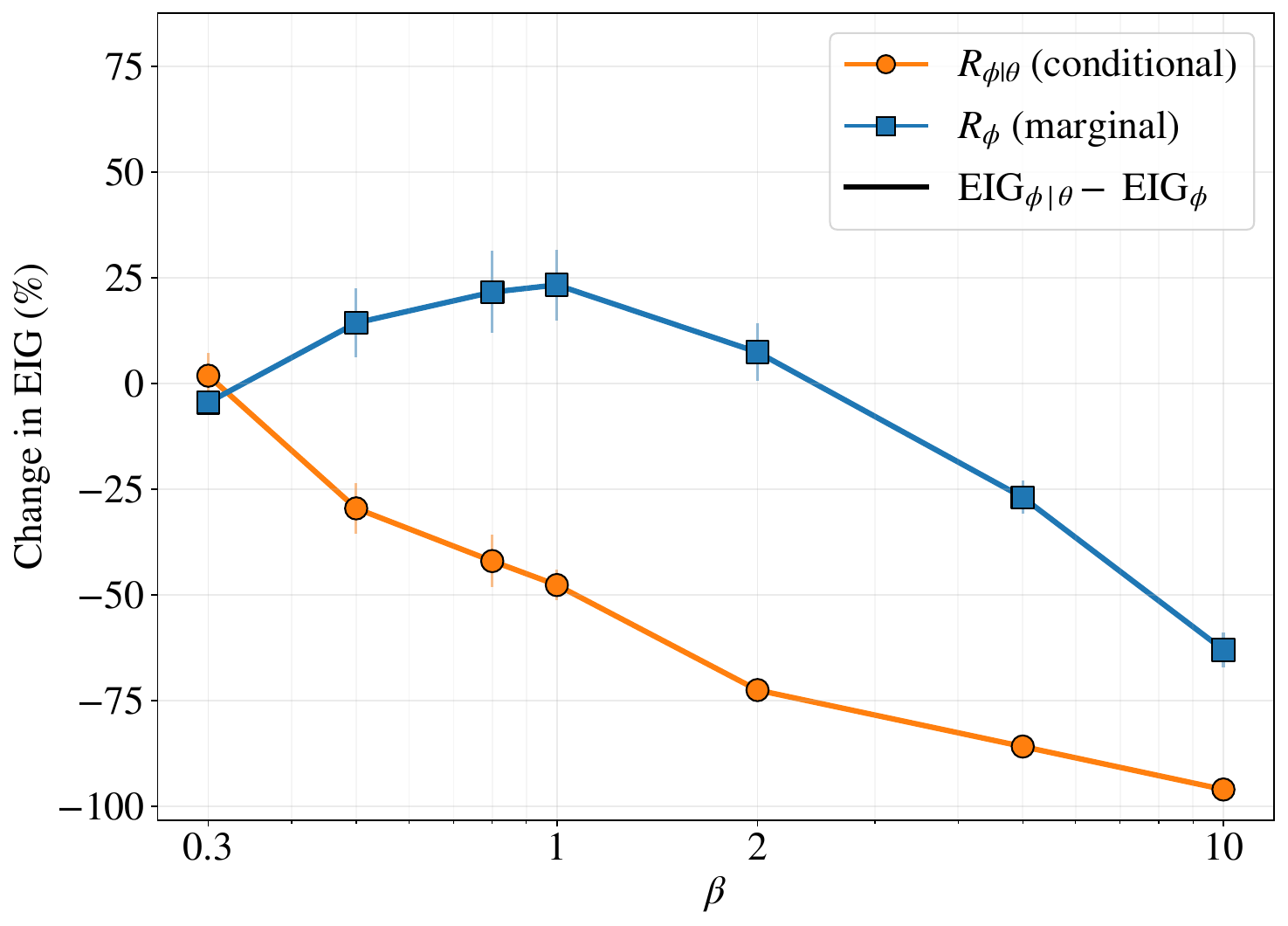}
        \caption{DAD.}
    \end{subfigure}
    \caption{\textbf{Change in the coupling $\Fc=\mathrm{EIG}_{\phi\mid\theta}-\mathrm{EIG}_\phi$ relative to standard BED ($T{=}10$).}
    Values above $0$ mean more coupling than standard BED. Mean $\pm$ 1 s.e.\ over 8 seeds.}
    \label{fig:loc_find_coupling}
\end{figure}

\begin{table}[t]
\caption{\textbf{Location finding} ($T{=}10$, $\beta{=}5$).
    Ratios of lower-bound EIG estimates, mean $\pm$ s.e.\ over 8 seeds; $0$ is optimal.
    $\Fc=\Fcond-\Fphi$ is the coupling.
    $\Rcond$ penalises $\Fcond$ and $\Rphi$ only $\Fphi$.
    A given $\beta$ is not directly comparable across variants, so Figs.~\ref{fig:loc_find_frontier}--\ref{fig:loc_find_frontier_marginal} compare them at equal $\mathrm{EIG}_\theta$ instead.}
\label{tab:loc_find_ratio_only}

\centering
\small
\setlength{\tabcolsep}{3pt}

\begin{tabular}{@{}l c c c@{}}
    \toprule
    Method & $\Fc / \mathrm{EIG}_\theta$
    & $\Fphi / \mathrm{EIG}_\theta$
    & $\Fcond / \mathrm{EIG}_\theta$ \\
    \midrule
    Random & $0.348 \pm 0.007$ & $1.237 \pm 0.014$ & $1.584 \pm 0.020$ \\
    BED (Static) & $0.087 \pm 0.002$ & $1.002 \pm 0.005$ & $1.089 \pm 0.004$ \\
    BED (DAD) & $0.072 \pm 0.003$ & $0.940 \pm 0.005$ & $1.012 \pm 0.006$ \\
    $\Rphi$ (Static) & $0.172 \pm 0.004$ & $0.064 \pm 0.002$ & $0.236 \pm 0.005$ \\
    $\Rphi$ (DAD) & $0.150 \pm 0.004$ & $0.047 \pm 0.002$ & $0.196 \pm 0.004$ \\
    \textbf{$\boldsymbol{\Rcond}$ (Static)} & $\mathbf{0.044 \pm 0.002}$ & $\mathbf{0.035 \pm 0.001}$ & $\mathbf{0.079 \pm 0.002}$ \\
    \textbf{$\boldsymbol{\Rcond}$ (DAD)} & $\mathbf{0.045 \pm 0.004}$ & $\mathbf{0.026 \pm 0.001}$ & $\mathbf{0.071 \pm 0.005}$ \\
    \bottomrule
\end{tabular}%
\end{table}

\paragraph{Comparison at equal information about $\theta$.}
At the same $\beta$, $\Rcond$ gives up more $\mathrm{EIG}_\theta$ than $\Rphi$, so Figs.~\ref{fig:loc_find_frontier}--\ref{fig:loc_find_frontier_marginal} compare the objectives at equal $\mathrm{EIG}_\theta$, with standard BED as the $\beta{=}0$ end point.
We summarise each seed by the mean, over the range of $\mathrm{EIG}_\theta$ both objectives reach, of the difference between the lowest coupling that $\Rphi$ and $\Rcond$ attain with at least that much $\mathrm{EIG}_\theta$.
$\Rcond$ attains lower coupling in every seed, by $0.019\pm0.002$ nats (static) and $0.024\pm0.004$ nats (DAD).
For the conditional leakage as a whole, $\Rcond$ is lower for DAD ($0.067\pm0.021$ nats) but not for static designs ($0.001\pm0.008$ nats): at equal $\mathrm{EIG}_\theta$, static $\Rcond$ designs leak about as much in total, but mostly as marginal leakage rather than coupling, so they reveal $\phi$ without teaching how it relates to $\theta$.
$\Rc$, which penalises the coupling alone, attains lower coupling still: at equal $\mathrm{EIG}_\theta$ it is below $\Rcond$ in every seed, by $0.042\pm0.003$ nats (static) and $0.040\pm0.005$ nats (DAD).
It does so while leaving the marginal leakage largely in place (Figs.~\ref{fig:loc_find_frontier_coupling_dad} and~\ref{fig:loc_find_frontier_marginal}).
Its conditional leakage at equal $\mathrm{EIG}_\theta$ is therefore higher than that of $\Rcond$ for DAD ($0.154\pm0.064$ nats) and similar for static designs ($0.033\pm0.029$ nats).
As in active learning (Fig.~\ref{fig:al_targeted_objectives}), penalising one component of the leakage reduces that component, and only $\Rcond$ reduces both.
Random designs are shown after each of $t{=}1,\dots,10$ experiments and lie above and to the left of the $\Rcond$ curves throughout: at $\beta{=}10$, $\Rcond$ learns more about $\theta$ than ten random experiments ($\mathrm{EIG}_\theta$ of $0.25$ and $0.27$ against $0.17$) while creating less coupling than a single one ($0.005$ and $0.006$ against $0.007$ nats).

\begin{figure}[t]
    \centering
    \begin{subfigure}[t]{0.8\linewidth}
        \centering
        \includegraphics[width=\linewidth]{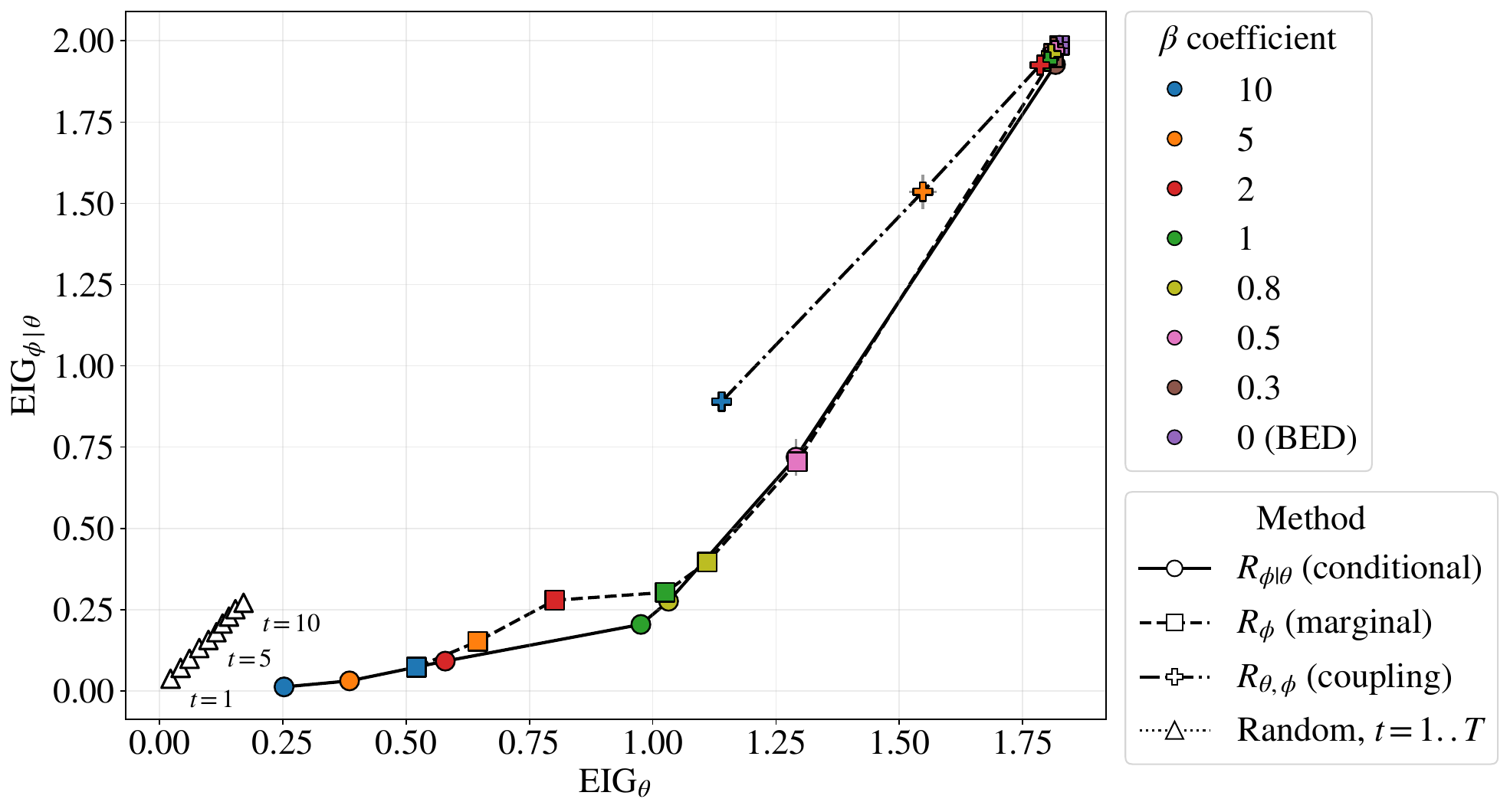}
        \caption{Static.}
    \end{subfigure}

    \begin{subfigure}[t]{0.8\linewidth}
        \centering
        \includegraphics[width=\linewidth]{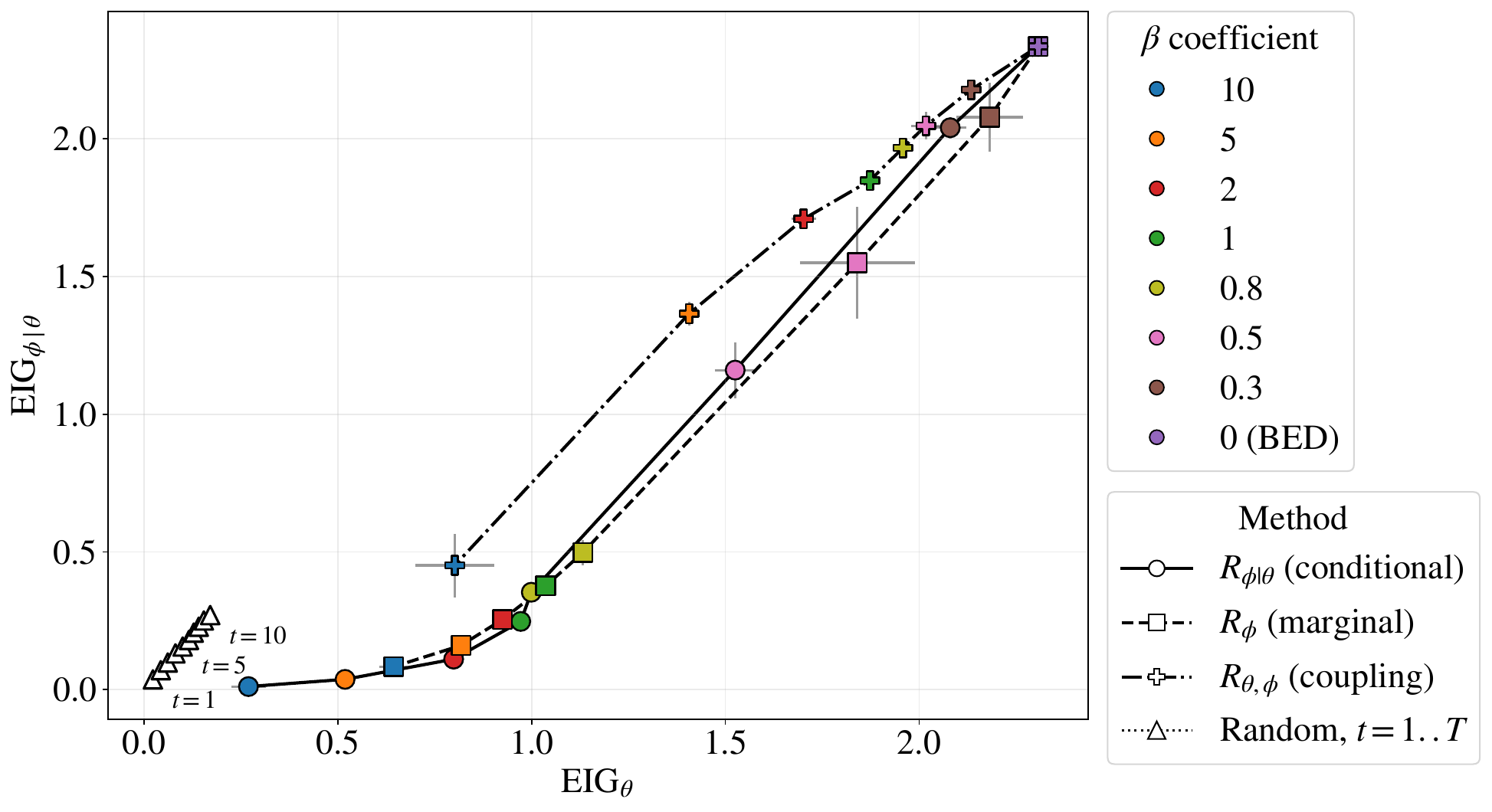}
        \caption{DAD.}
    \end{subfigure}
    \caption{\textbf{$\mathrm{EIG}_\theta$ against the conditional leakage $\Fcond=\mathrm{EIG}_{\phi\mid\theta}$ across $\beta$ ($T{=}10$).}
    Lower right is better. Line style and marker shape give the FairBED variant and marker colour gives $\beta$, with standard BED as $\beta{=}0$; the dotted line with hollow triangles is Random after $t{=}1,\dots,10$ experiments. Mean $\pm$ 1 s.e.\ over 8 seeds.}
    \label{fig:loc_find_frontier}
\end{figure}

\begin{figure*}[t]
    \centering
    \includegraphics[width=\linewidth]{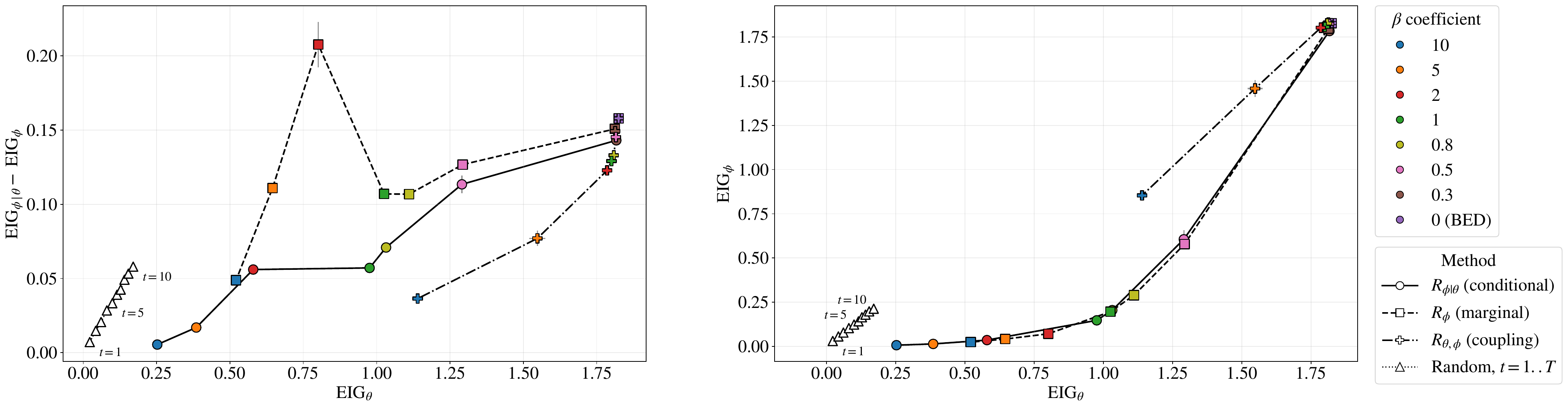}
    \caption{\textbf{Location finding (static, $T{=}10$).}
    $\mathrm{EIG}_\theta$ against the coupling $\Fc=\mathrm{EIG}_{\phi\mid\theta}-\mathrm{EIG}_\phi$ (left) and the marginal leakage $\Fphi=\mathrm{EIG}_\phi$ (right) across $\beta$, as Fig.~\ref{fig:loc_find_frontier}; the DAD version is Fig.~\ref{fig:loc_find_frontier_coupling_dad}.}
    \label{fig:loc_find_frontier_coupling}
    \label{fig:loc_find_frontier_marginal}
\end{figure*}


\subsection{Active Learning}
\label{app:al_details}

\looseness=-1
\textbf{Synthetic data.}~~
Each individual has a group $y_\phi=\phi\sim\mathrm{Bern}(\tfrac12)$, with $s=2\phi-1$, and six inputs: two proxies $x_1=0.75s+\epsilon_1$ and $x_2=0.5s+\epsilon_2$, two legitimate inputs $x_3,x_4$ and two noise inputs $x_5,x_6$, with all $\epsilon_j$ and $x_3,\dots,x_6$ standard normal.
The target satisfies $\operatorname{logit}p(y_\theta{=}1\mid x,\phi)=x_3+x_4+0.7x_1+\delta s$, where $\delta=0.8$ is a direct effect of the group on the outcome at fixed inputs.
Since $p(\phi{=}1\mid x)=\sigma(1.5x_1+x_2)$, a group-blind predictor can recover the group from the proxies and so absorb the direct effect, which is the relationship FairBED should avoid teaching.
We draw a pool of $3{,}000$ individuals and a test set of $2{,}000$, both fresh for every seed.

\looseness=-1
\textbf{Implementation.}~~
All forests use the default scikit-learn settings ($100$ trees) and are refit after every acquisition; the whole pool is scored at every step.
The active learning experiments were run on a single node with an NVIDIA H100 GPU, with the forests fitted and scored in parallel on its CPU cores.
Evaluation forests are trained on the acquired data with a different random seed from the acquisition forests, and the attribute sensitivity uses a forest that additionally takes $y_\phi$ as input.
The downstream models of \Cref{subsec:downstream} are an MLP and XGBoost trained on the same acquired data.
All variants use the per-step objective~\eqref{eq:fairbed_AL}: $\Rcond$ penalises the conditional leakage of the new pair, $\Rphi$ its marginal leakage, and $\Rc$ the increase in coupling it causes, computed as the difference of the first two~\eqref{eq:al_penalty_decomposition}; the estimation of these terms is described in \Cref{app:app_fairbed_bed_detail}.
$\Rcond$ and $\Rc$ use a single four-class forest over $(y_\theta,y_\phi)$ as the joint model; an autoregressive pair of forests $p(y_\theta\mid x)\,q(y_\phi\mid x,y_\theta)$, whose second forest takes the outcome as an input, is an equally valid alternative.
The heuristic fair entropy (HFE) baseline in \Cref{fig:app_al_curves,fig:app_al_coupling_beta} replaces both information gains with predictive entropies, $\mathbb{H}[p(y_\theta\mid x)]-\beta\,\mathbb{H}[p(y_\phi\mid x)]$.

\looseness=-1
\textbf{Results on every dataset.}~~
\Cref{fig:app_al_census,fig:app_al_synthetic,fig:app_al_graduate} show every FairBED variant against both metrics on each dataset, and serve two purposes.
First, they confirm the result of \Cref{fig:al_targeted_objectives}: each variant does well on the leakage it targets, with $\Rphi$ lowering the sensitive test accuracy and $\Rc$ lowering the attribute sensitivity relative to Random at the same target accuracy.
Second, they show what each variant does to the leakage it does not target, and here neither moves much below Random: $\Rphi$ leaves the attribute sensitivity close to Random, or above it on Synthetic, and $\Rc$ leaves the sensitive test accuracy close to Random, or above it on Student Graduation.
This is the behaviour we want: each penalty acts on its own component rather than on fairness in general, so a given kind of leakage is reduced by targeting it explicitly, which is why separate objectives for the marginal leakage and the coupling are useful.
$\Rcond$ penalises both components at once, and in practice tends to favour one of them, with the balance set by the data: on Student Graduation it lowers the attribute sensitivity most, at the cost of a higher sensitive test accuracy than Random; on Census it mainly lowers the sensitive test accuracy, with smaller gains in attribute sensitivity; and on Synthetic it lowers both modestly.
\Cref{fig:app_al_curves} shows the attribute sensitivity as labels are acquired at $\beta{=}1$, and \Cref{fig:app_al_coupling_beta} shows it for $\Rc$ at each $\beta$: at $\beta{=}0$, where $\Rc$ reduces to EPIG and HFE to predictive entropy, it is close to Random (above it on Census), and it falls further below Random as $\beta$ grows.

\begin{figure*}[!t]
    \centering
    \includegraphics[width=\linewidth]{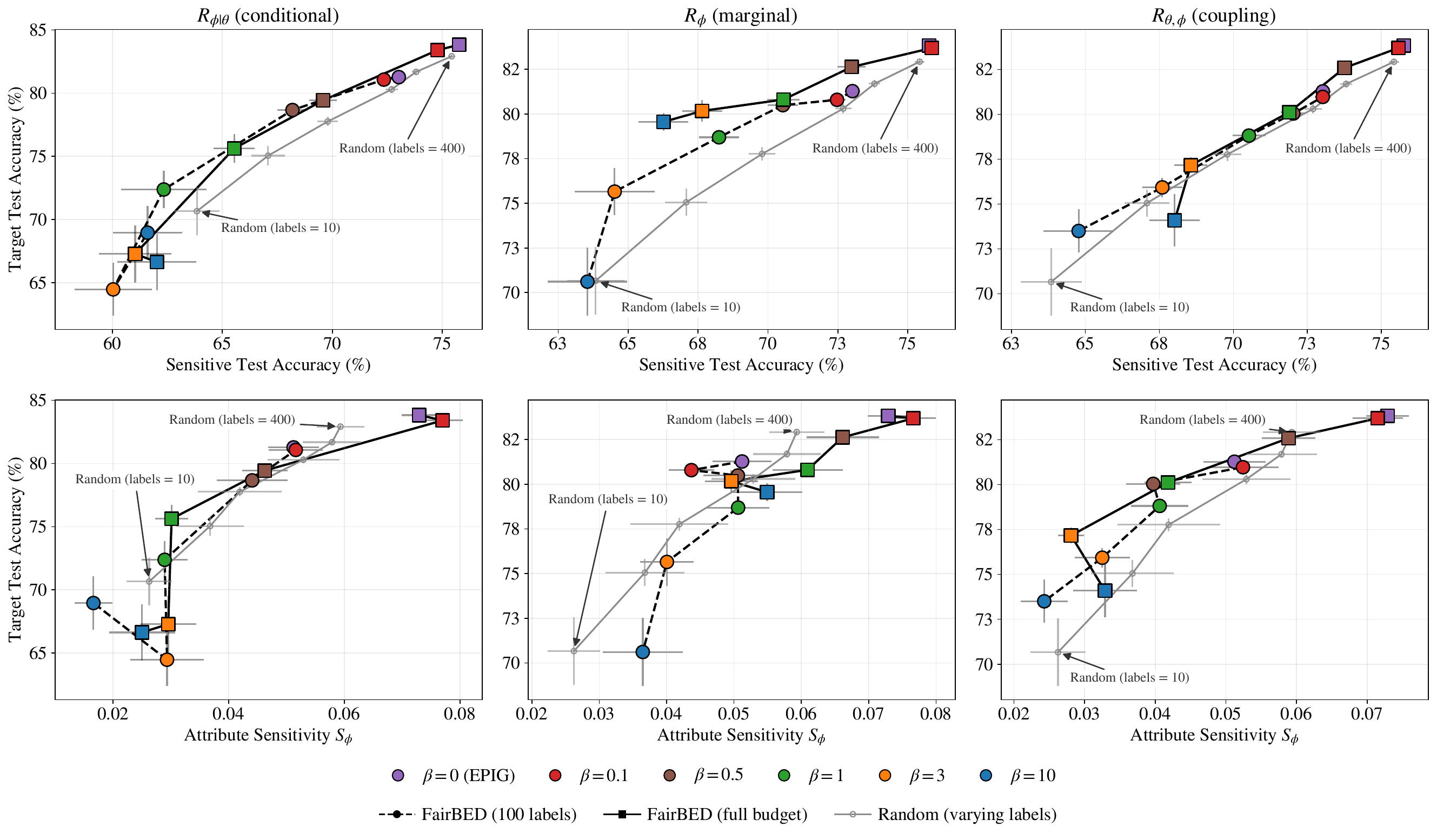}
    \caption{\textbf{Census}: target accuracy against sensitive test accuracy (top) and attribute sensitivity (bottom) after $100$ and $400$ labels, across $\beta$, for data acquired with each FairBED variant (columns); grey: Random at increasing label budgets. Mean $\pm$ s.e.\ over 24 seeds.}
    \label{fig:app_al_census}
\end{figure*}

\begin{figure*}[!t]
    \centering
    \includegraphics[width=\linewidth]{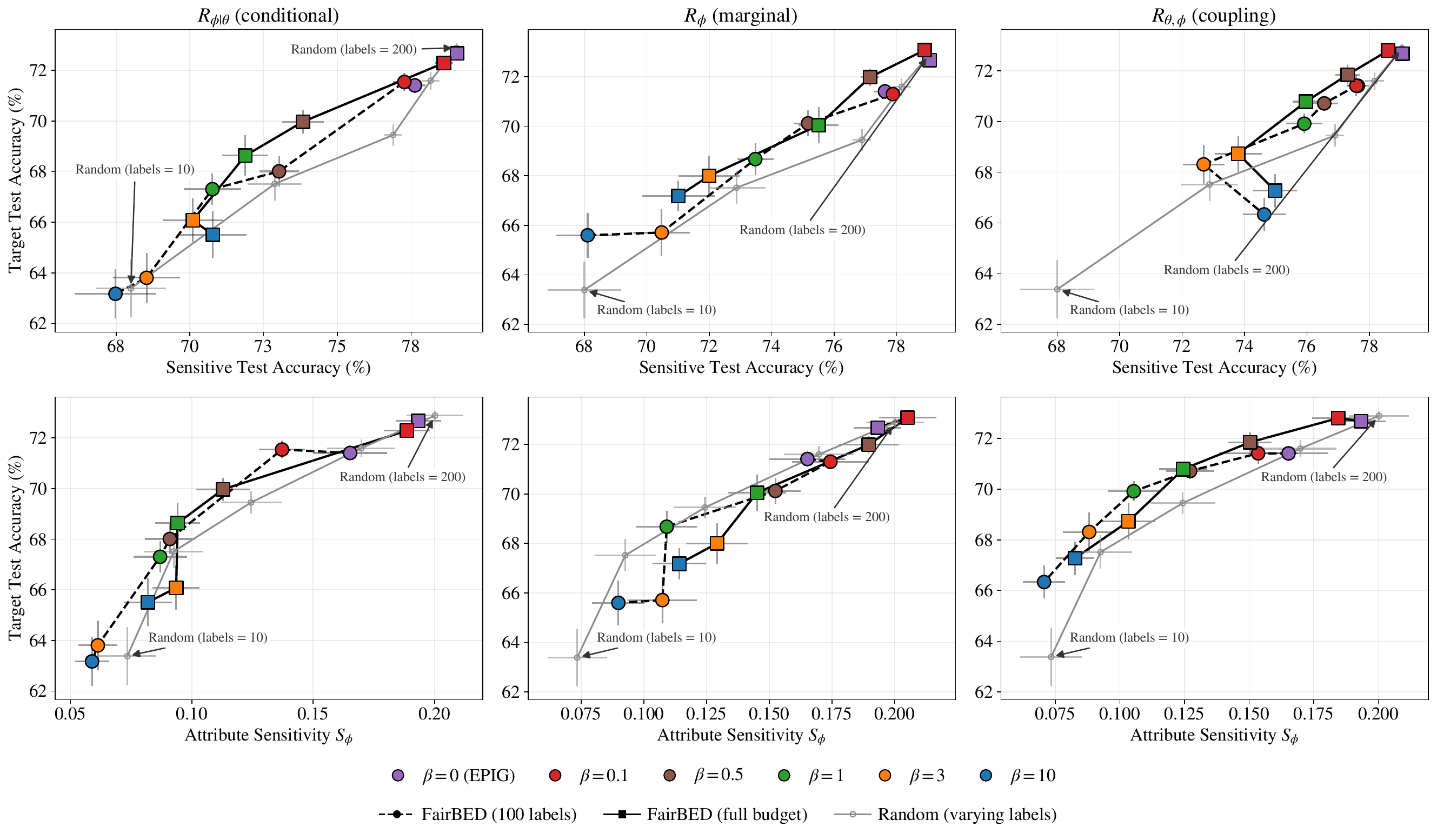}
    \caption{\textbf{Synthetic}, $100$ and $200$ labels, laid out as \Cref{fig:app_al_census}.}
    \label{fig:app_al_synthetic}
\end{figure*}

\begin{figure*}[!t]
    \centering
    \includegraphics[width=\linewidth]{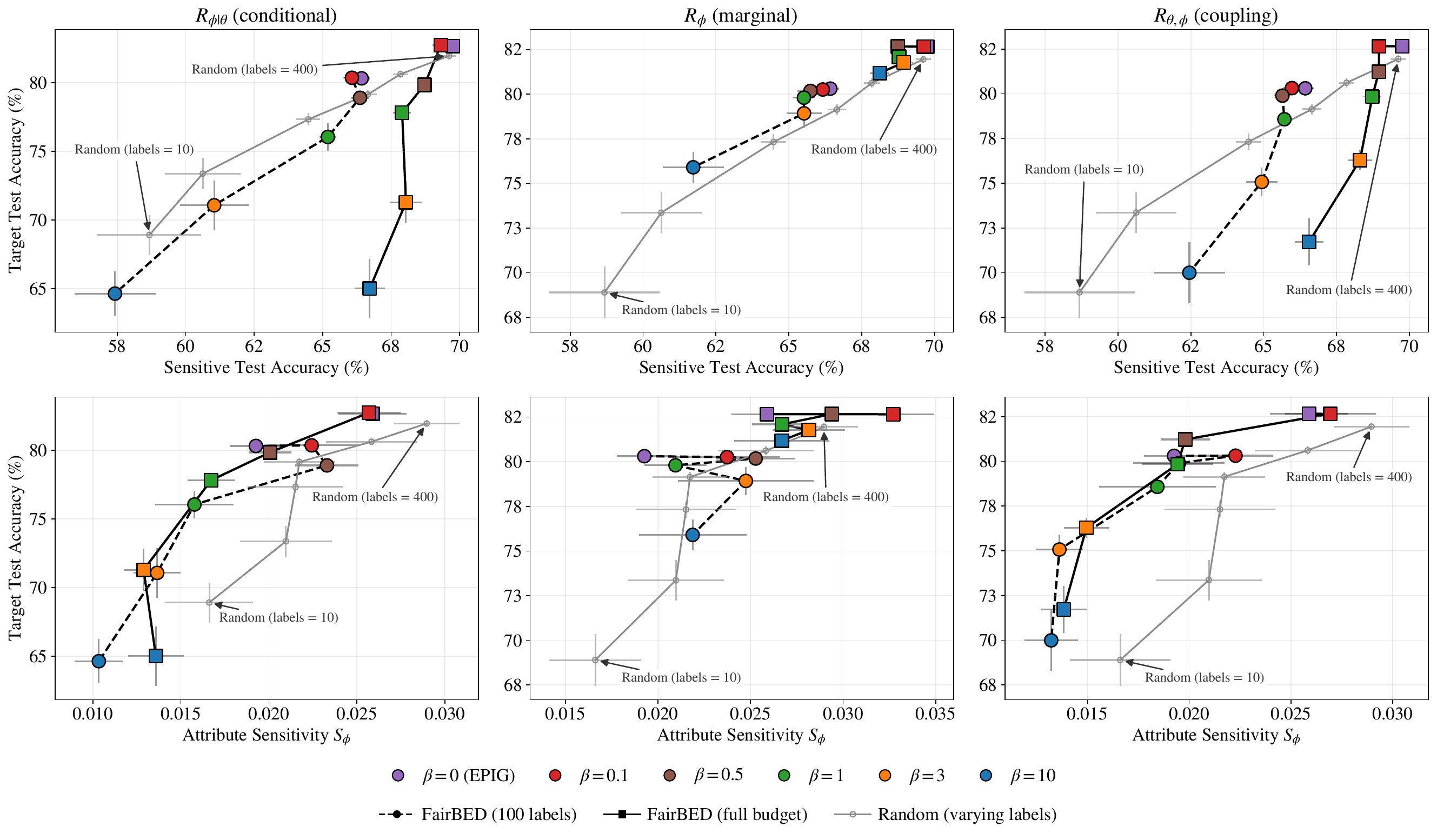}
    \caption{\textbf{Student Graduation}, $100$ and $400$ labels, laid out as \Cref{fig:app_al_census}.}
    \label{fig:app_al_graduate}
\end{figure*}

\begin{figure*}[!t]
    \centering
    \includegraphics[width=\linewidth]{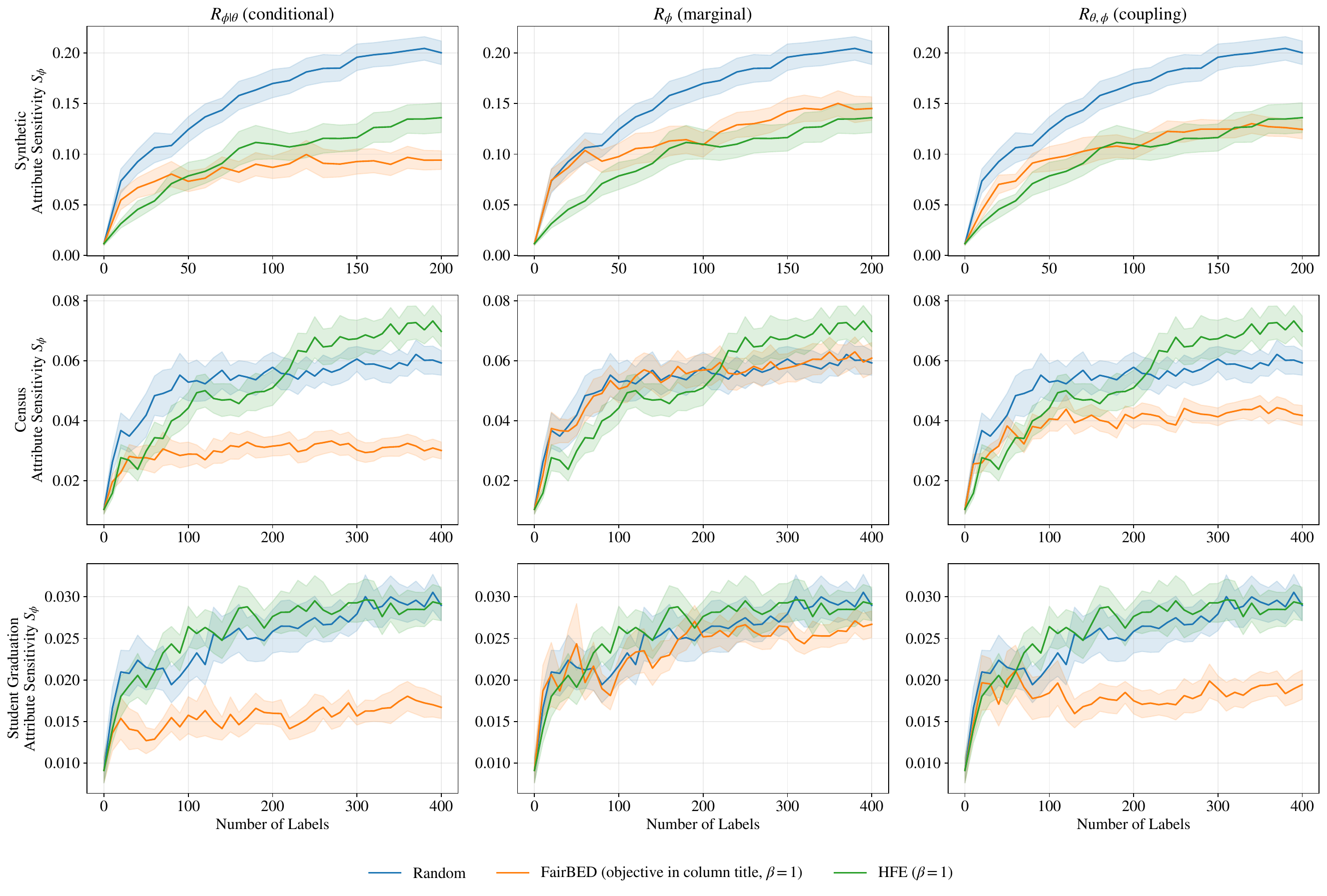}
    \caption{\textbf{Attribute sensitivity during acquisition} at $\beta{=}1$ for each dataset (rows) and FairBED variant (columns), with Random and HFE ($\beta{=}1$). Mean $\pm$ s.e.\ over 24 seeds.}
    \label{fig:app_al_curves}
\end{figure*}

\begin{figure*}[!t]
    \centering
    \includegraphics[width=\linewidth]{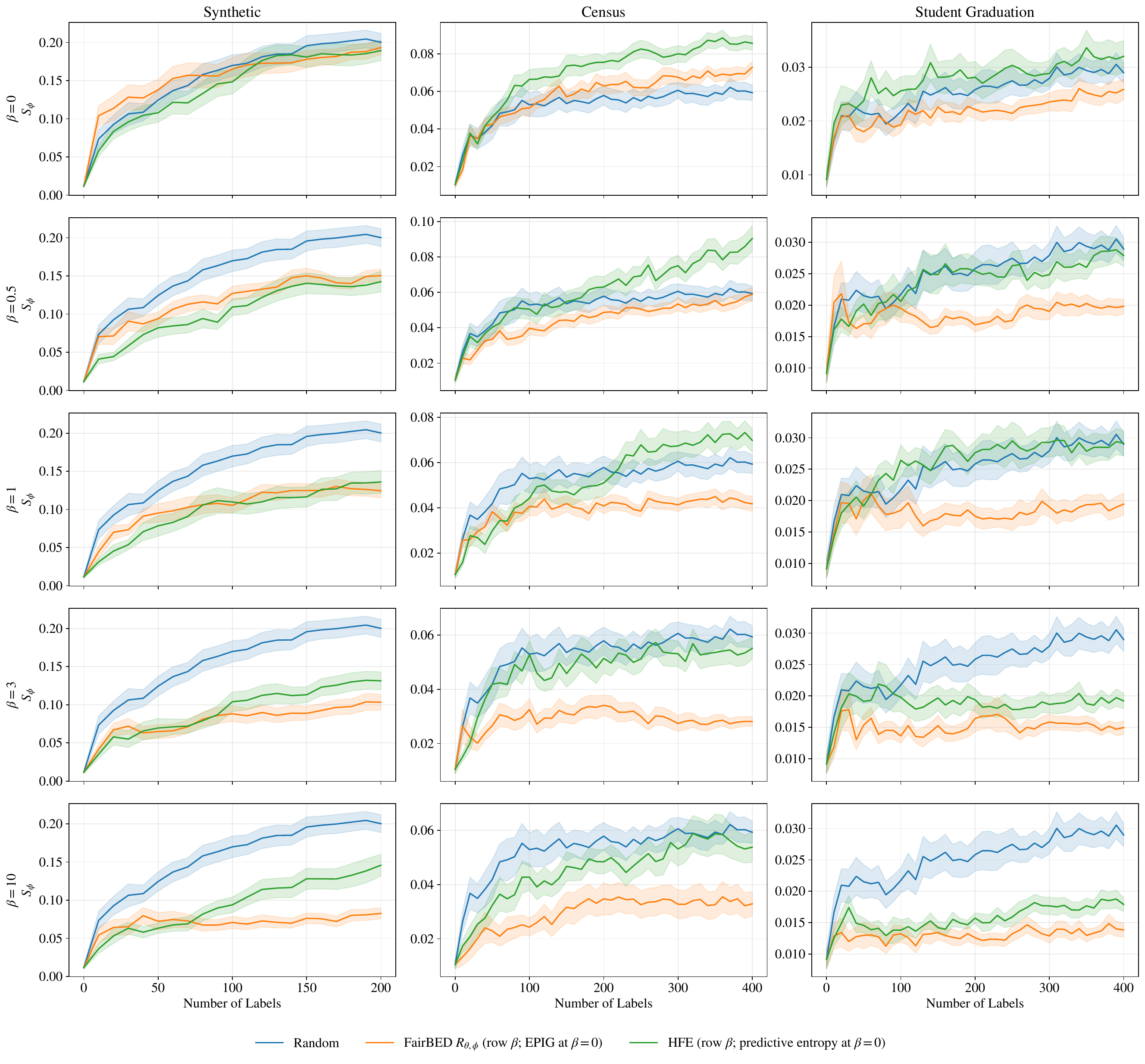}
    \caption{\textbf{Attribute sensitivity during acquisition with $\Rc$} for each $\beta$ (rows) and dataset (columns), with Random and HFE at the same $\beta$. At $\beta{=}0$ the penalty vanishes, so $\Rc$ is EPIG and HFE is predictive entropy. Mean $\pm$ s.e.\ over 24 seeds.}
    \label{fig:app_al_coupling_beta}
\end{figure*}

\looseness=-1
\textbf{Starting from one group.}~~
\Cref{tab:al_one_group} compares the standard start with a start from two labels of a single group (one per target class, both men) on Census.
Starting from one group changes the baselines little, and raises the accuracy of $\Rcond$ at $\beta\ge3$ with a similar attribute sensitivity.

\begin{table*}[!t]
\centering
\caption{Census, 400 labels, 24 seeds, mean $\pm$ s.e. Standard start: one initial label per $(y_\theta, y_\phi)$ combination (4 labels). One-group start: two initial labels (one per target class), both with $y_\phi = 1$.}
\label{tab:al_one_group}
\resizebox{\textwidth}{!}{%
\begin{tabular}{l cccc cccc}
\toprule
 & \multicolumn{4}{c}{\textbf{Standard start}} & \multicolumn{4}{c}{\textbf{One-group start}} \\
\cmidrule(lr){2-5} \cmidrule(lr){6-9}
\textbf{Method} & Acc.\,(\%)$\uparrow$ & Attr.\ sens.$\downarrow$ & Adv.\,(\%)$\downarrow$ & DP gap$\downarrow$ & Acc.\,(\%)$\uparrow$ & Attr.\ sens.$\downarrow$ & Adv.\,(\%)$\downarrow$ & DP gap$\downarrow$ \\
\midrule
Random & $82.9 \pm 0.1$ & $0.059 \pm 0.004$ & $75.4 \pm 0.2$ & $0.244 \pm 0.003$ & $82.8 \pm 0.2$ & $0.067 \pm 0.003$ & $75.5 \pm 0.1$ & $0.238 \pm 0.003$ \\
Predictive entropy & $84.1 \pm 0.1$ & $0.086 \pm 0.004$ & $75.5 \pm 0.2$ & $0.181 \pm 0.003$ & $84.0 \pm 0.1$ & $0.082 \pm 0.004$ & $75.6 \pm 0.2$ & $0.180 \pm 0.003$ \\
EPIG & $83.8 \pm 0.1$ & $0.073 \pm 0.003$ & $75.8 \pm 0.2$ & $0.238 \pm 0.003$ & $83.9 \pm 0.1$ & $0.077 \pm 0.003$ & $75.7 \pm 0.2$ & $0.236 \pm 0.004$ \\
\midrule
$\Rcond$ ($\beta=0.1$) & $83.4 \pm 0.1$ & $0.077 \pm 0.003$ & $74.8 \pm 0.2$ & $0.240 \pm 0.003$ & $83.5 \pm 0.1$ & $0.076 \pm 0.004$ & $75.1 \pm 0.2$ & $0.242 \pm 0.003$ \\
$\Rcond$ ($\beta=1$) & $75.6 \pm 1.1$ & $0.030 \pm 0.003$ & $65.5 \pm 0.9$ & $0.178 \pm 0.008$ & $77.8 \pm 0.5$ & $0.034 \pm 0.004$ & $65.4 \pm 0.6$ & $0.196 \pm 0.005$ \\
$\Rcond$ ($\beta=3$) & $67.3 \pm 2.3$ & $0.030 \pm 0.005$ & $61.0 \pm 1.7$ & $0.114 \pm 0.008$ & $74.8 \pm 1.5$ & $0.031 \pm 0.003$ & $63.3 \pm 0.4$ & $0.159 \pm 0.010$ \\
$\Rcond$ ($\beta=10$) & $66.6 \pm 2.2$ & $0.025 \pm 0.006$ & $62.0 \pm 1.8$ & $0.116 \pm 0.010$ & $74.5 \pm 1.2$ & $0.027 \pm 0.003$ & $62.5 \pm 0.7$ & $0.152 \pm 0.009$ \\
\bottomrule
\end{tabular}}
\end{table*}

\looseness=-1
\textbf{Group composition and per-group accuracy.}~~
Steering acquisition could in principle lower the leakage by under-sampling one group.
\Cref{tab:al_groups} reports, for Student Graduation, the share of acquired labels from each group and the target accuracy within each group for the two targeted objectives of \Cref{fig:al_targeted_objectives}.
Both keep the group composition close to that of Random, and the difference in accuracy between the groups stays at the level of Random and EPIG at every $\beta$, so on this dataset the reductions in leakage do not come from excluding a group.

\begin{table}[!t]
\centering
\caption{Student Graduation, $400$ labels, mean $\pm$ s.e.\ over 24 seeds. Men: share of acquired labels from men ($y_\phi{=}1$); Acc.\ women / men: target accuracy within each group; Gap: absolute difference between the two.}
\label{tab:al_groups}
\begin{tabular}{l cccc}
\toprule
\textbf{Method} & Men (\%) & Acc.\ women (\%) & Acc.\ men (\%) & Gap \\
\midrule
Random & $35.0 \pm 0.5$ & $81.5 \pm 0.2$ & $82.8 \pm 0.3$ & $1.3$ \\
EPIG & $34.9 \pm 0.3$ & $82.2 \pm 0.1$ & $83.5 \pm 0.3$ & $1.3$ \\
\midrule
$\Rphi$ ($\beta=1$) & $33.7 \pm 0.7$ & $81.7 \pm 0.2$ & $82.9 \pm 0.2$ & $1.2$ \\
$\Rphi$ ($\beta=3$) & $34.2 \pm 0.8$ & $81.3 \pm 0.2$ & $82.6 \pm 0.2$ & $1.3$ \\
$\Rphi$ ($\beta=10$) & $35.3 \pm 1.5$ & $80.6 \pm 0.3$ & $82.3 \pm 0.3$ & $1.7$ \\
$\Rc$ ($\beta=1$) & $40.2 \pm 1.6$ & $79.3 \pm 0.4$ & $80.8 \pm 0.4$ & $1.6$ \\
$\Rc$ ($\beta=3$) & $36.1 \pm 2.7$ & $75.8 \pm 0.7$ & $77.3 \pm 0.5$ & $1.5$ \\
$\Rc$ ($\beta=10$) & $32.1 \pm 3.3$ & $71.3 \pm 1.4$ & $72.4 \pm 1.6$ & $1.1$ \\
\bottomrule
\end{tabular}
\end{table}

\subsubsection{Early Stopping}
\label{app:al_early_stopping}

\looseness=-1
Each acquisition has a FairBED score, the estimated utility minus $\beta$ times the estimated penalty, and the running sum of these scores estimates the objective of the dataset acquired so far.
Early stopping keeps the dataset at the step where this running sum peaks (the latest step on ties) or when the objectives go negative.
\Cref{fig:app_al_es} overlays the early-stopped datasets on \Cref{fig:al_targeted_objectives}.
Wherever it stops before the full budget, early stopping moves the point to a lower value of the targeted leakage, the attribute sensitivity for $\Rc$ and the sensitive test accuracy for $\Rphi$, at a lower target accuracy: it prioritises the targeted fairness metric and never moves a point to the right (within error bars).
At small $\beta$ the cumulative objective is still rising at the full budget, so early stopping keeps every label and the points coincide.
Unlike acquisition, which only ranks the candidates at each step, early stopping applies a threshold, keeping a label only while the cumulative score is still rising.
Ranking stays consistent even when the estimated scores are noisy, whereas a threshold is more sensitive to that noise, particularly at large $\beta$, where the utility and the penalty nearly cancel and each step's score is small relative to its estimation error.
This shows up as larger error bars for the early-stopped datasets at larger $\beta$.
Early stopping can also be implemented in other ways, for example by stopping once a budget on the cumulative estimated leakage is spent, or by averaging the score over a window of steps before applying the threshold, which reduces its sensitivity to noise.

\begin{figure*}[!htbp]
    \centering
    \includegraphics[width=\linewidth]{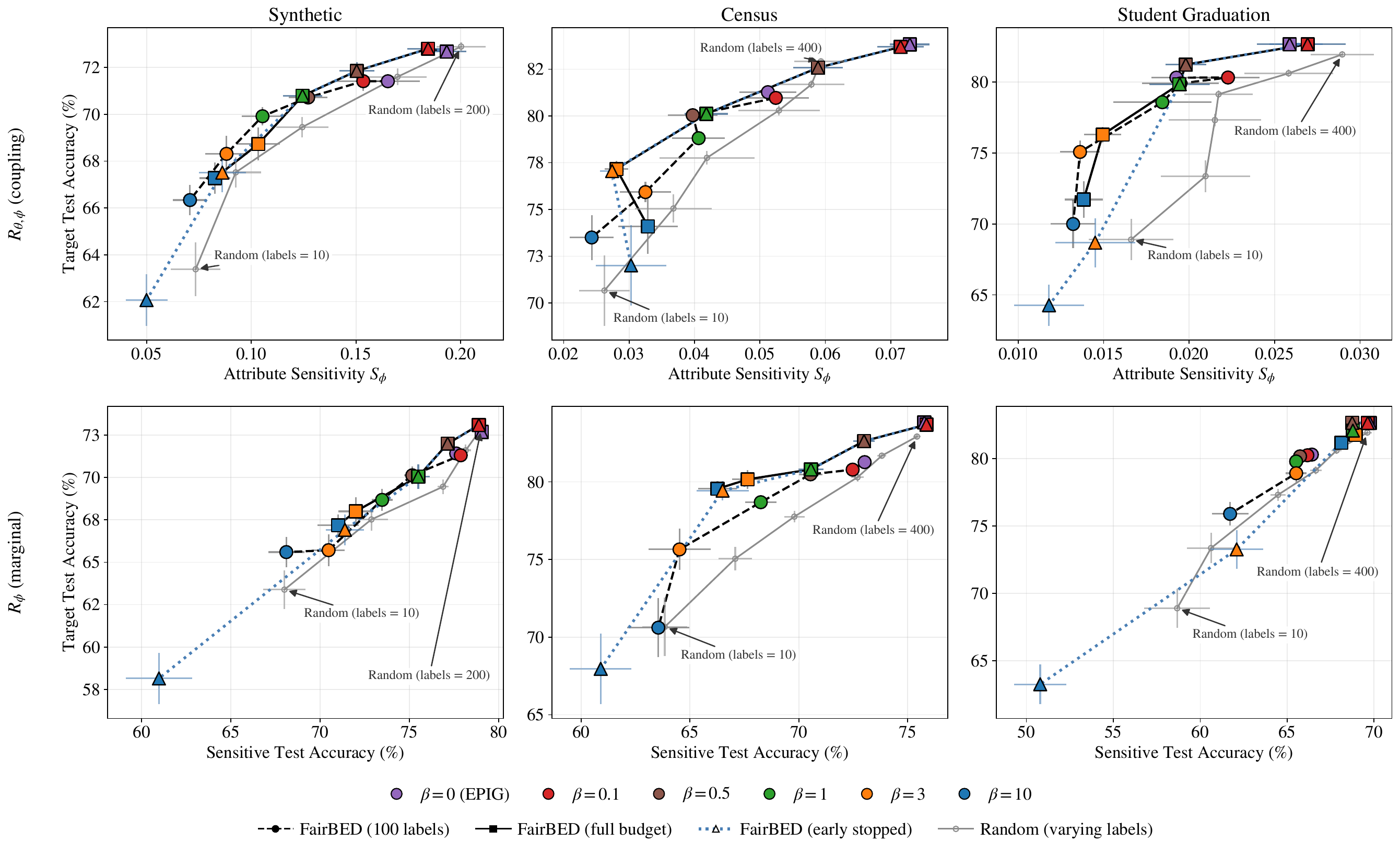}
    \caption{\textbf{Early stopping} on the objectives of \Cref{fig:al_targeted_objectives}: $\Rc$ against attribute sensitivity (top) and $\Rphi$ against sensitive test accuracy (bottom), with the early-stopped datasets at the full budget (triangles, dotted); grey: Random at increasing label budgets. Mean $\pm$ s.e.\ over 24 seeds.}
    \label{fig:app_al_es}
\end{figure*}


\FloatBarrier
\section{Fairness definitions}\label{app:fairness definitions}
In this section we present fairness definitions in terms of a binary classifier with true labels $Y\in\{0,1\}$, predicted labels $\hat{Y}\in\{0,1\}$ and sensitive attribute $\Phi\in\{\phi_1,\phi_2\}$. We do not target these notions but their definitions are included for completeness.
\newline
\newline
\paragraph{Demographic Parity \cite{dwork2011fairnessawareness}}
A predictor satisfies demographic parity (also called statistical parity)
if its rate of positive predictions is equal across sensitive groups,
regardless of the true labels:
\begin{equation}
P(\hat{Y} = 1 \mid \Phi = \phi_1) = P(\hat{Y} = 1 \mid \Phi = \phi_2).
\end{equation}
The demographic-parity gap is the absolute
difference between the two sides. With an empty context,
\cref{prop:guarantee} shows that coupling-free data give demographic
parity of the posterior predictive.

\textbf{Predictive Parity \citep{delbarrio2020reviewmathematicalframeworksfairness}}
Within supervised learning, \textit{predictive parity} requires equal precision across sensitive groups, i.e., the probability that the predicted positive instances are actually positive should be equal across sensitive attributes.

\begin{equation}
    P(Y = 1 \mid \hat{Y} = 1, \Phi = \phi_1) = P(Y = 1 \mid \hat{Y} = 1, \Phi = \phi_2),
\end{equation}

where $Y$ are the true labels and $\hat{Y}$ are the predicted labels.
This fairness metric is increasingly prioritised within healthcare \citep{gao2025fair}. Additionally, it is simple to extend this fairness definition to a continuous parameter setting as it implies an invariance across a subset of the sensitive attribute space. Thus, this fairness definition more naturally fits within fair decision making problems compared to fair parameter inference problems. 

\textbf{Equalised Odds \citep{NIPS2016_9d268236}}
A predictor satisfies \textit{equalised odds} if its true positive rate and false positive rate are equal across sensitive groups. Formally, for binary outcomes $Y \in \{0, 1\}$ and sensitive attribute $\Phi$:

\begin{equation}
P(\hat{Y} = 1 \mid Y = y, \Phi = \phi_1) = P(\hat{Y} = 1 \mid Y = y, \Phi = \phi_2) \quad \text{for } y \in \{0, 1\}
\end{equation}

Equalised odds asks for parity in error behaviour, not parity in selection rate. It is about fairness of mistakes relative to the ground truth.

\textbf{Equal Opportunity \citep{NIPS2016_9d268236}}
\textit{Equal opportunity} is a relaxation of equalised odds, focusing only on the true positive rate. Formally,

\begin{equation}
P(\hat{Y} = 1 \mid Y = 1, \Phi = \phi_1) = P(\hat{Y} = 1 \mid Y = 1, \Phi = \phi_2).
\end{equation}

This ensures that individuals who truly belong to the positive class have equal chances of being correctly predicted as such, irrespective of group membership.


\appendix
\thispagestyle{empty}

\end{document}